\documentclass[transmag, onecolumn, 11pt]{IEEEtran}

\usepackage[numbers]{natbib}
\ifCLASSINFOpdf
\else
\fi
\usepackage{amsmath}
\usepackage{array}

\ifCLASSOPTIONcompsoc
  \usepackage[caption=false,font=normalsize,labelfont=sf,textfont=sf]{subfig}
\else
  \usepackage[caption=false,font=footnotesize]{subfig}
\usepackage{fixltx2e}
\usepackage{url}
\usepackage{hyperref}

\usepackage{tikz}
\usepackage{amssymb}
\usepackage{amsthm}

\usepackage{thm-restate}

\newcommand{\RR}{\mathbb{R} }
\newcommand{\NN}{\mathbb{N}}

\newcommand{\act}[2]{\textsf{act}(#1;#2)}
\newcommand{\n}[2]{\textsf{n}(#1;#2)}

\newcommand{\actphi}[2]{\textsf{act}^\varphi(#1;#2)}
\newcommand{\nphi}[2]{\textsf{n}^\varphi(#1;#2)}

\newcommand{\nv}{\textbf{n}}
\newcommand{\av}{\textbf{a}}

\newcommand{\h}[2]{h_{\parm,#1}(#2)}

\newcommand{\hphi}[2]{h^\varphi_{\parm,#1}(#2)}

\newcommand{\w}{\mathbf{w}}
\newcommand{\wrest}{\mathbf{\bar{w}}}

\newcommand{\Loss}{\mathcal{L}}
\newcommand{\loss}{\ell}

\newcommand{\textdef}[1]{\textbf{#1}}

\newcommand{\parm}{\mathord{\color{black!35}\bullet}}%
\newcommand{\parmBlue}{\mathord{\color{blue!50}\bullet}}%

\DeclareMathOperator{\spn}{span}

\newtheorem{theorem}{Theorem}

\newtheorem{remark}[theorem]{Remark}

\newtheorem{definition}[theorem]{Definition}
\newtheorem{corollary}[theorem]{Corollary}

\newtheorem{appendixLemma}{Lemma B.\hspace{-0.12cm}}

\begin{document}
%
\title{Non-attracting Regions of Local Minima\\ in Deep and Wide Neural Networks}


\author{\IEEEauthorblockN{Henning Petzka\IEEEauthorrefmark{1},
Cristian Sminchisescu\IEEEauthorrefmark{1,2}}
\IEEEauthorblockA{\IEEEauthorrefmark{1}Lund University}
\IEEEauthorblockA{\IEEEauthorrefmark{2}Google Research}
\thanks{This work was supported in part by the European Research Council Consolidator grant SEED, CNCS-UEFISCDI (PN-III-P4-ID-PCE-2016-0535, PN-III-P4-ID-PCCF-2016-0180), the EU Horizon 2020 grant DE-ENIGMA (688835), and SSF.\newline
H. Petzka (email: henning.petzka@math.lth.se), C. Sminchisescu (cristian.sminchisescu@math.lth.se)}}

\markboth{}%
{ }
%



\IEEEtitleabstractindextext{%
\begin{abstract}
Understanding the loss surface of neural networks is essential for the design of models with predictable performance and their success in applications. Experimental results suggest that sufficiently deep and wide neural networks are not negatively impacted by suboptimal local minima. Despite recent progress, the reason for this outcome is not fully understood. Could deep networks have very few, if at all, suboptimal local optima? or could all of them be equally good? We provide a construction to show that suboptimal local minima (i.e., non-global ones), even though degenerate, exist for fully connected neural networks with sigmoid activation functions. The local minima obtained by our construction belong to a connected set of local solutions that can be escaped from via a non-increasing path on the loss curve. For extremely wide neural networks of decreasing width after the wide layer, we prove that every suboptimal local minimum belongs to such a connected set.  This provides a partial explanation for the successful application of deep neural networks. In addition, we also characterize under what conditions the same construction leads to saddle points instead of local minima for deep neural networks. 

\end{abstract}

}

\maketitle

\IEEEdisplaynontitleabstractindextext

%
\IEEEpeerreviewmaketitle

\section{Introduction}

At the heart of most optimization problems lies the search for the global minimum of a loss function. The common approach to finding a solution is to initialize at random in parameter space and subsequently follow directions of decreasing loss based on local methods. This approach lacks a global progress criteria, which leads to descent into one of the nearest local minima. The common approach of using gradient descent variants on non-convex loss curves of deep neural networks is vulnerable precisely to that problem.

Authors pursuing the early approaches to local descent by back-propagating gradients \citep{Rumelhart} experimentally noticed that suboptimal local minima appeared surprisingly harmless. More recently, for deep neural networks, the earlier observations were further supported by experiments of, e.g., \citet{Zhang}. Several authors aimed to provide theoretical insight for this behavior. 
Some, aiming at explanations, rely on simplifying modeling assumptions. Others investigate neural networks under realistic assumptions, but often focus on failure cases only. Recently, \citet{NguyenHein} provide partial explanations for deep and {\it extremely wide} neural networks for a class of activation functions including the commonly used sigmoid. Extreme width is characterized by a  ``wide'' layer that has more neurons than input patterns to learn. For almost every instantiation of parameter values $\w$ (i.e., for all but a set of parameter values of measure zero) it is shown that, if the loss function has a local minimum at $\w$, then this local minimum must be a global one. This extends results  by \cite{GoriTesi} who required the input layer to be extremely wide. This suggests that for deep and wide neural networks, possibly every local minimum is global. The question on what happens at the null set of parameter values, for which the result does not hold, remained unanswered.

Similar observations for shallow neural networks with one hidden layer were made earlier by \citet{Poston}. \citet{Poston} show for a neural network with one hidden layer and sigmoid activation function that, if the hidden layer has more nodes than there are training patterns, then the error function (squared sum of prediction losses over the samples) has no suboptimal ``local minimum'' and ``each point is arbitrarily close to a point from which a strictly decreasing path starts, so such a point cannot be separated from a so called “good” point by a barrier of any positive height'' \citep{Poston}. It was criticized by \citet{Sprinkhuizen} that the definition of a local minimum used in the proof of \citet{Poston} was rather strict and unconventional. In particular, the results do not imply that no suboptimal local minima, defined in the usual way, exist. As a consequence, the notion of attracting and non-attracting regions of local minima were introduced and the authors prove that non-attracting regions exist by providing an example for the extended XOR problem. The existence of these regions imply that a gradient-based approach descending the loss surface using local information may still not converge to the global minimum. 
The main objective of this work is to revisit the problem of such non-attracting regions and show that they also exist in deep and extremely wide networks. In particular, a gradient based approach may get stuck in a suboptimal local minimum also in these networks. Most importantly, the performance of deep and wide neural networks cannot be explained by the analysis of the loss curve alone, without taking proper initialization or the stochasticity of stochastic gradient descent (SGD) into account.

Our observations are not fundamentally negative. At first, the local minima we find are rather degenerate. With proper initialization, a local descent technique is unlikely to get stuck in one of the degenerate, suboptimal local minima.\footnote{That a proper initialization largely improves training performance is well-known. See, e.g., \citet{Wessels2}.} Secondly, the minima reside on a non-attracting region of local minima (see Definition~\ref{def:nonattractRegion}). Due to its exploration properties, stochastic gradient descent will eventually be able to escape from such a region \citep[see][]{Wei}. It is conceivable that  
in sufficiently wide and deep networks, except for a null set of parameter values as starting points, there is always a monotonically decreasing path down to the global minimum. This was shown for neural networks with one hidden layer, sigmoid activation function and square loss \citep{Poston}, and we generalize this result to deep neural networks. This implies that in such networks every local minimum belongs to a non-attracting region of local minima. (More precisely, our result holds for all extremely wide neural networks with square loss and a class of activation functions including the sigmoid, where the sequence of dimensions of hidden layers is non-increasing between the extremely wide layer and the output layer, i.e., the network architecture has no bottleneck layer of strictly lower dimension than both its neighboring layers.) 

Our proof of the existence of suboptimal local minima even in extremely wide and deep networks is based on a construction of local minima in shallow neural networks given by \citet{FukumizuAmari}. By relying on a careful computation we are able to characterize when this construction is applicable to deep neural networks. Interestingly, in deeper layers, the construction rarely seems to lead to local minima, but more often to saddle points. The argument that saddle points rather than suboptimal local minima are the main problem in deep networks has been raised before \citep{Dauphin} but a theoretical justification \citep{Choromanska} uses strong assumptions that do not exactly hold in neural networks. Here, we provide the first analytical argument, under realistic assumptions on the neural network structure, describing when certain critical points (i.e., points with gradient zero) of the training loss lead to saddle points in deeper networks. 

In summary, our results contain the following insight: There exist non-attracting regions of local minima and, in particular, suboptimal local minima in the loss surface of arbitrarily wide neural networks for a class of analytic activation functions including the sigmoid function. The minima can be both of finite type or only exist in the limit as some parameters converge to infinity. This disproves a conjecture made by \citet{NguyenHein} stating that for the therein studied extremely wide neural networks all local minima are globally optimal. Non-attracting regions of local minima, however, allow for non-increasing paths to the global minimum by first following degenerate directions of the local minimum. In sufficiently wide neural networks with no bottleneck layer, all local minima belong to non-attracting regions of local minima.

The extremely wide neural networks considered have zero loss at global minima. Naturally, training for zero global loss is not desirable in practice, neither is the use of fully connected extremely wide deep neural networks necessarily. The results of this paper are of theoretical importance. To be able to understand the complex learning behavior of deep neural networks in practice, it is a necessity to understand the networks with the most fundamental structure. In this regard, our results offer new understanding of the multidimensional loss surface of deep neural networks and their learning behavior.  

\newpage

\section{Related Work}
We discuss related work on suboptimal minima of the loss surface. In addition, we refer the reader to the overview article by \citet{Vidal} for a discussion on the non-convexity in neural network training. 

It is known that learning the parameters of neural networks is, in general, a hard problem. \citet{BlumRivest} prove NP-completeness for a specific neural network. It has also been shown that local minima and other critical points exist in the loss function of neural network training \citep{Auer,FukumizuAmari,Nitta,sminchisescu_ijcv04,Sprinkhuizen,Wessels, Yun}. The understanding of these critical points has led to significant improvements in neural network training. This includes weight initialization techniques \citep{Wessels2}, improved backpropagation algorithms to avoid saturation effects in neurons \citep{Wang}, entirely new activation functions, or the use of second order information \citep{MizutaniDreyfus,Amari}. That suboptimal local minima must become rather degenerate if the neural network becomes sufficiently large was observed for networks with one hidden layer by \citet{Poston}. Extending work by \citet{GoriTesi}, \citet{NguyenHein,NguyenHein2} generalized this result to deeper networks containing an extremely wide hidden layer. Our contribution can be considered as a continuation of this work.

To explain the persuasive performance of deep neural networks, \citet{Dauphin} experimentally show that there is a similarity in the behavior of critical points of the neural network's loss function with theoretical properties of critical points found for Gaussian fields on high-dimensional spaces \citep{BrayDean}. \citet{Choromanska} supply a theoretical connection, but they also require strong (arguably unrealistic) assumptions on the network structure. The results imply that (under their assumptions on the deep network) the loss at a local minimum must be close to the loss of the global minimum with high probability. In this line of research, \citet{Sagun} experimentally show a similarity between spin glass models and the loss curve of neural networks. 

There is a growing number of papers considering the existence of suboptimal local minima  for ReLU and LeakyReLU networks, where the space becomes combinatorial in terms of a positive activation, compared to a stalled (or weak) signal.  \citet{Yun} prove existence of bad local minima in ReLU networks for generic data sets by tuning the weight parameters in such a way that all neurons are active and the network becomes locally linear. The existence of bad local minima had previously been shown under stronger assumptions by \citet{Du,Zhou} and \citet{Swirszcz}, who construct data sets that allow them to find suboptimal local minima in overparameterized networks. For the hinge loss, \citet{Laurent2018b} study one-hidden-layer networks  and show that Leaky-ReLU networks don't have bad local minima, while ReLU networks do. Conditions for ReLU networks characterizing when no bad local minima exist or how to eliminate them is discussed by \citet{Liang,LiangB}. \citet{SoudryHoffer} probabilistically compare the volume of regions (for a specific measure) containing bad local and global minima in the limit, as the number of data points goes to infinity.  For networks with one hidden layer and ReLU activation function, \citet{Freeman} quantify the amount of hill-climbing necessary to go from one point in the parameter space to another and finds that for increasing overparameterization, all level sets become connected.  Instead of analyzing local minima, \citet{Xie} consider regions where the derivative of the loss is small for two-layer ReLU networks. \citet{SoudryCarmon} consider leaky ReLU activation functions to find, similarly to the result of \citet{NguyenHein}, that for almost every combination of activation patterns in two consecutive mildly wide layers, a local minimum has global optimality.

To gain better insight into theoretical aspects, some papers consider linear networks, where the activation function is the identity. The classic result by \citet{BaldiHornik} shows that linear two-layer neural networks have a unique global minimum and all other critical values are saddle points.  \citet{Kawaguchi},  \citet{LuKawaguchi} and \citet{Yun} discuss generalizations of the results by \citet{BaldiHornik} to deep linear networks, and \citet{Laurent2018b} finally show that for linear networks with no bottleneck layer, all minima are global.

The existence of non-increasing paths on the loss curve down to the global minimum is studied by \citet{Poston} for extremely wide two-layer neural networks with sigmoid activation functions. \citet{NguyenHein3} generalize this and show existence of such paths for a special type of architecture having as many skip connections to the output as there are input patterns to learn. For ReLU networks, \citet{SafranShamir} show that, if one starts at a sufficiently high initialization loss, then there is a strictly decreasing path of parameters into the global minimum.  \citet{Haeffele} consider a specific class of ReLU networks with regularization, give a sufficient condition that a local minimum is globally optimal, and show that a non-increasing path down to the global minimum exists.

Finally, worth mentioning is the study of \citet{LiaoPoggio} who use polynomial approximations to argue, by relying on Bezout's theorem, that the loss function should have many local minima with zero empirical loss. Why deep networks perform better than shallow ones is also investigated by \citet{Poggio} by considering a class of compositional functions. Also relevant is the observation by  \citet{Brady} showing that, if the global minimum is not of zero loss, then a perfect predictor may have a larger loss in training than one producing worse classification results.


\section{Main Results}\label{sct:problemDef}


We consider \textdef{neural network functions} with fully connected layers of size $n_l,\ 0\leq l\leq L$ given by
$$f(x)=\w^L(\sigma(\w^{L-1}(\sigma(\ldots\sigma (\w^2(\sigma(\w^1 x+\w^1_0))+\w^2_0) \ldots))+\w^{L-1}_0))+\w_0^L,$$
where $\mathbf{\w^l\in \RR ^{n_{l}\times n_{l-1}}}$ denotes the \textdef{weight matrix} of the $l$-th layer, $1\leq l\leq L$, $\w^l_0$ the \textdef{bias} terms, and $\sigma$ a nonlinear \textdef{activation function}. The neural network function is denoted by $f$ and we notationally suppress dependence on parameters. We assume the activation function $\sigma$ to belong to the class of strict monotonically increasing, analytic, bounded functions on $\RR$ with image an interval $(c,d)$ such that $0\in [c,d]$, a \textdef{class denoted by $\mathbf{\mathcal{A}}$}. As prominent examples, the sigmoid activation function $\sigma(t)=\frac{1}{1+\text{exp}(-t)}$ and $\sigma(t)=\text{tanh}(x)$ lie in $\mathcal{A}$. We assume no activation function at the output layer. All the networks considered in this paper are \textdef{regression networks} mapping into the real numbers $\RR $, i.e., $\mathbf{n_L=1}$ and $\w^L\in \RR ^{1\times n_{L-1}}$. We train on a \textdef{finite data set} $(x_\alpha,y_\alpha)_{1\leq \alpha\leq N}$ of size $N$ with \textdef{input patterns} $x_\alpha\in\RR^{n_0}$ and desired \textdef{target value} $y_\alpha\in \RR $. We suppose throughout that the input patterns are pairwise different. We aim to minimize the \textdef{squared loss }$\Loss= \sum_{\alpha=1}^N (f(x_\alpha) -y_\alpha)^2$. Further, $M$ denotes the \textdef{total number of parameters} and $\w\in\RR^M$ denotes the \textdef{collection of all $\w^l$ and $\w_0^l$}. 

The dependence of the neural network function $f$ on $\w$ translates into a dependence $\Loss=\Loss(\w)$ of the loss function on the parameters $\w$. Due to assumptions on $\sigma$, $\Loss(\w)$ is twice continuously differentiable. The goal of training a neural network consists of minimizing $\Loss(\w)$ over $\w$. There is a unique value $\Loss_0$ denoting the infimum of the neural network's loss (most often $\Loss_0=0$ in our examples). Any set of weight parameters $\w_{\parm}$ that satisfies $\Loss(\w_{\parm})=\Loss_0$ is called a \textdef{global minimum}.  Due to its non-convexity, the loss function $\Loss(\w)$ of a neural network is in general known to potentially suffer from local minima (precise definition of a local minimum below).  We will study the existence of \textdef{suboptimal local minima} in the sense that a local minimum $\w_*$ is suboptimal if its loss $\Loss(\w_*)$ is strictly larger than $\Loss_0$. 

We refer to \textdef{deep neural networks} as networks with more than one hidden layer. Further, we refer to \textdef{extremely wide neural networks} as the type of networks considered in other theoretical work \citep{GoriTesi, Poston, NguyenHein, NguyenHein2} with one hidden layer containing at least as many neurons as input patterns (i.e., $n_l\geq N$ for some $1\leq l<L$ in our notation).
\vspace{0.4cm}


\subsection{A Special Kind of Local Minimum}

The standard definition of a \textdef{local minimum}, which is also used here, is a point $\w_*$ such that $\w_*$ has a neighborhood $U$ with $\Loss(\w)\geq \Loss(\w_*)$ for all $\w\in U$. Since local minima do not need to be \textdef{isolated} (i.e., $\Loss(\w)> \Loss(\w_*)$ for all $\w\in U\setminus \{w_*\}$) two types of connected regions of local minima may be distinguished.  In the following definition,  a \textdef{continuous path} is a continuous map $\w_\Gamma:[0,1]\rightarrow \RR^{M}$ assigning each $t\in[0,1]$ a choice of parameters values $\w_\Gamma(t)$ with loss $\Loss(\w_\Gamma(t))$. We call the path \textdef{non-increasing} in $\Loss$ if $\Loss(\w_\Gamma(t))\leq \Loss(\w_\Gamma(s))$ for all $t\geq s$. A non-increasing path $\w_\Gamma(t)$ \textdef{decreases the loss maximally}, if it cannot be extended as a non-increasing path to a parameter setting of lower loss, or formally, if there exists no non-increasing path $\tilde \w_\Gamma(t)$ such that for each $t$ in $[0,1]$ there is $s$ in $[0,1]$ with $\w_\Gamma(t)=\tilde \w_\Gamma(s)$ and such that $\Loss(\tilde \w_\Gamma(1))<\Loss(\w_\Gamma(1))$. 

\begin{definition}\citep{Sprinkhuizen}\label{def:nonattractRegion}
Let $\Loss:\RR^n\rightarrow \RR$ be a differentiable function. Suppose $R$ is a maximal connected subset of parameter values $\w\in \RR^m$, such that every $\w\in R$ is a local minimum of $\Loss$ with value $\Loss(\w)=c$. 
\begin{itemize}
\item R is called an \textdef{attracting region of local minima}, if there is a neighborhood $U$ of $R$ such that every continuous path $\w_\Gamma(t)$, which is non-increasing in $\Loss$, which starts at some $\w_\Gamma(0)=\w\in U$ and which decreases the loss maximally, ends in $R$.   
\item R is called a \textdef{non-attracting region of local minima}, if every neighborhood $U$ of $R$ contains a point from where a continuous path $\w_\Gamma(t)$ exists that is non-increasing in $\Loss$ and ends in a point $\w_\Gamma(1)$ with $\Loss(\w_\Gamma(1))<c$. 
\end{itemize}
\end{definition}

\begin{figure}
\begin{center}
\includegraphics[height=5.8cm]{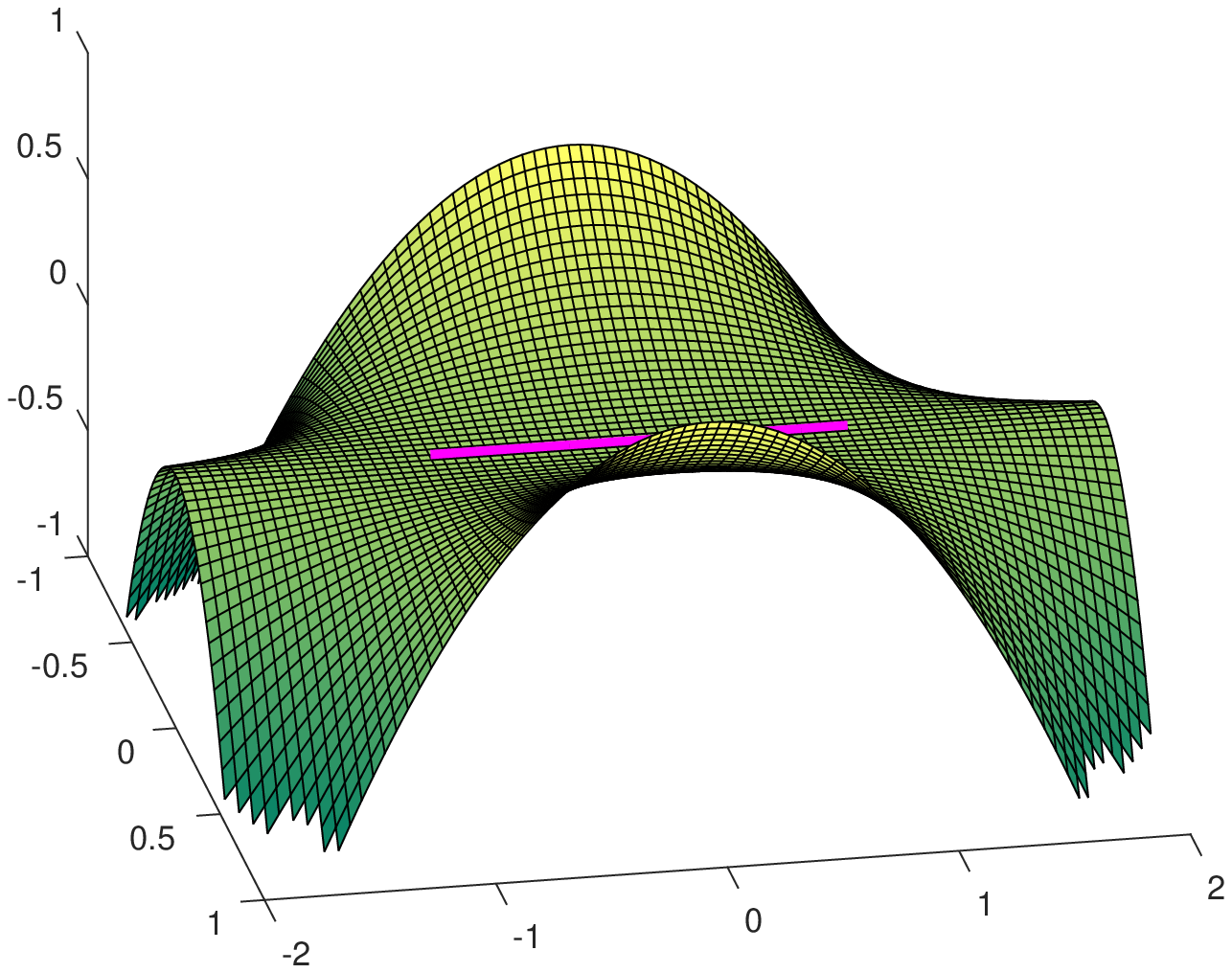}
\includegraphics[height=5.8cm]{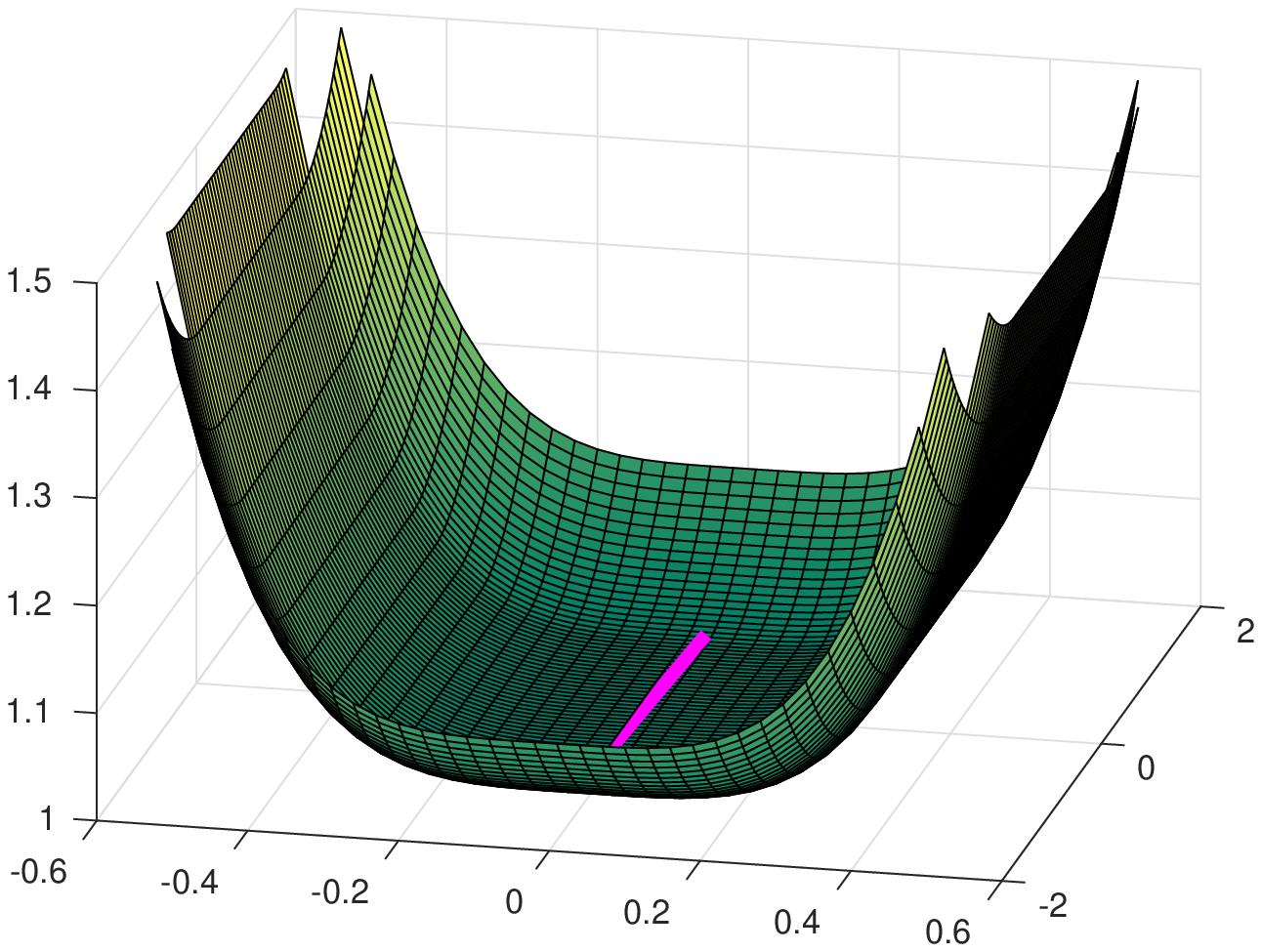}
\caption{Left: A non-attracting region of local minima given by $R=\{(x,y)\ |\ x=0, y\in(-1,1)\}$ illustrated by the function $f(x,y)=x^2(1-y^2)$. Right: An attracting region of local minima at the same region $R$ for comparison. (These examples do not exactly appear in neural networks considered in this paper, but are of similar nature.)}
\label{fig:minima}
\end{center}
\end{figure}

Attracting regions of local minima $R$ are called attracting, as decreasing paths starting in a neighborhood of $R$ eventually end up in $R$. Our notion differs from the one of \citet{Sprinkhuizen} by considering non-increasing paths instead of strictly increasing ones \citep[see also][]{Hamey}. Despite its non-attractive nature, a non-attracting region $R$ of local minima may be harmful for a gradient descent approach. A path of greatest descent can end in a local minimum on $R$. However, no point $z$ on $R$ needs to have a neighborhood of attraction in the sense that following the path of greatest descent from a point in a neighborhood of $z$ will lead back to $z$. (The path can lead to a different local minimum on $R$ close by or reach points with strictly smaller values than $c$.) A rough illustration of a non-attracting region of local minima is depicted in \figurename~\ref{fig:minima}.\footnote{While one might be tempted to term regions of local minima ``generalized saddle points'', we note that, under the usual mathematical definition, they do consists of a set of local minima.} 
Such non-attracting regions of local minima are considered for neural networks with one hidden layer by \citet{FukumizuAmari} and \citet{Wei} under the name of {\it singularities}. Their regions of local minima are characterized by singularities in the parameter space 
leading to a loss value strictly larger than the global loss. The dynamics around such a region are investigated by \citet{Wei}. 

Non-attracting regions of local minima do not only exist for shallow two-layer neural networks, but also for deep and arbitrary wide networks. {A construction of such regions is shown in Section~\ref{sct:constructionDeep}, proving the following result.}

\begin{theorem}\label{thm:existenceInWideAndDeep}
There exist deep and extremely wide fully-connected neural networks with sigmoid activation function such that the squared loss function of a finite data set has a non-attracting region of local minima (at finite parameter values).
\end{theorem}

\begin{corollary}
Any attempt to show for fully connected deep neural networks that a gradient descent technique will always lead to a global minimum only based on a description of the  loss curve will fail if it doesn't take into consideration properties of the learning procedure (such as the stochasticity of stochastic gradient descent), properties of a suitable initialization technique, or assumptions on the data set.
\end{corollary}

On the positive side, we point out that a stochastic method such as stochastic gradient descent has a good chance to escape a non-attracting region of local minima due to noise. With infinite time at hand and sufficient exploration, the region can be escaped from with high probability \citep[see][for a more detailed discussion]{Wei}. In Section~\ref{sct:characterization} we will further characterize when the method used to construct examples of regions of non-attracting local minima is applicable. This characterization limits us to the construction of extremely degenerate examples. We argue why assuring the necessary assumptions for the construction becomes difficult for wider and deeper networks and why it is natural to expect a lower suboptimal loss (where the suboptimal minima are less ``bad'') the less degenerate the constructed minima are and the more parameters a neural network possesses.

A different type of non-attracting regions of local minima is considered for the 2---3---1 XOR network by \citet{Sprinkhuizen}, where the region of local minima (of higher loss than the global loss) resides at points in parameter space with some coordinates being infinite. {For this, we consider the extended parameter space, where parameters can take on values $\pm \infty$. The standard topology on this space considers open neighborhoods of $\infty$ defined by sets $\{v\ |\ v>a\}$ for some $a$. A \textdef{local minimum at infinity} then satisfies, by definition, that for sufficiently large values of a parameter, the loss is higher at finite values than at the limit as the parameter tends to infinity.} In particular, a gradient descent approach may lead to diverging parameters in that case. However, a different non-increasing path down to the global minimum always exists for the network. It can be shown that such {generalized} local minima at infinity also exist for deep neural networks. (Our proof uses similar ideas as the proof for the 2---3---1- XOR network by \citet[Section III]{Sprinkhuizen}, but needs additional arguments due to a more general setting. The proof can be found in Appendix~\ref{app:infinity}.)

\begin{restatable}{theorem}{infinity}
   \label{thm:infinity}
Let $\Loss$ denote the squared loss of a fully connected regression neural network with sigmoid activation functions, having at least one hidden layer and each hidden layer containing at least two neurons. Then, for almost every finite data set, the loss function $\Loss$ possesses a {generalized local minimum in the extended parameter space with some coordinates being infinite}. The {generalized} local minimum is suboptimal whenever data set and neural network are such that a constant function is not an optimal solution. 
\end{restatable}

\subsection{Non-increasing Path to a Global Minimum}

By definition, all points belonging to a non-attracting region of local minima $R$ are local minima with the same loss value. Further, being non-attractive means that every neighborhood of $R$ contains points from where a non-increasing 
path to a value less than the value of the region exists. 
The question therefore arises under what conditions there is such a non-increasing path all the way down to a global minimum from almost everywhere in parameter space. The measure-theoretic term \textdef{almost everywhere} here refers to the Lebesgue measure, i.e., a condition holds almost everywhere when it holds for all points except for a set of Lebesgue measure zero. If the last hidden layer is the extremely wide layer having more neurons than input patterns (for example consider an extremely wide two-layer neural network), then indeed it holds true that non-increasing paths to the global minimum exist from almost everywhere in parameter space by the results of \citet{NguyenHein} (and \citet{GoriTesi,Poston}). We show the same conclusion to hold for extremely wide deep  neural networks, whenever the sequence of hidden dimensions is non-increasing, $n_{l+1}\leq n_l$, for all layers following the wide layer.

\begin{restatable}{theorem}{pathToGlobal}\label{thm:pathToGlobal}
Consider a fully connected regression neural network with activation function in the class $\mathcal{A}$ (as defined in the beginning of Section~\ref{sct:problemDef}) equipped with the squared loss function for a finite data set. Assume that {a} hidden layer contains more neurons than the number of input patterns {and the sequence of dimensions of all subsequent layers is non-increasing}. Then, for each set of parameters $\w$ and all $\epsilon>0$, there is $\w'$ such that $||\w-\w'||<\epsilon$ and such that a path, non-increasing in loss from $\w'$ to a global minimum (where $f(x_\alpha)=y_\alpha$ for each $\alpha$), exists.
\end{restatable}

\begin{corollary}
Consider an extremely wide, fully connected regression neural network with {non-increasing hidden dimensions following the wide layer,} activation function in the class $\mathcal{A}$ and trained to minimize the squared loss over a finite data set. Then all suboptimal local minima are contained in a non-attracting region of local minima.
\end{corollary}

The rest of the paper contains the arguments leading to the given results and an experimental construction of local minima in a deep and wide network. 

\section{Notation}

We fix additional notation aside the problem definition from Section~\ref{sct:problemDef}. For input $x_\alpha$ we denote the pattern vector of values at all neurons at layer $l$ before activation by $\n{l}{x_\alpha}$ and after activation by $\act{l}{x_\alpha}$.

In general, we will denote column vectors of size $n$ with coefficients $z_i$ by $[z_i]_{1\leq i\leq n}$ or simply $[z_i]_i$ and matrices with entries $a_{i,j}$ at position $(i,j)$ by $[a_{i,j}]_{i,j}$.  The neuron value pattern $\n{l}{x}$ is then a vector of size $n_l$ denoted by $\n{l}{x}=[\n{l,k}{x}]_{1\leq k \leq n_l}$, and the activation pattern $\act{l}{x}=[\act{l,k}{x}]_{1\leq k \leq n_l}$. 

For a fixed data point $x_\alpha$, we will further denote the squared loss on it by $\loss_\alpha$. The loss $\loss_\alpha$ can be considered as a function of the neuron values $\n{l,k}{x_\alpha}$, so that we consider partial derivatives of the loss $\loss_\alpha$ with infinitesimal changes of neuron values at $\n{l,k}{x_\alpha}$. 
For convenience of the reader, a tabular summary of all notation is provided in Appendix~\ref{app:notation}--3.

\section{Construction of Local Minima}\label{sct:construction}

We recall the construction of suboptimal local minima given by \citet{FukumizuAmari} and extend it to deep networks. Once we have fixed a layer $l$, we denote the parameters of the incoming linear transformation by $[u_{p,i}]_{p,i}$, so that $u_{p,i}$ denotes the contribution of neuron $i$ in layer $l-1$ to neuron $p$ in layer $l$, and the parameters of the outgoing linear transformation by $[v_{s,q}]$, where $v_{s,q}$ denotes the contribution of neuron $q$ in layer $l$ to neuron $s$ in layer $l+1$. For weights of the output layer (into a single neuron), we write $w_{\parm,j}$ instead of $w_{1,j}$. For the construction of \textdef{critical points} (i.e., points with gradient zero), we add one additional neuron $\n{l,-1}{x}$ to a hidden layer $l$. (Negative indices are unused for neurons, which allows us to add a neuron with this index.) 
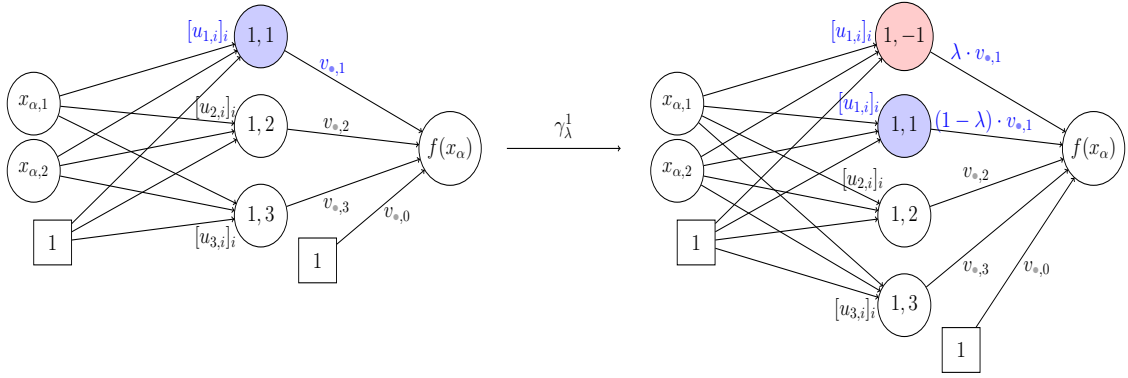
\begin{figure}
\centering
\resizebox{15cm}{5cm}{
\begin{tikzpicture}[state/.style={circle, draw, minimum size=1.4cm}, state2/.style={rectangle, draw, minimum size=1.0cm}]

\node[state, circle,draw] (x1) at (0,3.5) {\Large$x_{\alpha,1}$};
\node[state,circle,draw] (x2) at (0,2) {\Large$x_{\alpha,2}$};
\node[state2,rectangle,draw] (x3) at (0.5,0.4) {\Large $1$};

\node[state,circle,draw,fill=blue!20!] (a32) at (6,5) {\Large$1,1$};
\node[state,circle,draw] (a33) at (6,3) {\Large$1,2$};
\node[state,circle,draw] (a34) at (6,1) {\Large$1,3$};
\node[state2,rectangle] (a35) at (7.5,0) {\Large$1$};

\node[state,circle,draw] (z41) at (11.0,2.5) {\Large$f(x_\alpha)$};

\node[circle](u1) at (4.6, 5.1){\Large \textcolor{blue}{$[u_{1,i}]_i$}};
\node[circle](u2) at (4.8, 3.4){\Large$[u_{2,i}]_i$};
\node[circle](u3) at (4.8, 0.5){\Large$[u_{3,i}]_i$};

\node[circle](v1) at (7.9, 4.3){\Large \textcolor{blue}{\Large $v_{\parmBlue,1}$}};
\node[circle](v2) at (8.0, 3.0){\Large$v_{\parm,2}$};
\node[circle](v3) at (8.0, 1.2){\Large$v_{\parm,3}$};
\node[circle](v4) at (9.6, 1.0){\Large$v_{\parm,0}$};

\draw[->] (x1) -- (a32);
\draw[->] (x1) -- (a33);
\draw[->] (x1) -- (a34);
\draw[->] (x2) -- (a32);
\draw[->] (x2) -- (a33);
\draw[->] (x2) -- (a34);
\draw[->] (x3) -- (a32);
\draw[->] (x3) -- (a33);
\draw[->] (x3) -- (a34);

\draw[->]  (a32) -- (z41);
\draw[->]  (a33) -- (z41);
\draw[->]  (a34) -- (z41);
\draw[->]  (a35) -- (z41);

\node at (14.0, 3.0){\Large$\gamma_\lambda^1$};
\draw[->] (12.5,2.5) -- (15.5,2.5);

\node[state, circle,draw] (x1) at (17,3.5) {\Large$x_{\alpha,1}$};
\node[state,circle,draw] (x2) at (17,2) {\Large$x_{\alpha,2}$};
\node[state2,rectangle,draw] (x3) at (17.5,0.4) {\Large $1$};

\node[state,circle,draw,fill=red!20!] (a31) at (23,5) {\Large$1,-1$};
\node[state,circle,draw,fill=blue!20!] (a32) at (23,3) {\Large$1,1$};
\node[state,circle,draw] (a33) at (23,1) {\Large$1,2$};
\node[state,circle,draw] (a34) at (23,-1) {\Large$1,3$};
\node[state2,rectangle] (a35) at (24.5,-2) {\Large$1$};

\node[state,circle,draw] (z41) at (28.0,2.5) {\Large$f(x_\alpha)$};

\node[circle](u0) at (21.6, 5.1){\Large \textcolor{blue}{$[u_{1,i}]_i$}};
\node[circle](u1) at (21.8, 3.5){\Large \textcolor{blue}{$[u_{1,i}]_i$}};
\node[circle](u2) at (21.9, 1.8){\Large$[u_{2,i}]_i$};
\node[circle](u3) at (21.7, -1.1){\Large$[u_{3,i}]_i$};

\node[circle](v0) at (24.9, 4.6){\Large \textcolor{blue}{$\lambda\cdot v_{\parmBlue,1}$}};
\node[circle](v1) at (25.1, 3.1){\Large \textcolor{blue}{$(1-\lambda)\cdot v_{\parmBlue,1}$}};
\node[circle](v2) at (24.9, 1.9){\Large$v_{\parm,2}$};
\node[circle](v3) at (24.9, -0.2){\Large$v_{\parm,3}$};
\node[circle](v4) at (26.5, -0.2){\Large$v_{\parm,0}$};

\draw[->] (x1) -- (a31);
\draw[->] (x1) -- (a32);
\draw[->] (x1) -- (a33);
\draw[->] (x1) -- (a34);
\draw[->] (x2) -- (a31);
\draw[->] (x2) -- (a32);
\draw[->] (x2) -- (a33);
\draw[->] (x2) -- (a34);
\draw[->] (x3) -- (a31);
\draw[->] (x3) -- (a32);
\draw[->] (x3) -- (a33);
\draw[->] (x3) -- (a34);

\draw[->]  (a31) -- (z41);
\draw[->]  (a32) -- (z41);
\draw[->]  (a33) -- (z41);
\draw[->]  (a34) -- (z41);
\draw[->]  (a35) -- (z41);

\end{tikzpicture}
}
\caption{Embedding a smaller two-layer neural network into a larger one. Weights of the larger network are defined by the weights of the smaller network and the embedding map $\gamma_\lambda^1$. Numbers in hidden nodes (circles) denote the index of a neuron in form of (layer, neuron index) {with negative index for the added neuron}. Rectangles correspond to bias terms.}
\label{fig:gamma}
\end{figure}

A function $\gamma^r_\lambda$ describes the mapping from the parameters of the original network to the parameters after adding a neuron $\n{l,-1}{x}$. For a chosen neuron with index $r$ in layer $l$ of the smaller network, $\gamma_\lambda^r$ is determined by incoming weights $u_{-1,i}$ into $\n{l,-1}{x}$, outgoing weights $v_{s,-1}$ of $\n{l,-1}{x}$, and a change of the outgoing weights $v_{s,r}$ of $\n{l,r}{x}$. Sorting the network parameters in a convenient way, the embedding of the smaller network into the larger one is given, for any $\lambda \in \RR$, by a function $\gamma_\lambda^r$ mapping parameters $\{( [u_{r,i}]_i,[v_{s,r}]_s,\wrest\}$ of the smaller network to parameters $\{ ([u_{-1,i}]_i,[v_{s,-1}]_s, [u_{r,i}]_i,[v_{s,r}]_s,\wrest)\}$ of the larger network and is defined by 
$$\gamma_\lambda^r \left ([u_{r,i}]_i,[v_{s,r}]_s,\wrest\right ):=\left ([u_{r,i}]_i,[\lambda\cdot v_{s,r}]_s,[u_{r,i}]_i,[(1-\lambda) \cdot v_{s,r}]_s, \wrest \right ).$$
Here $\wrest$ denotes the collection of all remaining network parameters, i.e., all $[u_{p,i}]_i,[v_{s,q}]_s$ for $p,q\notin \{-1,r\}$ and all parameters from linear transformation of layers with index smaller than $l$ or larger than $l+1$, if existent. A visualization of $\gamma_\lambda^1$ is shown in \figurename~\ref{fig:gamma}.
\vspace{0.2cm}

 \textit{Important fact:}
For the network functions $\varphi,f$ of smaller and larger network at parameters $([{u}^*_{1,i}]_i,[{v}^*_{s,1}]_s,\wrest^*)$ and $\gamma_\lambda^r( [u^*_{r,i}]_i,[ v^*_{s,r}]_s,\wrest^* )$ respectively, we have $\varphi(x)=f(x)$ for all $x$. More generally, the activation values of all neurons in the smaller network agree with the activation values of corresponding neuron in the larger network, i.e., $\nphi{l,k}{x}=\n{l,k}{x}$ and $\actphi{l,k}{x}=\act{l,k}{x}$ for all $l,x$ and $k\geq 0$.

\subsection{Characterization of Critical Points Constructed Hierarchically by $\gamma$}\label{sct:characterization}

Using some $\gamma_\lambda^r$ to embed a smaller deep neural network into a larger one with one additional neuron, it has been shown that critical points get mapped to critical points. 

\begin{theorem}[\citet{Nitta}]\label{thm:nitta}
Consider two neural networks as in Section~\ref{sct:problemDef}, which differ by one neuron in layer $l$ with index $\n{l,-1}{x}$ in the larger network. If parameter choices $([u^*_{r,i}]_{i},[v^*_{s,r}]_{s},\wrest^*)$ determine a critical point for the squared loss over a finite data set in the smaller network then, for each $\lambda\in\RR$, $\gamma_\lambda^r([u^*_{r,i}]_i,[v^*_{s,r}]_s,\wrest^*)$ determines a critical point in the larger network.
\end{theorem}

As a consequence, whenever an embedding of a local minimum with $\gamma_\lambda^r$ into a larger network does not lead to a local minimum, then it leads to a \textdef{saddle point} instead, i.e., critical points where the Hessian has both strictly positive and strictly negative eigenvalues. (There are no local maxima in the networks we consider, since the loss function is convex with respect to the parameters of the last layer.) For shallow neural networks with one hidden layer, it was characterized  when a critical point leads to a local minimum.

\begin{theorem}[\citet{FukumizuAmari}]\label{thm:fukumizu}
Consider two neural networks as in Section~\ref{sct:problemDef} with only one hidden layer and which differ by one neuron in the hidden layer with index $\n{1,-1}{x}$ in the larger network. Assume that parameters $([{u}^*_{r,i}]_{i},{v}^*_{\parm,r},\wrest^*)$ determine an isolated local minimum for the squared loss over a finite data set in the smaller neural network and that $\lambda\notin\{0,1\}$. 

Then  $\gamma_\lambda^r([u^*_{r,i}]_i,v^*_{\parm,r},\wrest^*)$ determines a local minimum in the larger network if the matrix $[B_{i,j}^r]_{i,j}$ given by 
$$B_{i,j}^r= \sum_{\alpha} (f(x_\alpha)-y_\alpha) \cdot {v}^*_{\parm,r}\cdot \sigma''(\n{1,r}{x_\alpha}) \cdot x_{\alpha,i} \cdot x_{\alpha,j}$$
is positive definite and $0<\lambda<1$, or if $[B_{i,j}^r]_{i,j}$ is negative definite and $\lambda<0$ or $\lambda>1$. (Here, we denote the $k$-th input dimension of input $x_\alpha$ by $x_{\alpha,k}$.)
\end{theorem}

We extend the previous theorem to a characterization in the case of deep neural networks. We note that a similar computation has been previously performed for neural networks with two hidden layers by \citet{MizutaniDreyfus}.

\begin{restatable}{theorem}{construction}\label{thm:construction}
Consider two (possibly deep) neural networks as in Section~\ref{sct:problemDef}, which differ by one neuron in layer $l$ with index $\n{l,-1}{x}$ in the larger network. Assume that the parameter choices $([{u}^*_{r,i}]_{i},[{v}^*_{s,r}]_{s},\wrest^*)$ determine an isolated local minimum for the squared loss over a finite data set in the smaller network
. If the matrix $[B_{i,j}^r]_{i,j}$ defined by 
\begin{equation}\label{eq:B}\begin{split}
B_{i,j}^r:= \sum_{\alpha}  \sum_k &\frac{\partial \loss_\alpha}{\partial \n{l+1,k}{x_\alpha}} \cdot  v^*_{k,r} \cdot \sigma''(\n{l,r}{x_\alpha} ) \cdot \act{l-1,i}{x_\alpha}\cdot \act{l-1,j}{x_\alpha}
\end{split} \end{equation}
is either
\begin{itemize}
\item positive definite and $\lambda\in \mathcal{I}:= (0,1)$, or 
\item negative definite and $\lambda \in \mathcal{I}:= (-\infty,0)\cup (1,\infty),$
\end{itemize}
then $\left \{ \gamma_\lambda^r([{u}^*_{r,i}]_i,[{v}^*_{s,r}]_s,\wrest^*)\ |\ \lambda \in \mathcal{I} \right \}$ determines a non-attracting region of local minima in the larger network if and only if 
\begin{equation}\label{eq:D}
D_{i}^{r,s}:= \sum_\alpha  \frac{\partial \loss_\alpha}{\partial \n{l+1,s}{x_\alpha}}  \cdot \sigma'(\n{l,r}{x_\alpha}) \cdot \act{l-1,i}{x_\alpha}
\end{equation}
is zero, $D_{i}^{r,s}=0$, for all $i,s$. 
\end{restatable}


\begin{remark}
In the case of a neural network with only one hidden layer as considered in Theorem~\ref{thm:fukumizu}, the partial derivative $\frac{\partial \loss_\alpha}{\partial \n{l+1,s}{x_\alpha}}$ reduces to the residual $(f(x_\alpha)-y_\alpha)$ and the matrix $[B_{i,j}^r]_{i,j}$ in Equation~\ref{eq:B} reduces to the matrix $[B_{i,j}^r]_{i,j}$ in Theorem~\ref{thm:fukumizu}. The condition that $D_{i}^{r,s}=0$ for all $i,s$ does hold for shallow neural networks with one hidden layer as we show in Proposition~\ref{prop:constructionDeep}$(i)$. This proves Theorem~\ref{thm:construction} to be consistent with Theorem~\ref{thm:fukumizu}.
\end{remark}

\begin{remark}
The assumption of starting from an \textit{isolated} local minimum can be relaxed by special consideration of the degenerate directions (the eigenvectors of the Hessian with eigenvalue zero), which we will take advantage of.
\end{remark}

The theorem follows from a careful computation of the Hessian of the loss function $\Loss(\w)$ (i.e., we calculate the matrix of second order derivatives of the loss function with respect to the network parameters), characterizing when it is positive (or negative) semidefinite and checking that the loss function does not change along directions that correspond to an eigenvector of the Hessian with eigenvalue 0. We state the outcome of the computation of the loss Hessian in Lemma~\ref{lma:hessian} and refer the reader interested in a full proof of Theorem~\ref{thm:construction} to Appendix~\ref{app:construction}.

\begin{restatable}{lem}{hessian}\label{lma:hessian}
Consider two (possibly deep) neural networks as in Section~\ref{sct:problemDef}, which differ by one neuron in layer $l$ with index $\n{l,-1}{x}$ in the larger network. Fix $1\leq r \leq n_l$. Assume that the parameter choices $([{u}^*_{r,i}]_{i},[{v}^*_{s,r}]_{s},\wrest^*)$ determine a critical point in the smaller network. 

Let $\Loss$ denote the loss function of the larger network and $\loss$ the loss function of the smaller network. Let $\alpha\neq - \beta \in \RR$ such that $\lambda=\frac{\beta}{\alpha+\beta}$.

With respect to the basis of the parameter space of the larger network given by $( [u_{-1,i} + u_{r,i}]_i,[v_{s,-1} + v_{s,r}]_s, \wrest, [\alpha\cdot u_{-1,i} - \beta \cdot u_{r,i}]_i,[v_{s,-1} - v_{s,r}]_s) $,
the Hessian of the loss $\Loss$ at $\gamma_\lambda^r([{u}^*_{r,i}]_{i},[{v}^*_{s,r}]_{s},\wrest^*)$ is given by 

$$\begin{pmatrix}
[\frac{\partial^2 \loss}{\partial u_{r,i} \partial u_{r,j} }]_{i,j}  & 2 [\frac{\partial^2 \loss}{\partial u_{r,i} \partial v_{s,r} }]_{i,s}& [\frac{\partial^2 \loss}{\partial \wrest\ \partial u_{r,i} }]_{i,\wrest}  &0  &0  \\ 
2[\frac{\partial^2 \loss}{\partial u_{r,i} \partial v_{s,r} }]_{s,i} & 4 [\frac{\partial^2 \loss}{\partial v_{s,r} \partial v_{t,r} }]_{s,t} &  2[\frac{\partial^2 \loss}{\partial \wrest\ \partial v_{s,r} }]_{s,\wrest}  &  (\alpha-\beta)[D_{i}^{r,s}]_{s,i} &  0  \\
[\frac{\partial^2 \loss}{\partial \wrest\ \partial u_{r,i} }]_{\wrest,i} &  2[ \frac{\partial^2 \loss}{\partial \wrest\ \partial v_{s,r} }]_{\wrest,s} & [\frac{\partial^2 \loss}{\partial \wrest\ \partial \wrest' }]_{\wrest,\wrest'} & 0  & 0   \\
0 &  (\alpha-\beta)[D_{i}^{r,s}]_{i,s} & 0 & \alpha \beta [B_{i,j}^r]_{i,j}  &   (\alpha + \beta) [ D_{i}^{r,s}]_{i,s}  \\
0 &  0 & 0 & (\alpha + \beta) [ D_{i}^{r,s}]_{s,i} & 0  \\
\end{pmatrix}
$$

\end{restatable}

\subsection{Shallow Networks with a Single Hidden Layer}


For the construction of suboptimal local minima in extremely wide two-layer networks, we begin by following the experiments of \citet{FukumizuAmari} that prove the existence of suboptimal local minima in (non-wide) two-layer neural networks. 

Consider a neural network of size 1---2---1. We use the corresponding network function $f$ to construct a data set $(x_\alpha,y_\alpha)_{\alpha=1}^N$ by randomly choosing $x_\alpha$ and letting $y_\alpha=f(x_\alpha)$. By construction, we know that a neural network of size 1---2---1 can perfectly fit the data set with zero error.

Consider now a smaller network of size 1---1---1 having too little expressibility for a global fit of all data points. We find parameters  $[{u}^*_{1,1},{v}^*_{\parm,1}]$ where the loss function of the neural network is in an isolated local minimum with non-zero loss. For this small example, the required positive definiteness of $[B_{i,j}^1]_{i,j}$ from Equation~\ref{eq:B} for a use of $\gamma_\lambda^1$ with $\lambda\in(0,1)$ reduces to checking a real number $B_{1,1}^1$ for positivity. {An empirical example is easily found, and we assume the positivity condition to hold true.} We can now apply $\gamma_\lambda^1$ and Theorem~\ref{thm:fukumizu} to find parameters for a neural network of size 1---2---1 that determine a suboptimal local minimum. This concludes the construction of \citet{FukumizuAmari}. The obtained network now serves as a base for a proof by induction to show that suboptimal local minima also exist in arbitrarily wide neural networks.

\begin{theorem}\label{thm:twoLayerWide}
There is an extremely wide two-layer neural network with sigmoid activation functions and arbitrarily many neurons in the hidden layer that has a non-attracting region of suboptimal local minima.
\end{theorem}

\begin{proof} 
Having already established the existence of parameters for a (small) neural network leading to a suboptimal local minimum, it suffices to note that iteratively adding neurons using Theorem~\ref{thm:fukumizu} is possible. Iteratively at step $t$, we add a neuron $\n{1,-t}{x}$ {(negatively indexed)} to the network by an application of $\gamma_\lambda^1$ with the same $\lambda\in(0,1)${, using the same neuron with index $1$ in each iteration}. The corresponding matrix from Equation~\ref{eq:B} remains a positive number $$B_{1,1}^{1,(t)} = \sum_{\alpha} (f(x_\alpha)-y_\alpha) \cdot (1-\lambda)^{t-1}\cdot {v}^*_{\parm,1}\cdot \sigma''(\n{l,1}{x_\alpha}) \cdot x_{\alpha,1}^2.$$ (We use here that neither $f(x_\alpha)$ nor $\n{l,1}{x_\alpha}$ ever change during this construction { and that the outgoing weight from the hidden neuron with index $1$ changes by multiplication with $(1-\lambda)$}.) {The idea is to apply Theorem~\ref{thm:fukumizu} to guarantee that adding neurons as above always leads to a suboptimal minimum with nonzero loss for the network for $\lambda\in (0,1)$. The only problem is that the addition of previous neurons creates the possibility of reparameterizations that keep the loss unchanged, i.e., the assumption in Theorem~\ref{thm:fukumizu} of having an isolated minimum is violated. However, the computation of Lemma~\ref{lma:hessian} is still valid and to ensure a local minimum from a positive semi-definite Hessian it is only necessary to additionally make sure that no reduction is possible along directions defined by the kernel of the Hessian i.e., the space of eigenvectors of the Hessian with eigenvalue zero. Since all these eigenvectors correspond to reparameterizations that cannot be influenced by the additional weights in and out of newly added neurons, it can be seen that the loss is indeed locally constant into all directions defined by the kernel of the Hessian. In this case, positive semi-definiteness is sufficient to ensure a local minimum.  }

In particular, we may add an arbitrary number of neurons to the hidden layer and make the network extremely wide. Further, a continuous change of the  $\lambda$ belonging to the last added neuron via $\gamma_\lambda$ to a value outside of $[0,1]$ does not change the network function, but leads to a saddle point {(as we introduce a negative factor $\alpha\beta<0\Leftrightarrow \lambda\notin (0,1)$ into the calculation of the Hessian at the position of $\alpha\beta B_{1,1}^{1,(t)}$, see Lemma~\ref{lma:hessian})}. Hence, we found a non-attracting region of suboptimal minima. 
\end{proof}


\begin{remark}
Since we started the construction from a network of size 1---1---1, our constructed example is extremely degenerate: The suboptimal local minima of the wide network have identical incoming weight vectors for each hidden neuron. Obviously, the suboptimality of this parameter setting is easily discovered by inspection of the parameters. Also with proper initialization, the chance of landing in this local minimum is vanishing. 

However, one may also start the construction from a more complex network with a larger network with several hidden neurons. In this case, when adding a few more neurons using $\gamma_\lambda^1$, it is much harder to detect the suboptimality of the parameters from visual inspection. 
\end{remark}

\subsection{Deep Neural Networks}\label{sct:constructionDeep}

According to Theorem~\ref{thm:construction}, for deep networks there is a second condition for the construction of local minima using the map $\gamma_\lambda^r$. Next to positive definiteness of the matrix $B_{i,j}^r$ for some $r$ , we additionally require that $[D_i^{r,s}]_{i,s}=0$ for all $i,s$ and the same $r$.  We consider sufficient conditions for $D_i^{r,s}=0$. 

\begin{restatable}[]{prop}{constructionDeep}\label{prop:constructionDeep}
Suppose we have constructed a critical point of the squared loss of a neural network by starting from a local minimum of a smaller network and by adding a neuron into layer $l$ with index $\n{l,-1}{x}$ by application of the map $\gamma_\lambda^r$ to a neuron $\n{l,r}{x}$. Suppose further that for the outgoing weights $v^*_{s,r}$ of $\n{l,r}{x}$ we have $\sum_s {v}^*_{s,r}\neq 0$ , and suppose that $D_i^{r,s}$ is defined as in Equation~\ref{eq:D}. Then $D_i^{r,s}=0$  if one of the following holds.
\begin{itemize}
\item[(i)] The layer $l$ is the last hidden layer. (This condition includes the case $l=1$ indexing the hidden layer in a two-layer network.)
\item [(ii)] For all $t,t',\alpha$, we have \[\frac{\partial \loss_\alpha}{\partial \n{l+1,t}{x_\alpha}}= \frac{\partial\loss_\alpha}{\partial \n{l+1,t'}{x_\alpha}}.\] 
\item [(iii)] For each $\alpha$ and each $t$, $$\frac{\partial \loss_\alpha}{\partial \n{l+1,t}{x_\alpha}}=0.$$ 
(This condition holds in the case of the weight infinity attractors in the proof to Theorem~\ref{thm:infinity} for $l+1$ the second last layer. It also holds in a global minimum.)
\end{itemize}
\end{restatable}


The proof of the proposition is contained in Appendix~\ref{app:constructionDeep}, and we apply it to construct a non-attracting region of suboptimal local minima in a deep and wide neural network.

\begin{proof}\textbf{{(Theorem~\ref{thm:existenceInWideAndDeep})}}
{For simplicity of the presentation, we show the existence of local minima for a three-layer neural network, but the same construction naturally generalizes to deeper networks. We begin with a regression network of size $2$---$n_1$---$n_2$---1 for input dimension $n_0=2$, hidden layers of dimension $n_1,n_2$, and the output layer mapping into $\RR$. We use this network to construct a finite data set and train a network of size 2---1---1---1 with one neuron in each hidden layer on this data set to find an isolated local minimum. }

{In Equations~\ref{eq:B} and $\ref{eq:D}$ we have suppressed the choice of the layer to simplify the notation. To distinguish layers here, we write $[B^r_{i,j}(1)]_{i,j}$, $[D_{i}^{r,s}(1)]_{i,s}$ and $[B^r_{i,j}(2)]_{i,j}$, $[D_{i}^{r,s}(2)]_{i,s}$ for the matrices of the first and second hidden layer respectively.  We assume that the local minimum satisfies that $[B^1_{i,j}(1)]_{i,j}$ is positive definite and $B^1_{1,1}(2)>0$. Existence of such an example is easily verified empirically and we provide an example in the following section. 
Starting with the first hidden layer, the condition of Proposition~\ref{prop:constructionDeep}~(ii) is trivially satisfied as there is only one neuron in the second layer. Hence $D_1^{1,1}(1)=0$ and we can we can iteratively add arbitrarily many neurons to the first hidden layer by repetitively applying $\gamma_\lambda^1$ on $\n{1,1}{x}$ and Theorem~\ref{thm:construction} (with $\lambda \in (0,1)$). (Precisely as in the iterative application of Theorem~\ref{thm:fukumizu} in the proof of Theorem~\ref{thm:twoLayerWide}, we can see that in this iterative process the assumption of starting from an isolated local minimum can be relaxed.) The relevant matrices are given at each step $t$ by $[B^{1,(t)}_{i,j}(1)]_{i,j}=(1-\lambda)^{t-1} [B^{1,(1)}_{i,j}(1)]_{i,j}$ which is positive definite, and $D_1^{1,1}(1)$ does not change and remains zero. We can therefore find a local minimum in a network of size 2---$m_1$---1---1 for any $m_1$. }

{For the second layer, the matrix $B(2):=[B^1_{i,j}(2)]_{i,j}$ equals the constant matrix of size $(m_1\times m_1)$ with entry the positive number $B_{1,1}^1(2)$ calculated before adding neurons to the first layer and is positive semidefinite. Now, Proposition~\ref{prop:constructionDeep}~(i) applies to show that we can apply Theorem~\ref{thm:construction} to iteratively add arbitrarily many neurons to the second hidden layer using $\gamma_\lambda^1$ on $\n{2,1}{x}$. When adding the $t\text{-th}$ neuron to the second layer, we have  $[B^{1,(t)}_{i,j}(2)]_{i,j}= (1-\lambda)^{t-1}  B(2)$, which is positive semidefinite. Since the matrix $B(2)$ is only positive semidefinite and not positive definite, we again need special consideration to the newly added degenerate directions: The first layer contains identical neurons from enlarging that layer. The new degenerate directions are consequently given by increasing a weight connecting a new neuron in the second layer with one of these identical neurons in the first layer and decreasing the weight from the second identical neuron such that their contribution cancels out. 
This defines another reparameterization that leaves the loss constant and which does not interfere with any of the other reparameterizations. This excludes the possibility of loss reduction along degenerate directions and positive semi-definiteness is sufficient to guarantee a local minimum. }

{This leads to a local minimum in a network of size 2---$m_1$---$m_2$---1. If $m_1\geq n_1$,  $m_2\geq n_2$, then we know the local minimum to be suboptimal by construction. Adding sufficiently many neurons, we can construct the network to be extremely wide. A continuous change of the parameter $\lambda\in (0,1)$ used for adding the $t\text{-th}$ neuron,  changes the sign of the corresponding B-matrix for negative $\lambda$, which offers a direction in weight space for further reduction of the loss. This shows the existence of a non-attracting region in a deep and wide neural network, proving Theorem~\ref{thm:existenceInWideAndDeep}. }
\end{proof}

\subsection{Experiment for Deep Neural Networks}

\begin{figure*}[!t]
\centering
\minipage{0.49\textwidth}
\subfloat[]{\includegraphics[width=2.7in]{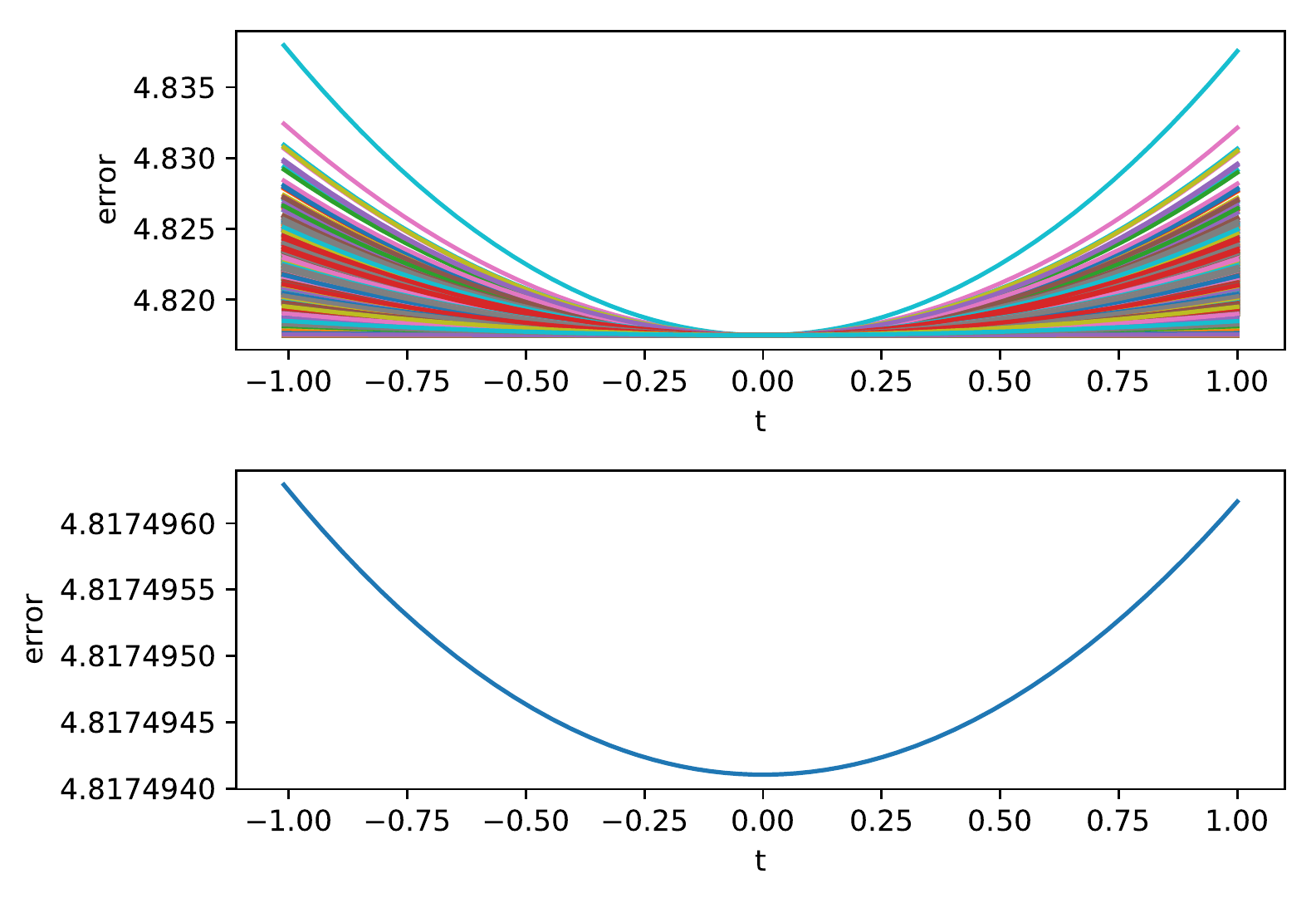}}
\hfill \subfloat[]{\includegraphics[width=2.7in]{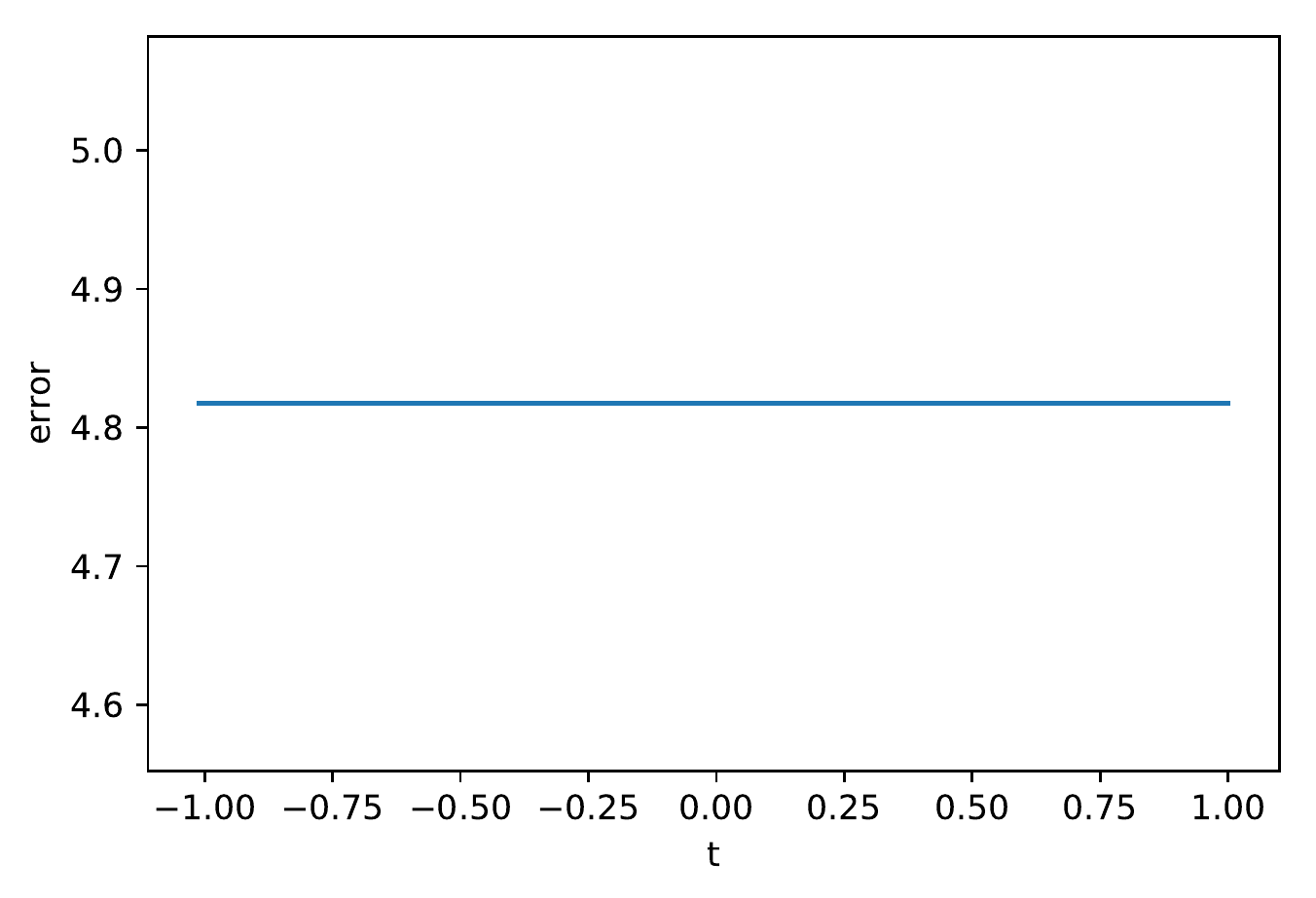}}%
\endminipage\hfill
\minipage{0.49\textwidth}
\subfloat[]{\includegraphics[width=2.7in]{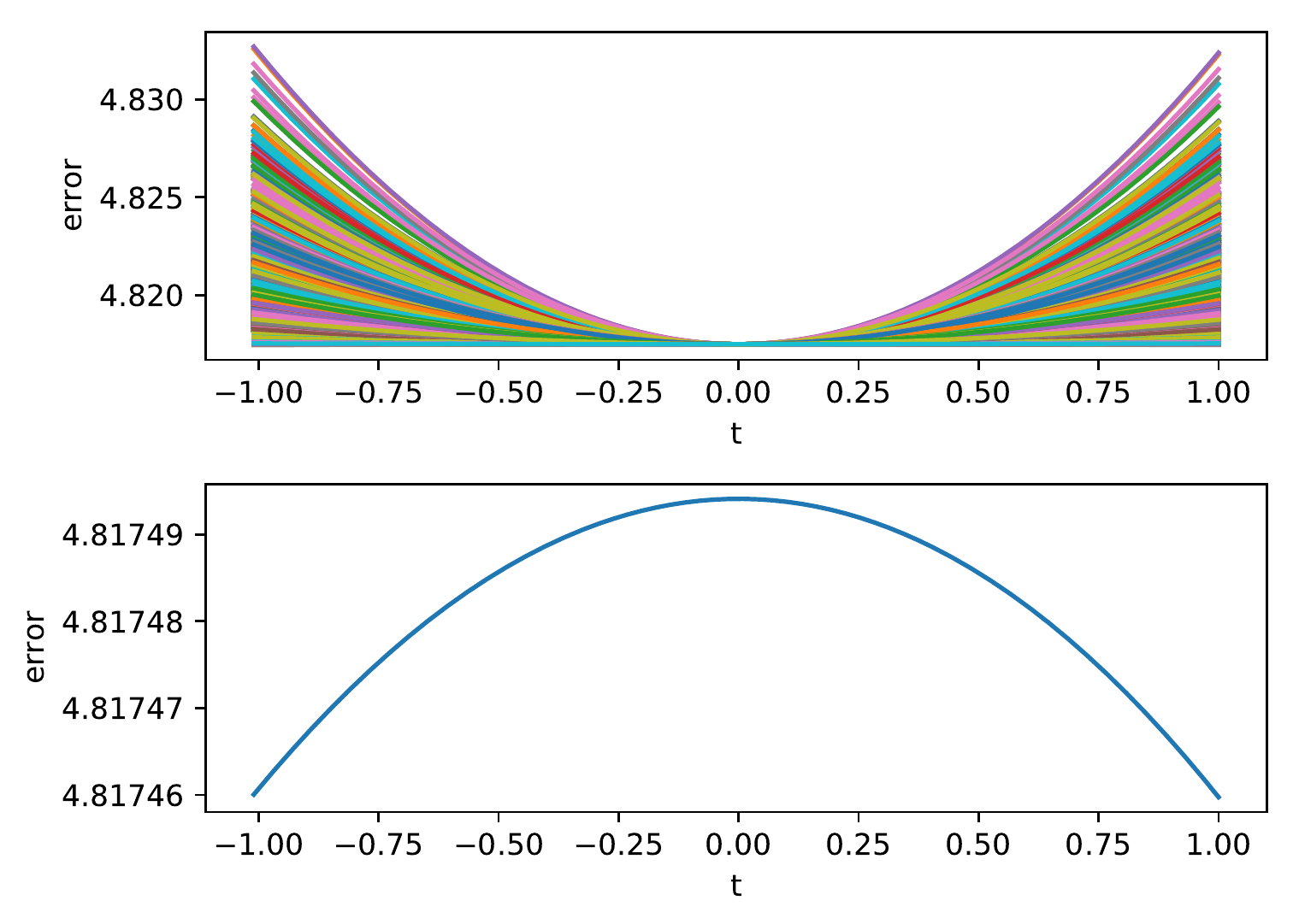}}%
\hfill \subfloat[]{\includegraphics[width=2.7in]{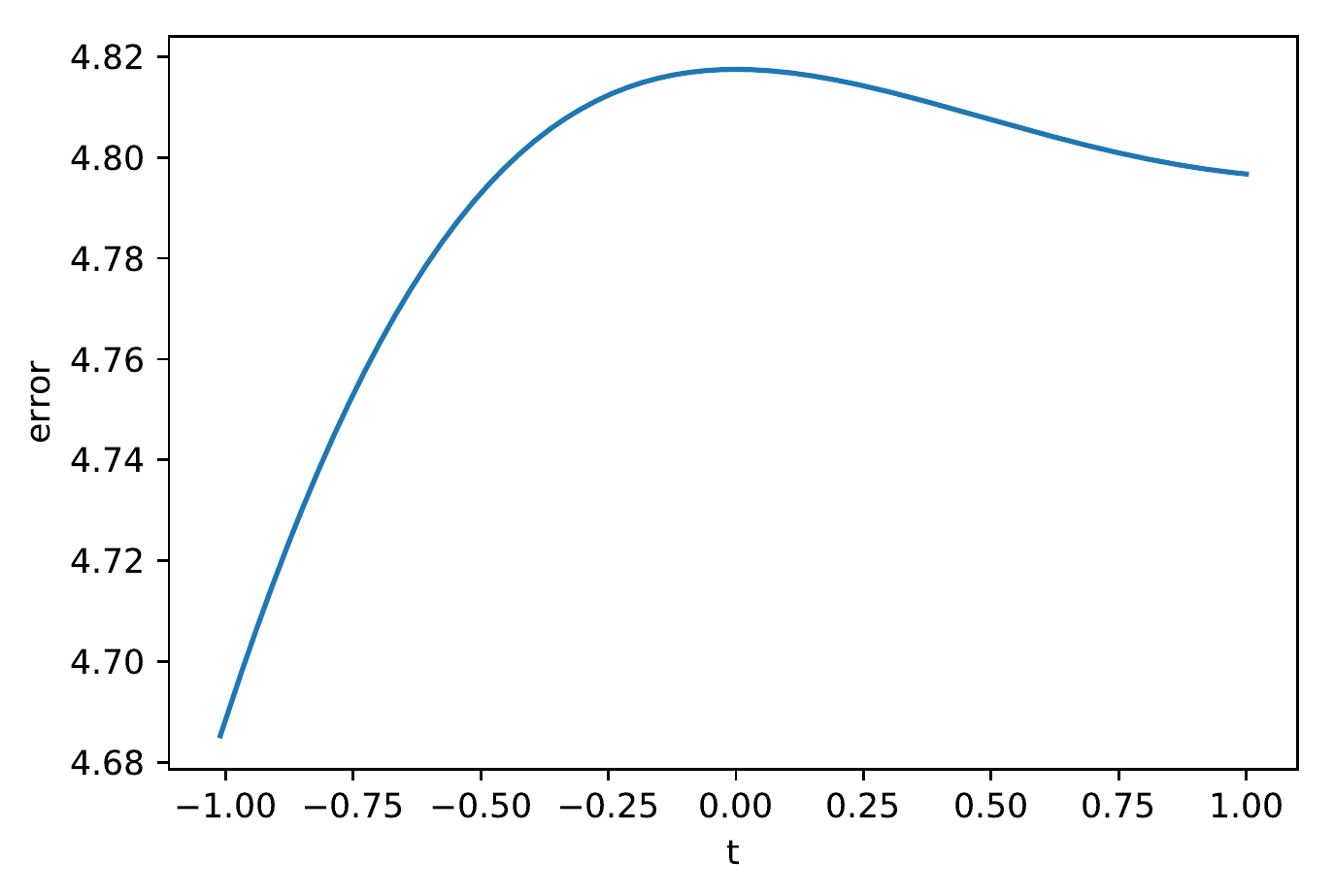}}%
\endminipage\hfill
 \caption{A non-attracting region of local minima in a deep and wide neural network.  (a) Local minimum. \textit{Top:} Loss evolution for 5000 random directions. \textit{Bottom:} Minimum over sampled directions. (b) Path along a degenerate direction to a saddle point. (c) Saddle point with the same loss value. \textit{Top:} Loss evolution for 5000 random directions. \textit{Bottom:} Minimum over sampled directions. (d) Error evolution along an analytically known direction of descent at the saddle point.}\label{fig:experiment}
\end{figure*} 

{We empirically validate the construction of a suboptimal local minimum in a deep and extremely wide neural network as given in the proof of Theorem~\ref{thm:existenceInWideAndDeep}.\footnote{The accompanying code can be found at \url{https://github.com/petzkahe/nonattracting_regions_of_local_minima.git}.} We start by considering a three-layer network of size 2---5---5---1, i.e., we have two input dimensions, one output dimension and hidden layers of five neurons. We use its network function $f$ to create a data set of 20 samples $\left (x_\alpha,f(x_\alpha)\right )$, hence we know that a network of size 2---5---5---1 can attain zero loss. }

{We initialize a new neural network of size 2---1---1---1 and train it until convergence to find a local minimum of total loss (sum over the $20$ data points) of $4.817$. We check for positive definiteness of the matrix $B_{i,j}^1(1)$ (eigenvalues here given by $0.0182$, $0.0004$) and positivity of $B_{1,1}^1(2)$ (here $0.0757$). Following the proof to Theorem~\ref{thm:existenceInWideAndDeep}, we add twenty neurons to both hidden layers to construct a local minimum in an extremely wide network of size 2---21---21---1. The local minimum must be suboptimal by construction of the data set.  Experimentally, we show not only that indeed we end up with a suboptimal minimum, but also that it belongs to a non-attracting region of local minima. }

{\figurename~\ref{fig:experiment} shows results of this construction. The plot in (a) shows the loss in the neighborhood of the local minimum in parameter space. The top image shows the loss curve into 5000 randomly generated directions, the bottom displays the minimal loss over all these directions. Evidently, we were not able to find a direction in parameter space that allows us to reduce the loss. The plot in (b) shows the change of loss along one of the degenerate directions that leads to a saddle point, which is shown in plot (c). Again, the top image shows the loss curve into 5000 randomly generated directions, and the bottom displays the minimum loss over all these directions. Most random directions in parameter space lead to an increase in loss, but the minimum loss shows the existence of directions that lead to a reduction of loss. In such a saddle point, we know a direction for loss reduction at this saddle point from Lemma~\ref{lma:hessian}.  The plot in (d) shows a significant reduction in loss for the analytically known direction. Being able to reach a saddle point from a local minimum by a path of non-increasing loss shows that we found a non-attracting region of local minima.}


\subsection{A Discussion of Limitations and of the Loss of Non-attracting Regions of Suboptimal Minima}
{We theoretically proved the possibility of suboptimal local minima in deep and wide networks and empirically validated their existence. Due to the degeneracy of the constructed examples, the following questions on consequences in practice remain unanswered. (i) How frequent are suboptimal local minima with high loss, and (ii) how degenerate must they be to have high loss?}

{Suppose we aim to find a suboptimal local minimum using the above construction in a network consisting of $L$ layers with hidden dimensions $n_l$, and we want the local minimum to be less degenerate than the above examples while still having considerably high loss. We therefore want to start from a local minimum in a smaller network with hidden dimension $\nu_i$ and add only a few neurons to each of the layers. We want each $\nu_i$ small enough to find a (non-degenerate) local minimum in the smaller network with large loss, but not too small so that the addition of several neurons renders the local minimum degenerate. For the construction to work, we need to find in each layer $l$ a neuron with index $r_l$ such that the $n_{l-1}$-many eigenvalues of the corresponding $B$-matrix are either all positive or all negative so that the matrix is either positive definite or negative definite (determining a suitable choice for $\lambda$ according to Theorem~\ref{thm:construction}). In addition, the same neuron must satisfy $D_i^{r_l,s}=0$ for all $i,s$, adding $n_{l-1}n_{l+1}$ many conditions.  To find such an example is difficult both theoretically (the sufficient conditions from Proposition~\ref{prop:constructionDeep} for a vanishing D-matrix are rather strong) as well as empirically. Whenever any of the necessary conditions is violated, then we cannot use the above construction to find a local minimum in a larger network. In other words, whenever we find a local minimum in the smaller network such that no neuron exists that satisfies all the necessary conditions (positive definiteness of $B$ and zero $D$-matrix), then the above construction leads to a saddle point. While these saddle points can be close to being a local minimum (in the sense that slight perturbations of the parameters yields a higher loss in most cases), there is at least one direction in parameter space with negative curvature of the loss surface. }

{It is therefore conceivable that suboptimal local minima with high loss are extremely rare in practical applications with sufficiently large networks, which is in line with empirical observations. From this perspective, our results suggest that for sufficiently wide deep networks \textit{almost} all local minima are global, but the existence of degenerate local minima makes a general statement impossible.}

\section{Proving the Existence of a Non-increasing Path to the Global Minimum}\label{sct:path}


In the previous section we showed the existence of non-attracting regions of local minima and Theorem~\ref{thm:infinity} showed the existence of generalized local minima at infinity that can cause divergent parameters during loss reduction. These type of local minima do not rule out the possibility of non-increasing paths to the global minimum from almost everywhere in parameter space.  In this section, we sketch the proof to Theorem~\ref{thm:pathToGlobal} illustrated in form of several lemmas, where up to the basic assumptions on the neural network structure as in Section~\ref{sct:problemDef} (with activation function in $\mathcal{A}$), the assumption of one lemma is given by the conclusion of the previous ones. A full proof can be found in Appendix~\ref{app:pathToGlobal}.



We consider vectors that we call activation vectors, different from the activation pattern vectors $\act{l}{x}$ from above. The \textdef{activation vector} at neuron $k$ in layer $l$ are denoted by and $\av^l_k$, and defined by all values at the given neuron for samples $x_\alpha$:
$$ \av^l_k:=[\act{l,k}{x_\alpha}]_\alpha.$$
In other words while we fix $l$ and $x$ for the activation pattern vectors $\act{l}{x}$ and let $k$ run over its possible values, we fix $l$ and $k$ for the activation vectors $\av^l_k$ and let $x$ run over its samples $x_\alpha$ in the data set. {We denote by $\av^l$ the matrix $[\av^l_k]_k=[\act{l,k}{x_\alpha}]_{k,\alpha}$ of size $(n_l \times N)$ containing the activations values of all neurons and samples at layer $l$. Similarly, we denote by $\nv^l$ the matrix $\nv^l=[\n{l,k}{x_\alpha}]_{k,\alpha}$  of size $(n_l \times N)$ containing the pre-activation neuron values for all neurons and samples at layer $l$}

The \textdef{first step} of the proof is to use the freedom given by $\epsilon>0$ to change the starting point in parameter space to satisfy that the activation vectors {$\av_k^{l^*}$ of the extremely wide layer $l^*$} span the entire space $\RR ^N$.

\begin{restatable}{lem}{linearIndependence}\label{lma:linearIndependence}
\citep[Corollary 4.5]{NguyenHein}
For each choice of parameters $\w$ and all $\epsilon>0$ there is $\w'$ such that (i) $||\w-\w'||<\epsilon$, (ii) the activation vectors {$\av^{l^*}_k$} of the extremely wide layer {$l^*$} (containing more neurons than the number of training samples $N$) at parameters $\w'$ satisfy {$$\spn_k \av^{l^*}_k=\RR ^N,$$} {and (iii) the weight matrices $(\w')^l$ have full rank for all $l> l^*+1$.}
\end{restatable}




{The \textdef{second step} of the proof is to guarantee that we can then induce any continuous change of activation vectors in layer $l^*+1$  by suitable paths in the parameter space changing only the weights of the same layer.} The following two lemmas ensure exactly that. We first consider pre-activation values and then consider the application of the activation function. We slightly abuse notation in the statement when adding a vector to a matrix, which shall mean the addition of the vector to all columns of the matrix.


\begin{figure*}[!t]
\centering
\subfloat{
\begin{tabular}{@{}c@{}}
    \includegraphics[width=2.7in]{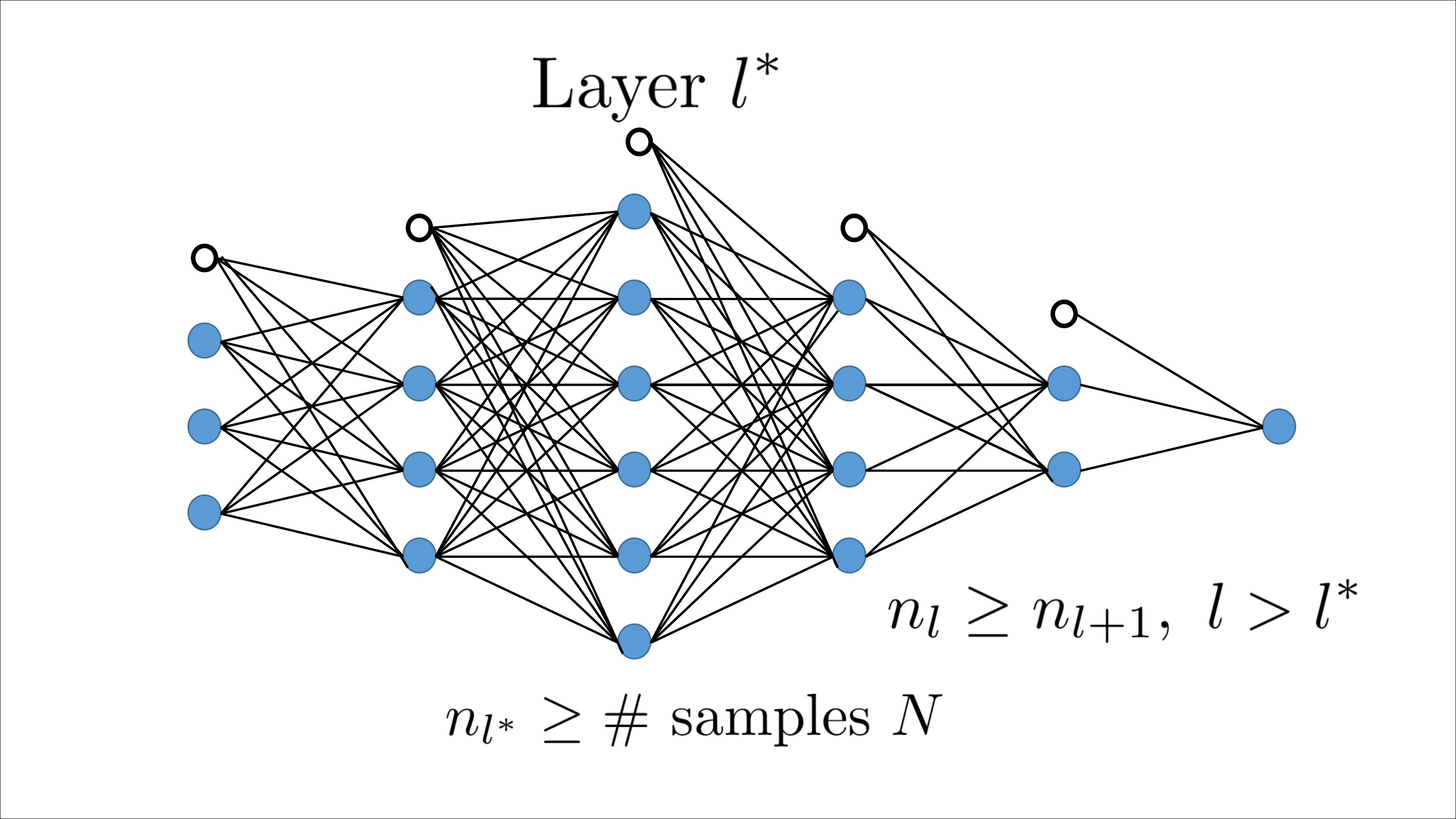} \\[\abovecaptionskip]
    \footnotesize Theorem~\ref{thm:pathToGlobal}: Visualization of considered  \\\footnotesize  neural network architecture. 
    \end{tabular}}
\hfil
\subfloat{
\begin{tabular}{@{}c@{}}
    \includegraphics[width=2.7in]{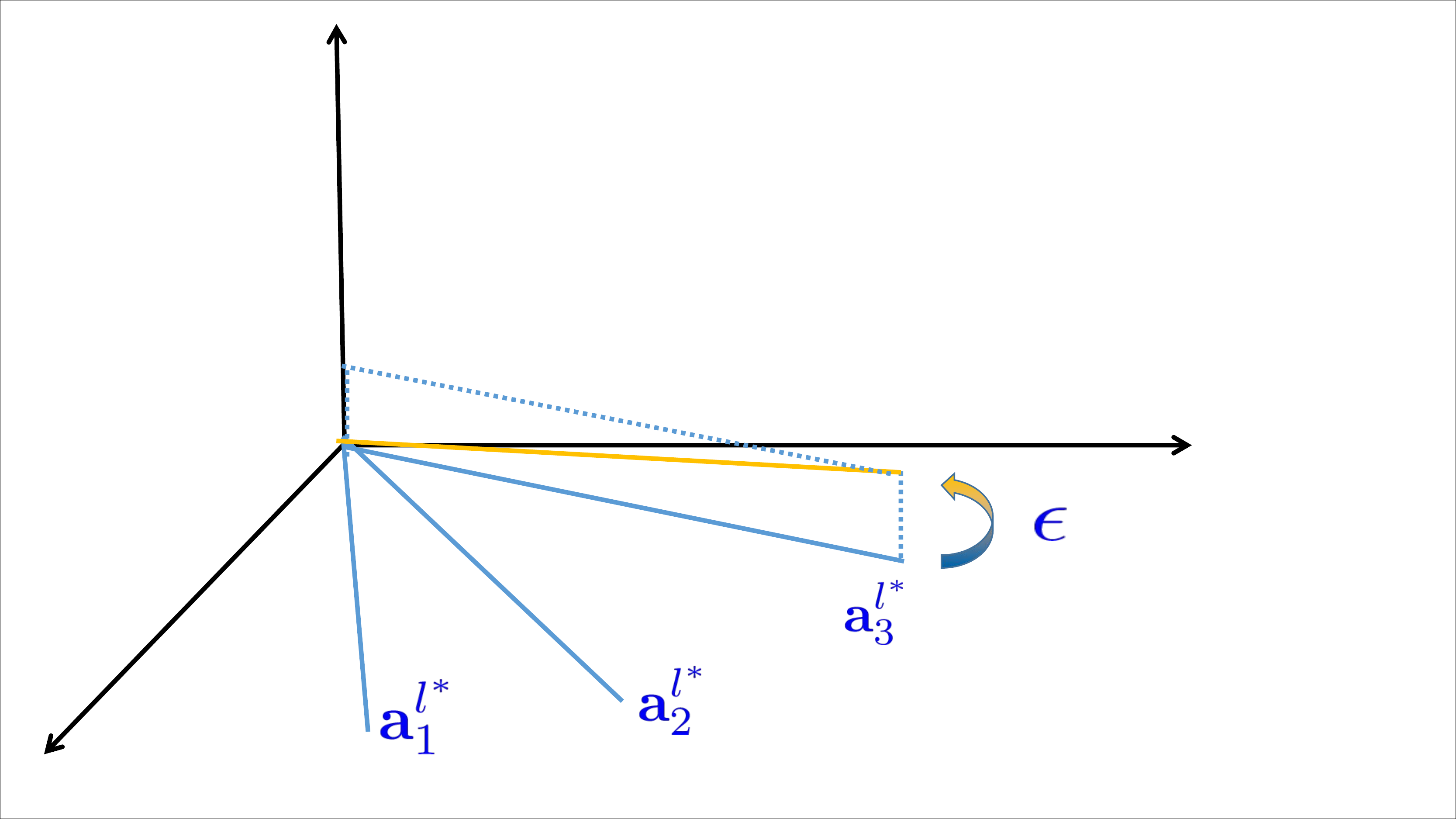} \\[\abovecaptionskip]
    \footnotesize Lemma~\ref{lma:linearIndependence}: 
    An $\epsilon$-change guarantees \\\footnotesize linear independence.
  \end{tabular}}
\hfil
\subfloat{
\begin{tabular}{@{}c@{}}
    \includegraphics[width=2.7in]{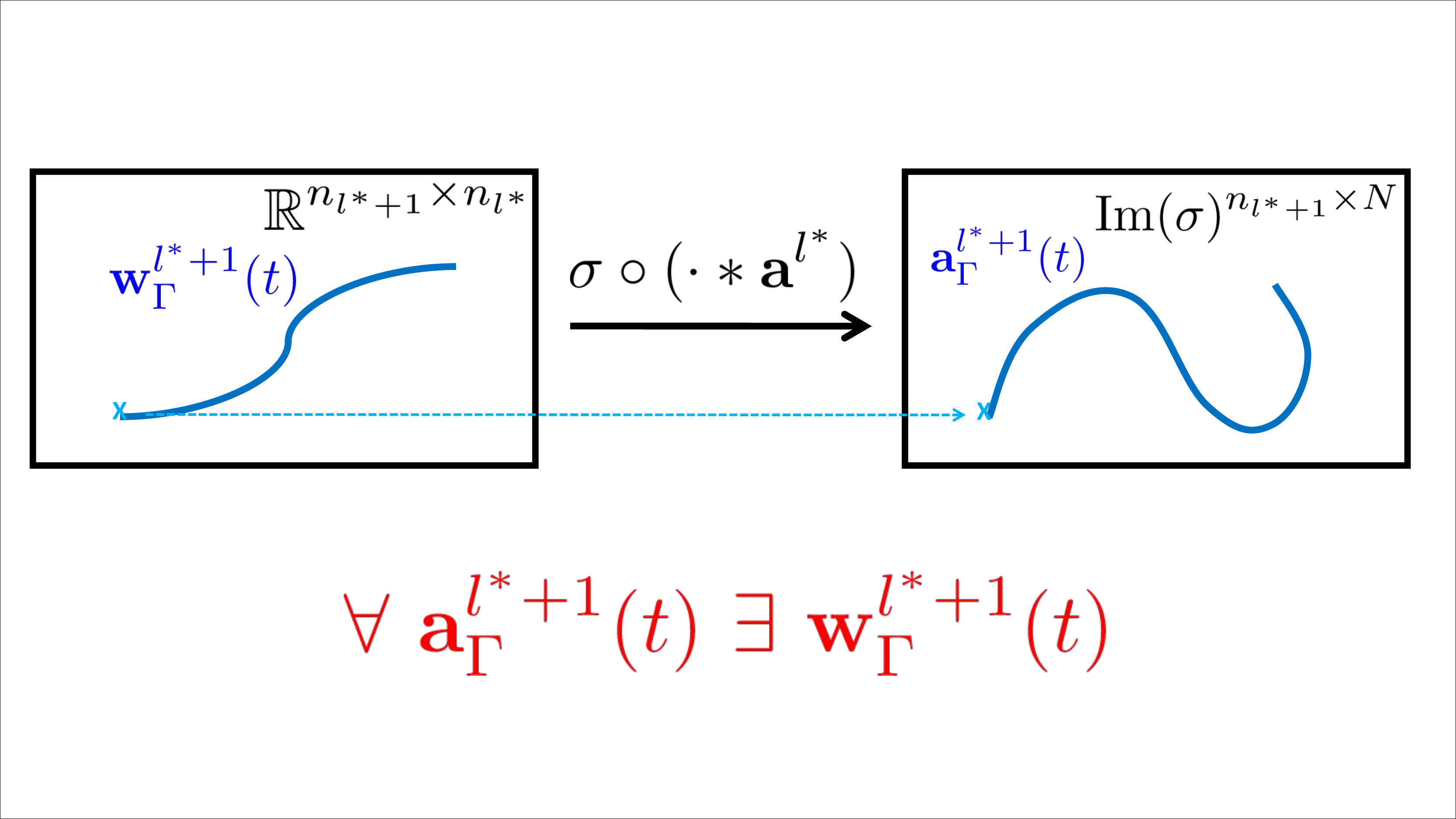} \\[\abovecaptionskip]
    \footnotesize Lemma~\ref{lma:ctsPath1}\&\ref{lma:ctsPath2}: 
    Realizing paths of activation\\\footnotesize  vectors by paths in parameter space 
  \end{tabular}}
\hfil
\subfloat{
\begin{tabular}{@{}c@{}}
    \includegraphics[width=2.7in]{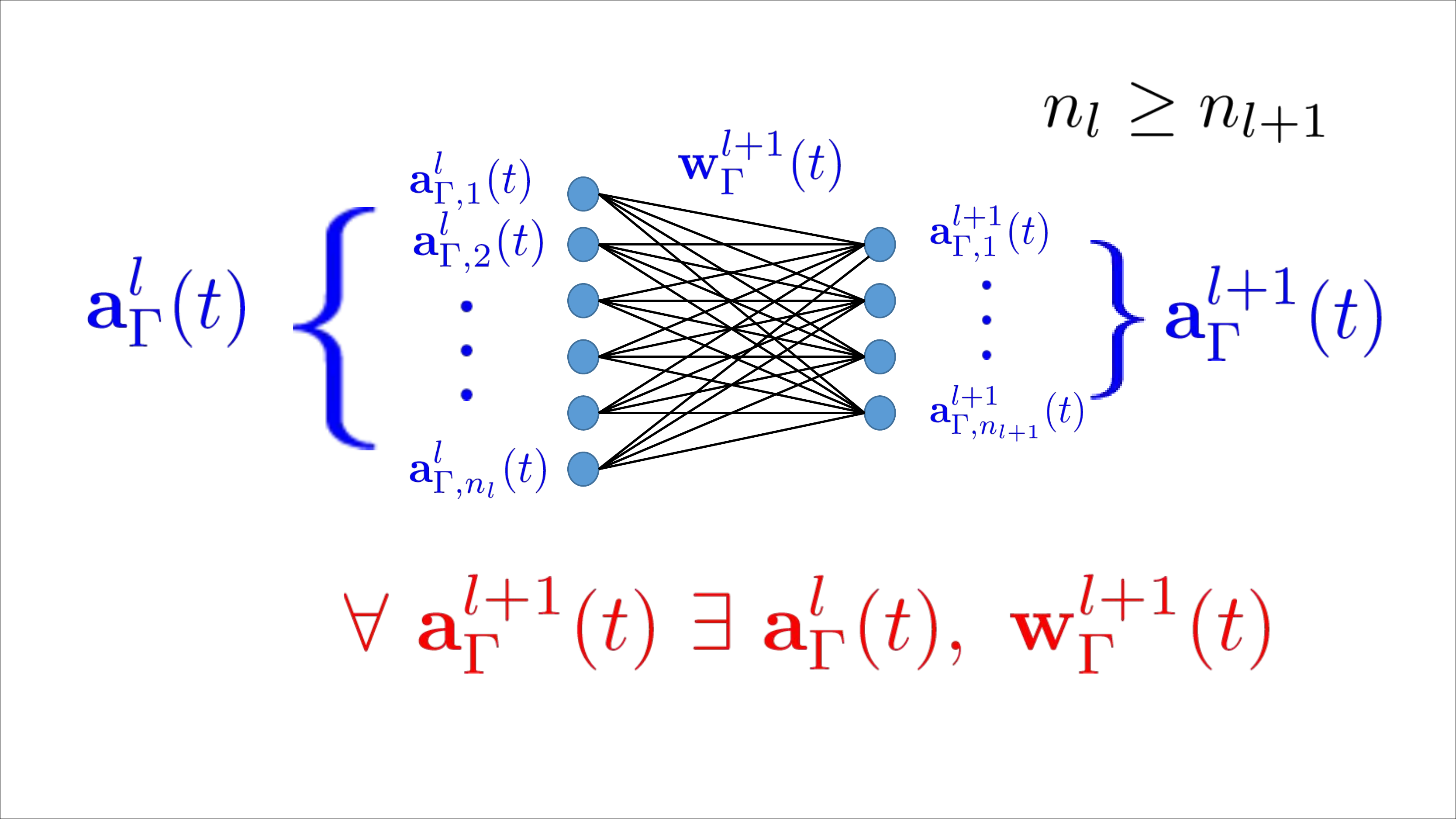} \\[\abovecaptionskip]
    \footnotesize Lemma~\ref{lma:ctsPathInductive}\&\ref{lma:ctsPath2}: 
    Realizing paths of activation \\\footnotesize vectors for decreasing hidden dimensions.
  \end{tabular}}
\caption{Non-increasing paths to the global minimum  exist from almost everywhere. Visualization of proof ideas. }
\label{fig:Dstatistics}
\end{figure*}


\begin{restatable}{lem}{ctsPathOne}\label{lma:ctsPath1}
Assume that in the extremely wide layer {$l^*$} we have that the activation vectors at a set of parameters $\w$ satisfy {$\spn_k[\av^{l^*}_k] =\RR ^N.$} Then, for any continuous path $\nv^{l^*+1}_\Gamma:[0,1]\rightarrow \RR ^{ n_{l^*+1}\times N}$ with starting point $\nv^{l^*+1}_\Gamma(0)=\nv^{l^*+1}=\w^{l^*+1} \cdot \av^{l^*}+\w_0^{l^*+1},$ 
there is a continuous path of parameters {$\w_\Gamma^{l^*+1}:[0,1]\rightarrow \RR ^{n_{l^*+1} \times n_{l^*}}$ of the $(l^*+1)$-th layer with $\w_\Gamma^{l^*+1}(0)=\w^{l^*+1}$} and such that
$$\nv^{l^*+1}_\Gamma(t)=  \w_\Gamma^{l^*+1}(t) \cdot \av^{l^*}+\w_0^{l^*+1}. $$
\end{restatable}

\begin{restatable}{lem}{ctsPathTwo}\label{lma:ctsPath2}
For all continuous paths $\av_\Gamma(t)$ in $\text{Im}(\sigma)^{n\times N}$, i.e., the $(n\times N)$-fold copy of the image of an activation function $\sigma\in\mathcal{A}$, there is a continuous path $\nv_\Gamma(t)$ in $\RR ^{n\times N}$ such that $\av_\Gamma(t)=\sigma(\nv_\Gamma(t))$ for all $t$.
\end{restatable}




{With activations values $\av^l$ depending on parameters $\w^{\iota}$  of previous layers with index $\iota \leq l$,  we denote this functional dependence by $\av^l(\w)$.   We say that a continuous path $\av^l_\Gamma:[0,1]\rightarrow \text{Im}(\sigma)^{n_l\times N}$ of activation values in the $l$-th layer is \textdef{realized by a path of parameters} $\w_{\Gamma}(t)$, if the path $\w_{\Gamma}(t)$ induces a change of activation values at layer $l$ according to the desired path $\av_\Gamma^l$, i.e., 
$\av^l(\w_\Gamma(t))=\av^l_\Gamma(t)$. Using this terminology, Lemma~\ref{lma:ctsPath1} and~\ref{lma:ctsPath2} show that in a layer following an extremely wide one, any continuous path of activation values can be realized by a path of parameters of the same layer. }

{The \textdef{third step} guarantees that, as long as the sequence of dimensions of subsequent hidden layers never increases, realizability of arbitrary paths in layer $l$ implies realizability of arbitrary paths in layer $l+1$ for both activation and pre-activation values. }

\begin{restatable}{lem}{ctsPathInductive}\label{lma:ctsPathInductive}
{Assume that $n_{l+1}\leq n_l$ and that the weight matrix $\w^{l+1}\in\RR^{n_{l+1}\times n_l}$ has full rank $n_{l+1}$.  Let $\av^l=\av^l(\w)$ and $\nv^{l+1}=\nv^{l+1}(\w)$ be the matrices of activation values at the $l$-th layer and pre-activations at the $(l+1)$-th layer for parameters $\w$ respectively. Then, for any continuous path $\nv_\Gamma^{l+1}:[0,1]\rightarrow \RR^{ n_{l+1}\times N}$ with 
$\nv_\Gamma^{l+1}(0)=\nv^{l+1}=\w^{l+1}\cdot  \av^l+\w^{l+1}_{0},$ there are continuous paths of (full-rank) weight matrices $\w_\Gamma^{l+1}:[0,1]\rightarrow \RR ^{n_{l+1} \times n_{l}}$ and bias $\w_{\Gamma,0}^{l+1} : [0,1]\rightarrow \RR ^{n_{l+1} }$ in the $(l+1)$-th layer and a continuous path $\av_\Gamma^l:[0,1]\rightarrow \text{Im}(\sigma)^{ n_{l}\times N}$ of activation values in the $l$-th layer, such that $\w_\Gamma^{l+1}(0)=\w^{l+1}$, $\w^{l+1}_{\Gamma,0}(0)=\w^{l+1}_0$, $\av_\Gamma^l(0)=\av^l$ and for all $t\in[0,1]$,
$$\nv_\Gamma^{l+1}(t)=  \w_\Gamma^{l+1}(t)\cdot  \av_\Gamma^l(t)+\w^{l+1}_{\Gamma,0}(t).$$}
\end{restatable}

Combining Lemma~\ref{lma:ctsPath2} and \ref{lma:ctsPathInductive}, we can realize any continuous path of activation values $\av^{l+1}_\Gamma(t)$ in layer $l+1$ by a  path of parameters, if arbitrary paths of activation values $\av^{l}_\Gamma(t)$ can be realized in the previous layer $l$. By Lemma~\ref{lma:ctsPath1}, we can indeed realize arbitrary paths in the layer following the extremely wide layer. Hence, by induction over the layers we find that any path at the output is realizable. In the following result, we denote the dependence of the network function's output of its parameters $\w$ on the training sample $x_\alpha$ by $f(\w;x_\alpha)$.

\begin{restatable}{lem}{finalPath}\label{lma:finalPath}
{Assume a neural network structure as above with activation vectors $\av^{l^*}_k$ of the extremely wide hidden layer spanning $\RR^N$, hidden dimensions $n_{l+1}\leq n_l$ for all $l > l^*$ and weight matrices $\w^{l+1}\in\RR^{n_{l+1}\times n_{l}}$ of full rank $n_{l+1}$ for all $l > l^*$.} Then for any continuous path $f_\Gamma:[0,1]\rightarrow \RR^N$ with $f_\Gamma(0)=[f(\w; x_\alpha)]_\alpha$ there is a continuous path $\w_\Gamma(t)$ from the current weights $\w_\Gamma(0)=\w$ that realizes $f_\Gamma(t)$ as the output of the neural network function, $f_\Gamma(t)=[f( \w_\Gamma(t); x_\alpha)]_\alpha$.
\end{restatable}

With fixed $z_\alpha=f(\w; x_\alpha)$, the prediction for the current weights, let $\mathbf{z}=[z_\alpha]_\alpha$ denote the vector of predictions, and let $\mathbf{y}=[y_\alpha]_\alpha$ denote the vector of all target values. An obvious path of decreasing loss at the output layer is then given by $f_\Gamma(t)=\mathbf{z}+t\cdot (\mathbf{y}-\mathbf{z})$, inducing the loss $\Loss=||\mathbf{z}+t\cdot (\mathbf{y}-\mathbf{z})-\mathbf{y}||_2^2=(1-t)||\mathbf{y}-\mathbf{z}||_2^2$. This concludes the proof of Theorem~\ref{thm:pathToGlobal} by applying Lemma~\ref{lma:finalPath} to this choice of $f_\Gamma(t)$ after a possible change of the starting parameters $\w$ to arbitrarily close parameters $\w'$ using Lemma~\ref{lma:ctsPath1}.

\section{Conclusion}
We have proved the existence of suboptimal local minima for regression neural networks with sigmoid activation functions of arbitrary width.  We established that the nature of local minima is such that they live in a special region of the cost function called a non-attractive region, and showed that a non-increasing path to a configuration with lower loss than that of the region can always be found. For sufficiently wide neural networks with decreasing hidden layer dimensions after the extremely wide layer, all local minima belong to such a region. We generalized a procedure to find such regions in shallow networks, introduced by \citet{FukumizuAmari}, to deep networks and described conditions for the construction to work. The necessary conditions become hard to satisfy in wider and deeper networks and, if they fail, the construction leads to saddle points instead.  The appearance of an additional condition when extending \citet{FukumizuAmari}'s construction to deeper networks suggests that local minima are rare and degenerate in deep networks, but their existence shows that no general statement about all local minima being global can be made.\\

\appendix

\section{Notation}
\subsection{General Notation}\label{app:notation}
\vspace{1cm}

\begin{center}

\begin{tabular}{|cl|l|}
\hline
$[x_\alpha]_\alpha$ & $\RR^n$ & column vector with entries $x_\alpha\in \RR$  \\
$[x_{i,j}]_{i,j}$ & $\in \RR^{n_1\times n_2}$ & matrix with entry $x_{i,j}$ at position $(i,j)$ \\
\text{Im}(f) & $\subseteq \RR$ & image of a function $f$ \\
$C^n(X,Y)$ & & n-times continuously differentiable function  \\ & & \hspace{0.5cm} from $X$ to $Y$ \\
&  & \\
$N$ & $\in\NN$ & number of data samples in training set  \\
$x_\alpha$ & $\in \RR^{n_0}$ & training sample input  \\
$y_\alpha$ & $\in \RR$ & target output for sample $x_\alpha$ \\
$\mathcal{A}$ & $\in \text{C}(\RR)$ & class of real-analytic, strictly monotonically \\
& & \hspace{0.5cm} increasing, bounded (activation) functions such  \\ & & \hspace{0.5cm} that the closure of the image contains zero\\
$\sigma$ & $\in C^2(\RR,\RR)$ & a nonlinear activation function in class $\mathcal{A}$ \\
$f$ & $\in C(\RR^{n_0},\RR)$ & neural network function  \\
$l$ &  $1\leq l\leq L$ & index of a layer  \\
$L$ &  $\in\NN$ & number of layers excluding the input layer  \\
 l=0  & & input layer  \\
 $l=L$  & & output layer  \\
$n_l$ & $\in \NN$ &  number of neurons in layer $l$  \\
$M$ & $=\sum_{l=1}^L (n_l\cdot n_{l-1}) $ & number of all network parameters \\
$k$ &  $1\leq k\leq n_l$ & index of a neuron in layer $l$  \\
$\w^l$ & $\in \RR^{n_{l}\times n_{l-1}}$ & weight matrix of the l-th layer\\
$\w$  & $\in\RR^{M}$ & collection of all $\w^l$ \\
$w^l_{i,j}$  & $\in\RR$ & the weight from neuron $j$ of layer $l-1$ to \\ & &\hspace{0.5cm} neuron $j$ of layer $l$\\
$w^L_{\parm,j}$  & $\in\RR$ & the weight from neuron $j$ of layer $L-1$ to \\ & &  \hspace{0.5cm} the output \\
$\w_{\Gamma}$ & $\in \text{C}([0,1],\RR^M)$ & a path in parameter space\\
 $\Loss,\loss$  & $\in \RR_+$ & squared loss over training samples \\
 $\loss_\alpha$  & $\in \RR_+$ & the squared loss for data sample $x_\alpha$ \\
$\n{l,k}{x}$ & $\in\RR$ & value at neuron $k$ in layer $l$ before activation  \\ & &  \hspace{0.5cm} for input pattern $x$\\
$\n{l}{x}$ & $\in\RR^{n_l}$ & neuron pattern at layer $l$ before activation for \\ & &  \hspace{0.5cm} input pattern $x$\\
$\act{l,k}{x}$ & $\in \text{Im}(\sigma)$ & activation pattern at neuron $k$ in layer $l$ for \\ & &  \hspace{0.5cm} input $x$\\
$\act{l}{x}$ & $\in\text{Im}(\sigma)^{n_l}$ & neuron pattern at layer $l$ for input $x$\\
\hline
\end{tabular}

\end{center}

\vspace{0.5cm}
\newpage 
\subsection{Notation Section~V}
In Section~V, where we fix a layer $l$, we additionally use the following notation.

\vspace{0.2cm}
\begin{center}

\noindent \begin{tabular}{|cl|l|}
\hline
$[u_{p,i}]_{p,i}$ & $\in \RR^{n_{l}\times n_{l-1}}$ & weights of the given layer $l$. \\
$[v_{s,q}]_{s,q}$ & $\in \RR^{n_{l}\times n_{l+1}}$ & weights the layer $l+1$. \\
$r$ & $\in \{1,2,\ldots,n_l\}$ & the index of the neuron of layer $l$ that we use \\
 & & \hspace{0.5cm} for the addition of one additional neuron \\ 
$M$ & $\in \NN$ & $=\sum_{t=1}^L ( n_t\cdot n_{t-1})$, the number of weights in \\
 & & \hspace{0.5cm} the smaller neural network\\
$\wrest$ & $\in \RR^{M-n_{l-1}-n_{l+1}}$ & all weights except $u_{1,i}$ and $v_{s,1}$ \\
$\gamma_\lambda^r$ & $\in \text{C}(\RR^{M}, \RR^{M'})$ & the map defined in Section~\ref{sct:construction} to add a  \\ $\hspace{1.5cm}M'=$ & $M+n_{l-1}+n_{l+1}  $ & \hspace{0.5cm} neuron in  layer $l$ using the neuron with \\ & & \hspace{0.5cm} index $r$  in layer $l$\\
$B_{i,j}^r$ & $\in \RR$ & $= \sum_{\alpha}  \sum_k \frac{\partial \loss_\alpha}{\partial \n{l+1,k}{x_\alpha}} \cdot  v^*_{k,r} \cdot \sigma''(\n{l,r}{x_\alpha} ) $ \\ & & \hspace{0.5cm} $\hspace{1.8cm} \cdot\: \act{l-1,i}{x_\alpha}\cdot \act{l-1,j}{x_\alpha}$\\
$D_{i}^{r,s}$ & $\in \RR$ & $= \sum_\alpha  \frac{\partial \loss_\alpha}{\partial \n{l+1,s}{x_\alpha}}  \cdot \sigma'(\n{l,r}{x_\alpha}) \cdot \act{l-1,i}{x_\alpha}$\\
$B=[B_{i,j}^r]_{i,j}$ & $\in\RR^{n_{l-1}\times n_{l-1}}$ & matrix needs to be pos.\ or neg.\ def.\ for local min.\ \\
$D=[D_{i}^{r,s}]_{i,s}$ & $\in \RR^{n_{l-1}\times n_{l+1}} $ & matrix needs to be 0 for local min.\ \\
\hline
\end{tabular}
\end{center}

\vspace{0.5cm}

\subsection{Notation Section~VI}
In Section~VI, we additionally use the following notation.

\vspace{0.2cm}

\begin{center}
\noindent\begin{tabular}{|cl|l|}
\hline
$\av^l_k$ & $\in \text{Im}(\sigma)^N$ & activation vector at neuron $k$ in layer $l$ given \\ & & \hspace{0.5cm} by    $\av^l_k=[\act{l,k}{x_\alpha}]_\alpha$\\
$\av^l$ & $\in \text{Im}(\sigma)^{n_l\times N}$ & matrix of activations in layer $l$ given  by    $\av^l=[\av^l_k]_k$\\
$\av^l(\w)$ & $\in \text{Im}(\sigma)^{n_l\times N}$ & activation vector at layer $l$ as a function \\ & & \hspace{0.5cm} of the parameters $\w$ \\ 
$\w_{\Gamma}$ & $\in \text{C}([0,1],\RR^M)$ & a path in parameter space\\
$\w_{\Gamma}^l$ & $\in \text{C}([0,1],\RR^{n_l\times n_{l-1}})$ & a path in parameter space at layer $l$\\
$\nv^l$ & $\in \RR^{n_l\times N}$ & matrix of pre-activation values in layer $l$ given \\ & & \hspace{0.5cm}  by    $\nv^l=[\n{l,k}{x_\alpha}]_{k,\alpha} $\\
$\nv_\Gamma^l$ & $\in \text{C}([0,1],\RR^{n_l \times N})$ & a path of neuron values in layer $l$\\
$\av_{\Gamma}^l$ & $\in \text{C}([0,1],\RR^{n_l\times N})$ & a path of activation values in layer $l$\\
$\av_{\Gamma,j}^l$ & $\in \text{C}([0,1],\RR^{N})$ & a path of activation vectors at neuron $j$ in layer $l$\\
$f(\w;x_\alpha)$ & $\in \RR$ & network  as function of parameters $\w$ and sample $x_\alpha$\\
$f_{\Gamma}^l$ & $\in \text{C}([0,1],\RR^{N})$ & a path of outputs over training samples\\
\hline
\end{tabular}
\end{center}

\vspace{1cm}


\newpage
\section{Proofs}

\subsection{Proofs for the Construction of Local Minima}\label{app:construction}

Here we prove Theorem~\ref{thm:construction},
which follows from two lemmas, with the first lemma being Lemma~\ref{lma:hessian} containing the computation of the Hessian of the cost function $\Loss$ of the larger network at parameters $\gamma_\lambda^r([{u}^*_{r,i}]_i,[{v}^*_{s,r}]_s,\wrest^*)$ with respect to a suitable basis.


\begin{proof}\textbf{(Lemma~\ref{lma:hessian})}
The proof only requires a tedious, but not complicated calculation (using the relation $\alpha\lambda-\beta(1-\lambda)=0$ multiple times. To keep the argumentation streamlined, we moved all the necessary calculations into the supplementary material. 
\end{proof}

The second lemma determines when matrices of the form as calculated in Lemma~\ref{lma:hessian} are positive definite.

\begin{appendixLemma}\label{lma:posSemidef}
Let $a,b,c,d,e,f,g,h,x$ be matrices of appropriate sizes. 
\begin{itemize}
\item[(a)] A matrix of the form
$$
\begin{pmatrix}
a & 2b & c & 0 \\
2b^T & 4d & 2e & 0 \\
c^T & 2e^T & f & 0 \\
0 & 0 & 0 & x\\ 
\end{pmatrix}
$$
is positive semidefinite if and only if both $x$ and the matrix
$$
\begin{pmatrix}
a & b & c \\
b^T & d & e  \\
c^T & e^T & f \\ 
\end{pmatrix}
$$ 
are positive semidefinite.
\item[(b)] A matrix $x$ of the form 
$$x= 
\begin{pmatrix}
g & h \\
h^T & 0\\ 
\end{pmatrix}
$$
is positive semidefinite if and only if $g$ is positive semidefinite and $h=0$.
\end{itemize}
\end{appendixLemma}

\begin{proof}
\begin{itemize}
\item[(a)] By definition, a matrix $A$ is positive semidefinite if and only if $z^TAz\geq 0$ for all $z$. Note now that
$$(z_1,z_2,z_3,z_4)
\begin{pmatrix}
a & 2b & c & 0 \\
2b^T & 4d & 2e & 0 \\
c^T & 2e^T & f & 0 \\
0 & 0 & 0 & x\\ 
\end{pmatrix}
\begin{pmatrix}
z_1 \\
z_2 \\
z_3 \\
z_4 \\ 
\end{pmatrix}
=(z_1,2z_2,z_3,z_4)
\begin{pmatrix}
a & b & c & 0 \\
b^T & d & e & 0 \\
c^T & e^T & f & 0 \\
0 & 0 & 0 & x\\ 
\end{pmatrix}
\begin{pmatrix}
z_1 \\
2z_2 \\
z_3 \\
z_4 \\ 
\end{pmatrix}
$$
\item[(b)] 
It is clear that the matrix $x$ is positive semidefinite for $g$ positive semidefinite and $h=0$. To show the converse, first note that if $g$ is not positive semidefinite and $z$ is such that $z^Tgz<0$ then 
$$(z^T,0)
\begin{pmatrix}
g & h \\
h^T & 0\\ 
\end{pmatrix}
\begin{pmatrix}
z  \\
0 \\ 
\end{pmatrix}=z^T g z <0.$$
It therefore remains to show that also $h=0$ is a necessary condition. Assume $h\neq 0$ and find $z$ such that $hz \neq 0$. Then for any $\lambda\in \RR$ we have
$$((hz)^T,-\lambda z^{T})
\begin{pmatrix}
g & h \\
h^T & 0\\ 
\end{pmatrix}
\begin{pmatrix}
hz  \\
-\lambda z \\ 
\end{pmatrix}=(hz)^T g (hz) - 2 (hz)^T h \lambda z $$
$$= (hz)^T g (hz)  - 2 \lambda ||hz||_2^2.$$
For sufficiently large $\lambda$, the last term is negative.
\end{itemize}
\end{proof}

In addition, to find local minima from positive semi-definiteness, one needs to explain away all degenerate directions, i.e., we need to show that the loss function actually does not change into the direction of eigenvectors of the Hessian with eigenvalue $0$. Otherwise a higher derivative into this direction could be nonzero and potentially lead to a saddle point.

\begin{proof}[Proof of Theorem~\ref{thm:construction}]
In Lemma~\ref{lma:hessian}, we calculated the Hessian of $\Loss$ with respect to a suitable basis at a the critical point  $\gamma_\lambda([{u}^*_{r,i}]_i,[{v}^*_{s,r}]_s,\wrest^*)$. If the matrix $[D_i^{r,s}]_{i,s}$ is nonzero, then by Lemma~\ref{lma:posSemidef}(b) the Hessian is not positive semidefinite, hence none of the critical points are local minima.

If, on the other hand, the matrix $[D_i^{r,s}]_{i,s}$ is zero, then by Lemma~\ref{lma:posSemidef}(a+b) the Hessian is positive semidefinite, since
$$\begin{pmatrix}
[\frac{\partial^2 \loss}{\partial u_{r,i} \partial u_{r,j} }]_{i,j}  &  [\frac{\partial^2 \loss}{\partial u_{r,i} \partial v_{s,r} }]_{i,s}& [\frac{\partial^2 \loss}{\partial \wrest\ \partial u_{r,i} }]_{i,\wrest}   \\ 
[\frac{\partial^2 \loss}{\partial u_{r,i} \partial v_{s,r} }]_{s,i} &  [\frac{\partial^2 \loss}{\partial v_{s,r} \partial v_{t,1} }]_{s,t} &  [\frac{\partial^2 \loss}{\partial \wrest\ \partial v_{s,r} }]_{s,\w}   \\
[\frac{\partial^2 \loss}{\partial \wrest\ \partial u_{r,i} }]_{\wrest,i} &  [ \frac{\partial^2 \loss}{\partial \wrest\ \partial v_{s,r} }]_{\wrest,s} & [\frac{\partial^2 \loss}{\partial \wrest\ \partial \wrest' }]_{\wrest,\wrest'} \\
\end{pmatrix}
$$
is positive definite by assumption (isolated minimum), and $\alpha\beta [B_{i,j}^r]_{i,j}$ is positive definite if $\lambda\in (0,1)\Leftrightarrow \alpha\beta>0$ and $[B_{i,j}^r]_{i,j}$ is positive definite, or if ($\lambda <0$ or $\lambda >1)\Leftrightarrow \alpha\beta<0$ and $[B_{i,j}^r]_{i,j}$ is negative definite. In each case we can alter $\lambda$ to values leading to saddle points without changing the network function or loss. Therefore, the critical points can only be saddle points or local minima on a non-attracting region of local minima.

To determine whether the critical points in question lead to local minima when $[D_i^{r,s}]_{i,s}=0$, it is insufficient to only prove the Hessian to be positive semidefinite (in contrast to (strict) positive definiteness), but we need to consider directions for which the second order information is insufficient. We know that the loss is at a minimum with respect to all coordinates except for the degenerate directions defined by a change of $[v_{s,-1} - v_{s,r}]_s$ that keeps $[v_{s,-1} + v_{s,r}]_s$ constant. However, the network function $f(x)$ is constant along $[v_{s,-1} - v_{s,r}]_s$ (keeping $[v_{s,-1} + v_{s,r}]_s$ constant) at the critical point where $u_{-1,i}= u_{r,i}$ for all $i$. Hence, no higher order information leads to saddle points and it follows that the critical point lies on a region of local minima.
\end{proof}

\subsection{Local Minima at Infinity in Neural Networks}\label{app:infinity}

In this section we prove Theorem~\ref{thm:infinity}, showing the existence of local minima at infinity in neural networks.


\begin{proof}\textbf{(Theorem~\ref{thm:infinity})} We will show that, if all bias terms $u_{i,0}$ of the last hidden layer are sufficiently large, then there are parameters $u_{i,k}$ for $k\neq 0$ and parameters $v_i$ of the output layer such that the minimal loss is achieved at $u_{i,0}=\infty$ for all $i$.

We note that, if $u_{i,0}=\infty$ for all $i$, all neurons of the last hidden layer are fully active for all samples, i.e., $\act{L-1,i}{x_\alpha}=1$ for all $i$. Therefore, in this case $f(x_\alpha)=\sum_i v_{\parm,i}$ for all $\alpha$. A constant function $f(x_\alpha)=\sum_i v_{\parm,i} =c$ minimizes the loss $\sum_\alpha(c-y_\alpha)^2$ uniquely for $c:= \frac{1}{N} \sum_{\alpha=1}^N y_\alpha$. We will assume that the $v_{\parm,i}$ are chosen such that $\sum_i v_{\parm,i} =c$ does hold. That is, for fully active hidden neurons at the last hidden layer, the $v_{\parm,i}$ are chosen to minimize the loss.

We write $f(x_\alpha)=c +\epsilon_\alpha$. Then 
$$\Loss = \frac{1}{2} \sum_\alpha (f(x_\alpha)-y_\alpha)^2 = \frac{1}{2} \sum_\alpha (c+\epsilon_\alpha-y_\alpha)^2$$
$$=  \frac{1}{2} \sum_\alpha (\epsilon_\alpha+(c-y_\alpha))^2$$
$$= \underbrace{\frac{1}{2} \sum_\alpha (c-y_\alpha)^2}_{\text{Loss at $u_{i,0}=\infty$ for all $i$}} +\underbrace{\frac{1}{2} \sum_\alpha \epsilon_\alpha^2}_{\geq 0}  + \underbrace{\sum_{\alpha} \epsilon_\alpha (c-y_\alpha) }_{(*)}.$$
The idea is now to ensure that $(*)\geq 0$ for sufficiently large $u_{i,0}$ and in a neighborhood of the $v_{\parm,i}$ chosen as above. Then the loss $\Loss$ is larger than at infinity, and any such point in parameter space with $u_{i,0}=\infty$ and $v_{\parm,i}$ with $\sum_i v_{\parm,i}=c$ is a local minimum.

To study the behavior at $u_{i,0}=\infty$, we consider $p_i=\exp(-u_{i,0})$. Note that $$\lim_{u_{i,0}\rightarrow \infty} p_i = 0.$$ We have 
$$f(x_\alpha)=\sum_i v_{\parm,i} \sigma(u_{i,0}+\sum_k u_{i,k} \act{L-2,k}{x_\alpha})$$
$$= \sum_i v_{\parm,i}\cdot \frac{1}{1+p_i\cdot \exp(-\sum_k u_{i,k} \act{L-2,k}{x_\alpha})}$$

Now for $p_i$ close to $0$ we can use Taylor expansion of $g^j_i(p_i):=\frac{1}{1+p_i exp(a^j_i)}$ to get $g^j_i(p_i)=1-\exp(a^j_i)p_i+\mathcal{O}(|p_i|^2)$. Therefore 
$$f(x_\alpha)=c - \sum_i v_{\parm,i}p_i \exp(-\sum_k u_{i,k} \act{L-2,k}{x_\alpha}) + \mathcal{O}(p_i^2)$$
and we find that $\epsilon_\alpha = -\sum_i v_{\parm,i}p_i \exp(-\sum_k u_{i,k} \act{L-2,k}{x_\alpha}) + \mathcal{O}(p_i^2)$.

Recalling that we aim to ensure 
$$\left ( \ast\right ) = \sum_{\alpha} \epsilon_\alpha (c-y_\alpha) \geq 0$$
we consider  
$$  \sum_{\alpha} \epsilon_\alpha (c-y_\alpha) = -  \sum_{\alpha}  (c-y_\alpha) (\sum_i v_{\parm,i}p_i \exp(-\sum_k u_{i,k} \act{L-2,k}{x_\alpha})) + \mathcal{O}(p_i^2)$$
$$ = -  \sum_i v_{\parm,i} p_i  \sum_\alpha  (c-y_\alpha)  \exp(-\sum_k u_{i,k} \act{L-2,k}{x_\alpha}) + \mathcal{O}(p_i^2)$$

We are still able to choose the parameters $u_{i,k}$ for $i\neq 0$, the parameters from previous layers, and the $v_{\parm,i}$ subject to $\sum_i v_{\parm,i} =c$. If now \begin{equation*}
    \begin{split}
        v_{\parm,i}>0 & \text{ whenever }\sum_\alpha  (c-y_\alpha)  \exp(-\sum_k u_{i,k} \act{L-2,k}{x_\alpha})<0\text{, and} \\
        v_{\parm,i}<0 & \text{ whenever } \sum_\alpha  (c-y_\alpha)  \exp(-\sum_k u_{i,k} \act{L-2, k}{x_\alpha})>0,
    \end{split}
\end{equation*}
then the term $(\ast)$ is strictly positive, hence the overall loss is larger than the loss at $p_i=0$ for sufficiently small $p_i$ and in a neighborhood of $v_{\parm,i}$. The only obstruction we have to get around is the case where we need all $v_{\parm,i}$ of the opposite sign of $c$ (in other words, $\sum_\alpha  (c-y_\alpha)  \exp(-\sum_k u_{i,k} \act{L-2,k}{x_\alpha})$ has the same sign as $c$), conflicting with $\sum_i v_{\parm,i}=c$. To avoid this case, we impose the mild condition that $\sum_\alpha (c-y_\alpha)\act{L-2,r}{x_\alpha}\neq 0$ for some $r$, which can be arranged to hold for almost every data set by fixing all parameters of layers with index smaller than $L-2$. By Lemma~\ref{lma:infinity} below (with $d_\alpha=(c-y_\alpha)$ and $a_\alpha^r=\act{L-2,r}{x_\alpha}$), we can find $u^>_k$ such that $\sum_\alpha  (c-y_\alpha)  \exp(-\sum_k u^>_{k} \act{L-2,k}{x_\alpha})>0$ and $u^<_k$ such that $\sum_\alpha  (c-y_\alpha)  \exp(-\sum_k u^<_{k} \act{L-2,k}{x_\alpha})<0$. We fix $u_{i,k}$ for $k\geq 0$ such that there is some $i_1$ with $[u_{i_1,k}]_k=[u^>_k]_k$ and some $i_2$ with $[u_{i_2,k}]_k=[u^<_k]_k$. This assures that we can choose the $v_{\parm,i}$ of opposite sign to $\sum_\alpha  (c-y_\alpha)  \exp(-\sum_k u_{i,k} \act{L-2,k}{x_\alpha})$ and such that $\sum_i v_{\parm,i}=c$, leading to a local minimum at infinity. 

The local minimum is suboptimal whenever a constant function is not the optimal network function for the given data set.




\end{proof}

\begin{appendixLemma}\label{lma:infinity}
Suppose that m is a positive integer, $m\geq 2$, and for $\alpha=1,\ldots,N$ and $r=1,\ldots,m$  we have numbers $d_\alpha, a_\alpha^r$ in $\RR$ such that
$$\sum_{\alpha=1}^N d_\alpha=0,\text{ and }\sum_\alpha d_\alpha a_\alpha^r\ \neq 0 \text{ for some }r.$$

Then there are $u^<_k, k=1,2,...,m$ such that $$\sum_\alpha  d_\alpha \exp(-\sum_k u^<_{k} a_\alpha^k )<0$$ and $u^>_k, k=1,2,...,m$ such that $$\sum_\alpha  d_\alpha  \exp(-\sum_k u^>_{k} a_\alpha^k)>0.$$
\end{appendixLemma}

\begin{proof}
Consider the function $$\phi(u_1,u_2,\ldots,u_m):= \sum_\alpha  d_\alpha  \exp(-\sum_k u_{k} a_\alpha^k).$$
We have $$\phi(0,0,\ldots,0)=\sum_\alpha  d_\alpha = 0.$$
Further 
$$\frac{\partial \phi}{\partial u_r}_{|(0,0,\ldots,0)}  = -\sum_\alpha  d_\alpha a_\alpha^r .$$
By assumption, there is $r$ such that the last term is nonzero. Hence, using coordinate $r$, we can choose $w=(0,0,\ldots,0,w_r,0,\ldots,0)$ such that $\phi(w)$ is positive and we can choose $w$ such that $\phi(w)$ is negative.
\end{proof}

%


\subsection{Construction of Local Minima in Deep Networks}\label{app:constructionDeep}


\begin{proof}\textbf{(Proposition~\ref{prop:constructionDeep})}
The fact that property (i) suffices  uses that $\frac{\partial \ell_\alpha}{\n{l+1,\parm}{x_\alpha}}$ reduces to $(f(x_\alpha)-y_\alpha)$.
Then, considering a regression network as before, our assumption says that ${v}^*_{\parm,r}\neq 0$, hence its reciprocal can be factored out of the sum in Equation~\ref{eq:D}. Denoting incoming weights into $\n{l,r}{x}$ by $u_{r,i}$ as before, this leads to
\begin{equation*}
\begin{split}
D_{i}^{r,\parm} & = \frac{1}{{v}^*_{\parm,r}}\cdot  \sum_\alpha (f(x_\alpha)-y_\alpha)\cdot {v}^*_{\parm,r}\cdot  \sigma'(\n{l,r}{x_\alpha}) \cdot \act{l-1,i}{x_\alpha} \\
& = \frac{1}{{v}^*_{\parm,r}}\cdot \frac{\partial \loss}{\partial u_{r,i}} = 0 \\
\end{split}
\end{equation*}

In the case of (ii),  $$\frac{\partial \loss_\alpha }{\partial \n{l+1,s}{x_\alpha}}= \frac{\partial \loss_\alpha }{\partial \n{l+1,t}{x_\alpha}}$$
for all $s,t$ and we can factor out the reciprocal of $\sum_t {v}^*_{r,t}\neq 0$ in Equation~\ref{eq:D} to obtain for each $i,s$
\begin{equation*}
\begin{split}
D_{i}^{r,s}&:= \sum_\alpha  \frac{\partial \loss_\alpha }{\partial \n{l+1,s}{x_\alpha}}  \cdot \sigma'(\n{l,r}{x_\alpha}) \cdot \act{l-1,i}{x_\alpha}\\
 &= \frac{1}{\left (\sum_t {v}^*_{r,t}\right )} \sum_\alpha \sum_t {v}^*_{r,t} \cdot \frac{\loss_\alpha}{\partial \n{l+1,t}{x_\alpha}}  \cdot \sigma'(\n{l,r}{x_\alpha}) \cdot \act{l-1,i}{x_\alpha}\\
&= \frac{1}{\left (\sum_t {v}^*_{r,t}\right )}\cdot \frac{\partial \loss}{\partial u_{r,i}}=0\\
\end{split}
\end{equation*}

Part (iii) is evident since in this case clearly every summand in Equation~\ref{eq:D} is zero.
\end{proof}

\subsection{Proofs for the Non-increasing Path to a Global Minimum}\label{app:pathToGlobal}

In this section we discuss how in extremely wide neural networks with a {non-increasing sequence of dimensions of hidden layers following the extremely wide layer}, a path to the global minimum that is non-increasing in loss  may be found from almost everywhere in the parameter space. 

\pathToGlobal*

The \textdef{first step} of the proof is to use the freedom given by $\epsilon>0$ to change the starting point in parameter space to satisfy that the activation vectors {$\av_k^{l^*}$ of the extremely wide layer $l^*$} span the entire space $\RR ^N$.

\linearIndependence*

{The \textdef{second step} of the proof is to guarantee that we can then induce any continuous change of activation vectors in layer $l^*+1$  by suitable paths in the parameter space changing only the weights of the same layer.} The following two lemmas ensure exactly that. We first consider pre-activation values and then consider the application of the activation function. We slightly abuse notation in the statement when adding a vector to a matrix, which shall mean the addition of the vector to all columns of the matrix.

\ctsPathOne*

\begin{proof}
We write $\nv_\Gamma^{l^*+1}(t)=\nv^{l^*+1}+\tilde{\nv}_\Gamma(t)$ with $\tilde{\nv}_\Gamma(0)=0$. We will find $\tilde \w_\Gamma(t)$ such that $\tilde \nv_\Gamma(t)=\tilde \w_\Gamma(t) \cdot \av^{l^*}$ with $\tilde \w_\Gamma(0)=0$. Then $\w^{l^*+1}_\Gamma(t):=\w^{l^*+1}+\tilde \w_\Gamma(t)$ does the job. 



Since, by assumption, $\av^{l^*}=[\act{l^*,k}{x_\alpha}]_{k,\alpha}$ has full rank, 
we can find an invertible submatrix $\bar{A}\in \RR^{N\times N}$ of $\av^{l^*}$. Then we can define a continuous path $\omega$ in $\RR^{n_{l^*+1}\times N}$ given by $\omega(t):=\tilde \nv_\Gamma(t)\cdot \bar{A}^{-1}$, which satisfies $\omega(t)\cdot \bar{A}=\tilde \nv_\Gamma(t)$ and $\omega(0)=0$. Extending $\omega(t)$ to a path $\tilde \w_\Gamma(t)$ in $\RR^{n_{l^*+1}\times n_{l^*}}$ by zero columns at positions corresponding to rows of $\av^{l^*}$ missing in $\bar{A}$ gives $\tilde \w_\Gamma(t) \cdot \av^{l^*} =\tilde \nv_\Gamma (t)$ and $\tilde \w_\Gamma(0)=0$ as desired.

\end{proof}

\ctsPathTwo*

\begin{proof}
Since $\sigma: \RR^{n\times N}\rightarrow \text{Im}(\sigma)^{n\times N}$ is invertible with a continuous inverse, take $$\nv_\Gamma(t)=\sigma^{-1}(\av_\Gamma(t)).$$
\end{proof}

{With activations values $\av^l$ depending on parameters $\w^{\iota}$  of previous layers with index $\iota \leq l$,  we denote this functional dependence by $\av^l(\w)$.   We say that a continuous path $\av^l_\Gamma:[0,1]\rightarrow \text{Im}(\sigma)^{n_l\times N}$ of activation values in the $l$-th layer is \textdef{realized by a path of parameters} $\w_{\Gamma}(t)$, if the path $\w_{\Gamma}(t)$ induces a change of activation values at layer $l$ according to the desired path $\av_\Gamma^l$, i.e., 
$\av^l(\w_\Gamma(t))=\av^l_\Gamma(t)$. Using this terminology, Lemma~\ref{lma:ctsPath1} and~\ref{lma:ctsPath2} show that in a layer following an extremely wide one, any continuous path of activation values can be realized by a path of parameters of the same layer. }

{
The \textdef{third step} guarantees that, as long as the sequence of dimensions of subsequent hidden layers never increases, realizability of arbitrary paths in layer $l$ implies realizability of arbitrary paths in layer $l+1$ for both activation and pre-activation values. }

\ctsPathInductive*


\begin{proof}
{Since the matrix  $\w^{l+1}$ has full rank, it contains an invertible submatrix $W\in\RR^{n_{l+1}\times n_{l+1}}$. Since we can permute indices of neurons in each layer, we can assume without loss of generality that this submatrix consists of the first $n_{l+1}$ columns of  $\w^{l+1}$.}

For some suitable paths $\lambda(t)\in \RR_{>0}$ and $\delta(t)\in\RR^{n_{l+1}}$, which will be chosen later, we define
\begin{equation*}
\begin{split}
\tilde \nv_\Gamma(t)&:=\nv_\Gamma^{l+1}(t)-\nv_\Gamma^{l+1}(0)\\
\w_\Gamma^{l+1}(t)&:=\lambda(t)\cdot \w^{l+1}\\
\w^{l+1}_{\Gamma,0}(t)&:=\w_0^{l+1}- \delta(t)\\
\av_\Gamma^l(t)&:=\frac{1}{\lambda(t)}\left (\av^l + \left [ \begin{array}{c}  W^{-1} \tilde \nv_\Gamma(t) \\ 0_{n_l-n_{l+1}} \end{array} \right ]+ \left [ \begin{array}{c}  W^{-1} \delta(t)  \\ 0_{n_l-n_{l+1}} \end{array} \right ] \right ).\\ 
\end{split}
\end{equation*}
We then have $\tilde \nv_\Gamma(0)=0$, and we choose $\lambda(0)=1$ and $\delta(0)=0$. This implies that at $t=0$ we have $\av_\Gamma^l(0)= \av^l \in \text{Im}(\sigma)^{n_l\times N}$. Further
\begin{equation*}
\begin{split}
\w_\Gamma^{l+1}(t)\cdot \av_\Gamma^l(t)&+\w^{l+1}_{\Gamma,0}(t) = \lambda(t)\cdot \w^{l+1}\cdot \av_\Gamma^l(t) + \w_0^{l+1}- \delta(t)  \\
 &=\w^{l+1} \av^l + \w_0^{l+1} + \underbrace{\w^{l+1}}_{=[W \ \ast ]}  \left ( \left [ \begin{array}{c} W^{-1} \tilde \nv_\Gamma(t) \\ 0_{n_l-n_{l+1}} \end{array} \right ]+ \left [ \begin{array}{c} W^{-1}  \delta(t)  \\ 0_{n_l-n_{l+1}} \end{array} \right ] \right ) - \delta(t)\\
 &= \nv_\Gamma^{l+1}(0) + \tilde \nv_\Gamma(t) + \delta(t) - \delta(t) = \nv_\Gamma^{l+1}(t)\\
 \end{split}
\end{equation*}
as desired. Note that with $\delta(t)=0$ and $\lambda(t)=1$ for all $t$, we would obtain suitable paths with $\av_\Gamma^l(t)\in\RR^{n_l\times N}$, but due to the activations in the previous layer we must require that $\av_\Gamma^l(t) \in \text{Im}(\sigma)^{n_l\times N}$. Here, we use the full freedom of choosing $\delta(t)$ and $\lambda(t)$ to ensure this. In the case that $0\in (c,d)=\text{Im}(\sigma)$ it suffices to fix $\delta(t)=0$ and to always choose sufficiently large $\lambda(t)>0$ such that $\av_\Gamma^l(t) \in (c,d)^{n_l\times N}$. In the case that $0$ lies on the boundary of the interval $[c,d]$, we also need to to choose $\delta(t)$ to guarantee the correct sign in each component of $\av^l_{\Gamma,k}(t)$, i.e., if $c=0$ then choose $\delta(t)$ such that each entry of $W^{-1}\delta(t)$ is sufficiently large to guarantee that $\left (\av^l + \left [ \begin{array}{c}  W^{-1} \tilde \nv_\Gamma(t) \\ 0_{n_l-n_{l+1}} \end{array} \right ]+ \left [ \begin{array}{c}  W^{-1} \delta(t)  \\ 0_{n_l-n_{l+1}} \end{array} \right ] \right ) \in \RR_{>0}^{n_{l}\times N}$. 
\end{proof}

Combining Lemma~\ref{lma:ctsPath2} and \ref{lma:ctsPathInductive}, we can realize any path of activation values $\av^{l+1}_\Gamma(t)$ in layer $l+1$ by a path of parameters, if arbitrary paths of activation values $\av^{l}_\Gamma(t)$ can be realized in the previous layer $l$.  By Lemma~\ref{lma:ctsPath1} and~\ref{lma:ctsPath2}, arbitrary paths of activation values can be realized in the layer following the extremely wide layer. Hence, by induction over the layers we find that any path at the output is realizable. In the following result, we denote the dependence of the network function of its parameters $\w$ on the training sample $x_\alpha$ by $f(\w;x_\alpha)$.

\finalPath*

\begin{proof}
As outlined before the statement of the lemma, the proof only requires a composition of previous lemmas. We first show by induction over $l$ that, for each $l^*<l \leq L-1$, every continuous path $\av^{l}_\Gamma(t)\in C([0,1],\RR^{n_{l}\times N})$  can be realized by a continuous change of parameters from previous layers. That is, for all $l^*<l \leq L-1$ and all continuous paths $\av^{l}_\Gamma(t)\in C([0,1],\RR^{n_{l}\times N})$ starting at the activations values in layer $l$ for parameters $\w^l$, i.e., $\av^{l}_\Gamma(0)=\av^l(\w)$, there is a continuous path of parameters $\w_\Gamma(t)$ with $\w_\Gamma(0)=\w$ such that the activation values $\av$ as a function of $\w_\Gamma(t)$ satisfy $\av(\w_\Gamma(t))=\av_\Gamma(t)$. 

The base case for the induction, $l=l^*+1$, holds true by combining Lemma~\ref{lma:ctsPath1} and~\ref{lma:ctsPath2}. The induction step is further shown by combining Lemma~\ref{lma:ctsPath2} and \ref{lma:ctsPathInductive}. 

This guarantees all necessary assumptions to also apply Lemma~\ref{lma:ctsPathInductive} to the last layer, showing that any path $f_\Gamma(t)$ can be realized at the output.
\end{proof}

\begin{proof}\textbf{(Theorem~\ref{thm:pathToGlobal})}
Let $\w$ be a given set of parameters for the neural network and $\epsilon>0$. Applying Lemma~\ref{lma:ctsPath1} we find $\w'$ such that (i) $||\w-\w'||<\epsilon$, the activation vectors {$\av^{l^*}_k$} of the extremely wide layer {$l^*$} (containing more neurons than the number of training samples $N$) at parameters $\w'$ satisfy {$$\spn_k \av^{l^*}_k=\RR ^N,$$} {and (iii) the weight matrices $(\w')^l$ have full rank for all $l> l^*+1$.}

Part (ii) and (iii) together with the assumption on the architecture of the network guarantee the assumptions of Lemma~\ref{lma:finalPath} for $\w'$, so that for any continuous path $f_\Gamma:[0,1]\rightarrow \RR^N$ with $f_\Gamma(0)=[f(\w'; x_\alpha)]_\alpha$ there is a continuous path $\w'_\Gamma(t)$ with $\w'_\Gamma(0)=\w'$ and $f_\Gamma(t)=[f( \w'_\Gamma(t); x_\alpha)]_\alpha$. So we only need to specify a desired path at the output, which we can then realize by a continuous change of parameters of the neural network.

With fixed $z_\alpha=f(\w'; x_\alpha)$, the prediction for the current weights, let $\mathbf{z}=[z_\alpha]_\alpha$ denote the vector of predictions, and let $\mathbf{y}=[y_\alpha]_\alpha$ denote the vector of all target values. An obvious path of decreasing loss at the output layer is then given by $t\in [0,1] \mapsto \mathbf{z}+t\cdot (\mathbf{y}-\mathbf{z})$, inducing the loss $\Loss=||\mathbf{z}+t\cdot (\mathbf{y}-\mathbf{z})-\mathbf{y}||_2^2=(1-t)||\mathbf{y}-\mathbf{z}||_2^2$. 
\end{proof}

\bibliography{localVsglobal}


\newpage

\subsection{Calculations for Lemma~\ref{lma:hessian}}\label{app:calculations}

For the calculations we may assume without loss of generality that $r=1$. If we want to consider a different $\n{l,r}{x}$ and its corresponding $\gamma_\lambda^r$, then this can be achieved by a reordering of the indices of neurons.)

We let $\varphi$ denote the network function of the smaller neural network and $f$ the neural network function of the larger network after adding one neuron according to the map $\gamma_\lambda^1$. To distinguish the parameters of $f$ and $\varphi$, we write $w^\varphi$ for the parameters of the network before the embedding. This gives for all $i,s$ and all $m\geq 2$:
\begin{center}
\begin{tabular}{ccccccc}
$u_{-1,i}=u^\varphi_{1,i}$ & &  $u_{1,i}=u^\varphi_{1,i}$ & & 
$v_{s,-1}=\lambda v^\varphi_{s,1}$ & &  $v_{s,1}=(1-\lambda) v^\varphi_{s,1}$\\
$u_{m,i}=u^\varphi_{m,i}$ & &  $v_{s,m}=v^\varphi_{s,m}$ & & 
$\wrest = \wrest^\varphi$ & &  \\
\end{tabular}
\end{center}

We do the same for neuron vectors and activation vectors. Using that $f$ can be considered a composition of functions from consecutive layers, we denote the function from $\n{k}{x}$ to the output by $\h{k}{x}$ and we also specify by usage of an upper-script $\varphi$ when the function $h_{\parm,l+1}$ belongs to the network before the embedding. 

Key to the computation is the fact that all derivatives of $f$ can be naturally written as derivatives of $\varphi$. Concretely, implied by the embedding, all values at neurons $\n{l,i}{x}$ and their activation values $\act{l,i}{x}$ remain unchanged, i.e., we have for all $m\geq 1$ and all $\tilde{l}\neq l$ that 

\begin{center}
\begin{tabular}{ccccc}
$\act{l,-1}{x}=\actphi{l,1}{x}$ & &  $\act{l,m}{x}=\actphi{l,m}{x}$ & & $\act{\tilde{l},m}{x}=\actphi{\tilde{l},m}{x}$ \\
$\n{l,-1}{x}=\nphi{l,1}{x}$ & &  $\n{l,m}{x}=\nphi{l,m}{x}$ & & $\n{\tilde{l},m}{x}=\nphi{\tilde{l},m}{x}$ \\
\end{tabular}
\end{center}
%
%
%
%

\subsubsection{First order derivatives of network functions $f$ and $\varphi$.}
For the function $f$ we have the following partial derivatives.
$$\frac{\partial f(x)}{\partial u_{p,i}} = \sum_k \frac{\partial \h{l+1}{\n{l+1}{x}}}{\partial \n{l+1,k}{x}} \cdot v_{k,p}\cdot \sigma'(\n{l,p}{x} ) \cdot \act{l-1,i}{x}$$
and 
$$\frac{\partial f(x)}{\partial v_{s,q}} =  \frac{\partial \h{l+1}{\n{l+1}{x}}}{\partial \n{l+1,s}{x}} \cdot \act{l,q}{x}$$
The analogous equations hold for $\varphi$. 

\subsubsection{Relating first order derivatives of network functions $f$ and $\varphi$}

Therefore, at $([u_{1,i}]_i,[v_{s,1}]_s,\wrest)$ and $\gamma_\lambda^1([u_{1,i}]_i,[v_{s,1}]_s,\wrest)$ respectively, we get for $k=-1,1$ that 
\begin{equation*}
\frac{\partial f(x)}{\partial u_{-1,i}} = \lambda \frac{\partial \varphi(x)}{\partial u^\varphi_{1,i}}\ \text{, and }\
\frac{\partial f(x)}{\partial u_{1,i}} = (1-\lambda) \frac{\partial \varphi(x)}{\partial u^\varphi_{1,i}}\ \text{, and }\
\frac{\partial f(x)}{\partial v_{s,k}} =  \frac{\partial \varphi(x)}{\partial v^\varphi_{s,1}} \\
\end{equation*}
and for $k\geq 2$ we get that 
\begin{equation*}
\frac{\partial f(x)}{\partial u_{k,i}} =  \frac{\partial \varphi(x)}{\partial u^\varphi_{k,i}}\text{, and}\
\frac{\partial f(x)}{\partial v_{s,k}} =  \frac{\partial \varphi(x)}{\partial v^\varphi_{s,k}}.
\end{equation*}

\subsubsection{Second order derivatives of network functions $f$ and $\varphi$.}

For the second derivatives we get (with $\delta(a,a)=1$ and $\delta(a,b)=0$ for $a\neq b$) 

\begin{equation*}
\begin{split}\frac{\partial^2 f(x)}{\partial u_{p,i} \partial u_{q,j}} & = \frac{\partial}{\partial u_{q,j}}\left (
\sum_k \frac{\partial \h{l+1}{\n{l+1}{x}}}{\partial \n{l+1,k}{x}} \cdot v_{k,p}\cdot \sigma'(\n{l,p}{x} ) \cdot \act{l-1,i}{x}
 \right )\\
&=\sum_m \sum_k  \frac{\partial^2 \h{l+1}{\n{l+1}{x}}}{\partial \n{l+1,m}{x} \partial \n{l+1,k}{x}}  \cdot v_{m,q}\cdot \sigma'(\n{l,q}{x} ) \cdot  \act{l-1,j}{x}  \\ &  \cdot v_{k,p}\cdot \sigma'(\n{l,p}{x} ) \cdot \act{l-1,i}{x} \\
& + \delta(p,q)  \sum_k \frac{\partial \h{l+1}{\n{l+1}{x}}}{\partial \n{l+1,k}{x}} \cdot v_{k,p}\cdot \sigma''(\n{l,p}{x} ) \\ \cdot & \act{l-1,i}{x}\cdot \act{l-1,j}{x}
\end{split}
\end{equation*}
and
\begin{equation*}
\begin{split}\frac{\partial^2 f(x)}{\partial v_{s,p} \partial v_{t,q}} & = \frac{\partial}{\partial v_{t,q}}\left (  \frac{\partial \h{l+1}{\n{l+1}{x}}}{\partial \n{l+1,s}{x}} \cdot \act{l,p}{x} \right )\\
& = \frac{\partial^2 \h{l+1}{\n{l+1}{x}}}{\partial \n{l+1,s}{x}\partial \n{l+1,t}{x}} \cdot \act{l,p}{x} \cdot \act{l,q}{x}
\end{split}
\end{equation*}
and
\begin{equation*}
\begin{split}\frac{\partial^2 f(x)}{\partial u_{p,i} \partial v_{s,q}} & = \frac{\partial}{\partial v_{s,q}}\left (
\sum_k \frac{\partial \h{l+1}{\n{l+1}{x}}}{\partial \n{l+1,k}{x}} \cdot v_{k,p}\cdot \sigma'(\n{l,p}{x} ) \cdot \act{l-1,i}{x}
 \right )\\
&= \sum_k  \frac{\partial^2 \h{l+1}{\n{l+1}{x}}}{\partial \n{l+1,s}{x} \partial \n{l+1,k}{x}}  \cdot \act{l,q}{x} \cdot v_{k,p} \\ & \cdot \sigma'(\n{l,p}{x}) \cdot \act{l-1,i}{x}  \\
& +  \delta (q,p) \cdot \frac{\partial \h{l+1}{\n{l+1}{x}}}{\partial \n{l+1,s}{x}} \cdot \sigma'(\n{l,p}{x}) \cdot \act{l-1,i}{x}
\end{split}
\end{equation*}
For a parameter $w$ closer to the input than $[u_{p,i}]_{p,i},[v_{s,q}]_{s,q}$, we have
\begin{equation*}
\begin{split}\frac{\partial^2 f(x)}{\partial u_{p,i} \partial w} & = \frac{\partial}{\partial w}\left (
\sum_k \frac{\partial \h{l+1}{\n{l+1}{x}}}{\partial \n{l+1,k}{x}} \cdot v_{k,p}\cdot \sigma'(\n{l,p}{x} ) \cdot \act{l-1,i}{x}
 \right )\\
 & = \sum_m \sum_k \frac{\partial \h{l+1}{\n{l+1}{x}}}{\partial \n{l+1,k}{x} \partial \n{l+1,m}{x}} \cdot \frac{\partial \n{l+1,m}{x}}{\partial w} \cdot v_{k,p}\\ & \cdot \sigma'(\n{l,p}{x} ) \cdot \act{l-1,i}{x}\\ & +
\sum_k \frac{\partial \h{l+1}{\n{l+1}{x}}}{\partial \n{l+1,k}{x}} \cdot v_{k,p}\cdot \sigma''(\n{l,p}{x} ) \cdot  \frac{\partial \n{l,p}{x}}{\partial w}\cdot  \act{l-1,i}{x} \\
&+  \sum_k \frac{\partial \h{l+1}{\n{l+1}{x}}}{\partial \n{l+1,k}{x}} \cdot v_{k,p}\cdot \sigma'(\n{l,p}{x} ) \cdot \frac{\partial \act{l-1,i}{x}}{\partial w}
\end{split}
\end{equation*}
and 
\begin{equation*}
\begin{split}\frac{\partial^2 f(x)}{\partial v_{s,q} \partial w} & = \frac{\partial}{\partial w}\left (  \frac{\partial \h{l+1}{\n{l+1}{x}}}{\partial \n{l+1,s}{x}} \cdot \act{l,q}{x} \right )\\
& = \sum_n \frac{\partial^2 \h{l+1}{\n{l+1}{x}}}{\partial \n{l+1,s}{x}\partial \n{l+1,n}{x}}\cdot \frac{ \partial \n{l+1,n}{x}}{\partial w} \cdot \act{l,q}{x} \cdot \act{l,q}{x} \\
& + \frac{\partial \h{l+1}{\n{l+1}{x}}}{\partial \n{l+1,s}{x}} \cdot \frac{\partial \act{l,q}{x}}{\partial w}
\end{split}
\end{equation*}
For a parameter $w$ closer to the output than $[u_{p,i}]_{p,i},[v_{s,q}]_{s,q}$, we have
\begin{equation*}
\begin{split}\frac{\partial^2 f(x)}{\partial u_{p,i} \partial w} & = \frac{\partial}{\partial w}\left (
\sum_k \frac{\partial \h{l+1}{\n{l+1}{x}}}{\partial \n{l+1,k}{x}} \cdot v_{k,p}\cdot \sigma'(\n{l,p}{x} ) \cdot \act{l-1,i}{x}
 \right )\\
 &= \sum_k \frac{\partial^2 \h{l+1}{\n{l+1}{x}}}{\partial \n{l+1,k}{x}\partial w} \cdot v_{k,p}\cdot \sigma'(\n{l,p}{x} ) \cdot \act{l-1,i}{x}
\end{split}
\end{equation*}


 

\subsubsection{Relating second order derivatives of network functions $f$ and $\varphi$}

To relate the second derivatives of $f$ at $\gamma_\lambda^1([u_{1,i}]_i,[v_{s,1}]_s,\wrest)$ to the second derivatives of $\varphi$ at $([u_{1,i}]_i,[v_{s,1}]_s,\wrest)$, we define 
\begin{equation*}
\begin{split}A_{i,j}^{p,q}(x) & := \sum_m \sum_k \frac{\partial^2 \hphi{l+1}{\nphi{l+1}{x}}}{\partial \nphi{l+1,m}{x} \partial \nphi{l+1,k}{x}} \\ &  \cdot v^\varphi_{m,q}\cdot \sigma'(\nphi{l,q}{x} ) \cdot  \actphi{l-1,j}{x} \cdot  v^\varphi_{k,p}\cdot \sigma'(\nphi{l,p}{x} ) \cdot \actphi{l-1,i}{x}\\
B_{i,j}^{p}(x)& := \sum_k \frac{\partial \hphi{l+1}{\nphi{l+1}{x}}}{\partial \nphi{l+1,k}{x}} \cdot v^\varphi_{k,p}\cdot \sigma''(\nphi{l,p}{x} ) \cdot \actphi{l-1,i}{x}\cdot \actphi{l-1,j}{x}\\
C_{i,q}^{p,s}(x)& := \sum_k  \frac{\partial^2 \hphi{l+1}{\n{l+1}{x}}}{\partial \nphi{l+1,s}{x} \partial \nphi{l+1,k}{x}}  \cdot \actphi{l,q}{x} \cdot  v^\varphi_{k,p} \cdot \sigma'(\nphi{l,p}{x}) \cdot \actphi{l-1,i}{x} \\
D_{i}^{p,s}(x)& := \frac{\partial \hphi{l+1}{\nphi{l+1}{x}}}{\partial \nphi{l+1,s}{x}} \cdot \sigma'(\nphi{l,p}{x}) \cdot \actphi{l-1,i}{x}\\
E_{p,q}^{s,t}(x)&:= \frac{\partial^2 \hphi{l+1}{\nphi{l+1}{x}}}{\partial \nphi{l+1,s}{x}\partial \nphi{l+1,t}{x}} \cdot \actphi{l,p}{x} \cdot \actphi{l,q}{x}
\end{split}
\end{equation*}
Then for all $i,j,p,q,s,t$, we have
\begin{equation*}
\begin{split}\frac{\partial^2 \varphi (x)}{\partial u^\varphi_{p,i} \partial u^\varphi_{q,j}} &=A_{i,j}^{p,q}(x) + \delta(q,p) B_{i,j}^{p}(x)\\
\frac{\partial^2 \varphi(x)}{\partial u^\varphi_{p,i} \partial v^\varphi_{s,q}} &= C_{i,q}^{p,s}(x) + \delta(q,p) D_{i}^{p,s}(x)\\
\frac{\partial^2 \varphi(x)}{\partial v_{s,p} \partial v_{t,q}} & = E_{p,q}^{s,t}(x)
\end{split}
\end{equation*}

For $f$ we get for $p,q\in\{-1,1\}$ and all $i,j,s,t$ 
\begin{equation*}
\begin{split}\frac{\partial^2 f (x)}{\partial u_{-1,i} \partial u_{-1,j}} &=\lambda^2 A_{i,j}^{1,1}(x) + \lambda B_{i,j}^{1}(x)  \\ \frac{\partial^2 f (x)}{\partial u_{1,i} \partial u_{1,j}} &=(1-\lambda)^2 A_{i,j}^{1,1}(x) + (1-\lambda) B_{i,j}^{1}(x)\\
\frac{\partial^2 f (x)}{\partial u_{-1,i} \partial u_{1,j}}  &= \frac{\partial^2 f (x)}{\partial u_{1,i} \partial u_{-1,j}}= \lambda(1-\lambda)\cdot A_{i,j}^{1,1}(x)\\
\frac{\partial^2 f (x)}{\partial u_{-1,i} \partial v_{s,-1}} &= \lambda C_{i,1}^{1,s}(x) +  D_{i}^{1,s}(x)\\
\frac{\partial^2 f (x)}{\partial u_{1,i} \partial v_{s,1}} &= (1-\lambda) C_{i,1}^{1,s}(x) +  D_{i}^{1,s}(x)\\
\frac{\partial^2 f (x)}{\partial u_{-1,i} \partial v_{s,1}}  & = \lambda \cdot C_{i,1}^{1,s}(x)=\lambda \cdot \frac{\partial^2 \varphi (x)}{\partial u^\varphi_{1,i} \partial v^\varphi_{s,1}} \\
\frac{\partial^2 f (x)}{\partial u_{1,i} \partial v_{s,-1}}  & = (1-\lambda)\cdot  C_{i,1}^{1,s}(x)=(1-\lambda)\cdot \frac{\partial^2 \varphi (x)}{\partial u^\varphi_{1,i} \partial v^\varphi_{s,1}} \\
\frac{\partial^2 f(x)}{\partial v_{s,p} \partial v_{t,q}} & = E_{1,1}^{s,t}(x)=\frac{\partial^2 \varphi(x)}{\partial v^\varphi_{s,1} \partial v^\varphi_{t,1}}
\end{split}
\end{equation*}
and for $q\geq 2$ and $p\in\{-1,1\}$ and all $i,j,s,t$ 
\begin{equation*}
\begin{split}\frac{\partial^2 f (x)}{\partial u_{-1,i} \partial u_{q,j}}  & = \lambda A_{i,j}^{1,q}(x)  = \lambda\cdot  \frac{\partial^2 \varphi (x)}{\partial u^\varphi_{1,i} \partial u^\varphi_{q,j}}\\
\frac{\partial^2 f (x)}{\partial u_{1,i} \partial u_{q,j}}  & = (1-\lambda) A_{i,j}^{1,q}(x) =(1-\lambda)\cdot  \frac{\partial^2 \varphi (x)}{\partial u^\varphi_{1,i} \partial u^\varphi_{q,j}}\\
\frac{\partial^2 f (x)}{\partial u_{-1,i} \partial v_{s,q}}  & = \lambda C_{i,q}^{1,s}(x)  = \lambda \cdot \frac{\partial^2 \varphi (x)}{\partial u^\varphi_{1,i} \partial v^\varphi_{s,q}}\\
\frac{\partial^2 f (x)}{\partial u_{1,i} \partial v_{s,q}}  & = (1-\lambda) C_{i,q}^{1,s}(x)  = (1-\lambda) \frac{\partial^2 \varphi (x)}{\partial u^\varphi_{1,i} \partial v^\varphi_{s,q}}\\
\frac{\partial^2 f (x)}{\partial u_{q,i} \partial v_{s,p}}  & = C_{i,1}^{q,s}(x)  =\frac{\partial^2 \varphi (x)}{\partial u^\varphi_{q,i} \partial v^\varphi_{s,1}} \\
\frac{\partial^2 f(x)}{\partial v_{s,p} \partial v_{t,q}} & = E_{1 ,q}^{s,t}(x)=\frac{\partial^2 \varphi(x)}{\partial v^\varphi_{s,1} \partial v^\varphi_{t,q}}\\
\end{split}
\end{equation*}
and for $p,q\geq 2$ and all $i,j,s,t$ 
\begin{equation*}
\begin{split}\frac{\partial^2 f (x)}{\partial u_{p,i} \partial u_{q,j}} &= A_{i,j}^{p,q}(x) + \delta(q,p) B_{i,j}^{p}(x)=\frac{\partial^2 \varphi (x)}{\partial u^\varphi_{p,i} \partial u^\varphi_{q,j}}\\
\frac{\partial^2 f (x)}{\partial u_{p,i} \partial v_{s,q}} &= C_{i,q}^{p,s}(x) + \delta(q,p) D_{i}^{p,s}(x)=\frac{\partial^2 \varphi (x)}{\partial u^\varphi_{p,i} \partial v^\varphi_{s,q}}\\
\frac{\partial^2 f(x)}{\partial v_{s,p} \partial v_{t,q}} & = E_{p ,q}^{s,t}(x)=\frac{\partial^2 \varphi(x)}{\partial v^\varphi_{s,p} \partial v^\varphi_{t,q}}
\end{split}
\end{equation*}

\subsubsection{Derivatives of $\Loss$ and $\loss$}

Let 
$$A_{i,j}^{p,q} :=\sum_\alpha (\varphi(x_\alpha)- y_\alpha) \cdot A_{i,j}^{p,q}(x_\alpha)$$
$$B_{i,j}^{p} := \sum_\alpha (\varphi(x_\alpha)- y_\alpha) \cdot B_{i,j}^{p}(x_\alpha)$$
$$C_{i,q}^{p,s} :=\sum_\alpha (\varphi(x_\alpha)- y_\alpha) \cdot C_{i,q}^{p,s}(x_\alpha) $$
$$D_{i}^{p,s}  :=\sum_\alpha (\varphi(x_\alpha)- y_\alpha) \cdot D_{i}^{p,s}(x_\alpha)$$
$$E_{p,q}^{s,t} :=\sum_\alpha (\varphi(x_\alpha)- y_\alpha) \cdot E_{p,q}^{s,t}(x_\alpha) $$
and
$$\mathcal{A}_{i,j}^{p,q}:= \sum_\alpha \left ( \frac{\partial \varphi(x_\alpha)}{\partial u^\varphi_{p,i}} \right ) \cdot \left ( \frac{\partial \varphi(x_\alpha)}{\partial u^\varphi_{q,j}}\right ) $$
$$\mathcal{C}_{i,q}^{p,s}:= \sum_\alpha \left ( \frac{\partial \varphi(x_\alpha)}{\partial u^\varphi_{p,i}} \right ) \cdot \left ( \frac{\partial \varphi(x_\alpha)}{\partial v^\varphi_{s,q}}\right ) $$
$$\mathcal{E}_{p,q}^{s,t}:= \sum_\alpha \left ( \frac{\partial \varphi(x_\alpha)}{\partial v^\varphi_{s,p}} \right ) \cdot \left ( \frac{\partial \varphi(x_\alpha)}{\partial v^\varphi_{t,q}}\right ) $$

Now we have everything together to compute the derivatives of the loss. For the first derivative of the loss, we have for any variables $w,r$ that $$\frac{\partial \Loss}{\partial w}=\sum_\alpha (f(x_\alpha)-y_\alpha)\cdot \frac{\partial f(x_\alpha )}{\partial w}$$
and
$$\frac{\partial \loss}{\partial w^\varphi}=\sum_\alpha (\varphi(x_\alpha)-y_\alpha)\cdot \frac{\partial \varphi(x_\alpha )}{\partial w^\varphi}.$$
From this it follows immediately that if $\frac{\partial \loss}{\partial w^\varphi}(\w^\varphi)=0$, then $\frac{\partial \Loss}{\partial w}(\gamma_\lambda^1(\w^\varphi))=0$ for all $\lambda$ \citep[cf.][]{FukumizuAmari, Nitta}.

For the second derivative we get
$$\frac{\partial^2 \Loss}{\partial w \partial r}=\sum_\alpha (f(x_\alpha)-y_\alpha)\cdot \frac{\partial^2 f(x_\alpha )}{\partial w \partial r} + \sum_\alpha \left ( \frac{\partial f(x_\alpha )}{\partial w}\right ) \cdot \left (  \frac{\partial f(x_\alpha )}{\partial r}\right ) $$
This leads to the following equations for $\loss$
\begin{equation*}
\begin{split}
\frac{\partial^2 \loss}{\partial u^\varphi_{p,i} \partial u^\varphi_{q,j}}&= A_{p,q}^{1,1} +  \delta(p,q) B_{i,j}^p + \mathcal{A}_{i,j}^{p,q}\\
\frac{\partial^2 \loss}{\partial u^\varphi_{p,i} \partial v^\varphi_{s,q}}&= C_{i,q}^{p,s} + \delta(p,q) D_{i}^{p,s} +  \mathcal{C}_{i,q}^{p,s}\\
\frac{\partial^2 \loss}{\partial v^\varphi_{s,p} \partial v_{t,q}}&=  E_{p,q}^{s,t} + \mathcal{E}_{p,q}^{s,t} \\
\end{split}
\end{equation*}
For $\Loss$ we get for $p,q\in\{-1,1\}$ and all $i,j,s,t$ at $\gamma_\lambda^1([u_{1,i}]_i,[v_{s,1}]_s,\wrest)$
\begin{equation*}
\begin{split}
\frac{\partial^2 \Loss}{\partial u_{-1,i} \partial u_{-1,j}}&= \lambda^2 A_{i,j}^{1,1} + \lambda B_{i,j}^1 + \lambda^2  \mathcal{A}_{i,j}^{1,1}\\
\frac{\partial^2 \Loss}{\partial u_{1,i} \partial u_{1,j}}&= (1-\lambda)^2 A_{i,j}^{1,1} + (1-\lambda) B_{i,j}^1 + (1-\lambda)^2 \mathcal{A}_{i,j}^{1,1}\\
\frac{\partial^2 \Loss}{\partial u_{-1,i} \partial u_{1,j}}&= \lambda(1-\lambda) A_{i,j}^{1,1} + \lambda(1-\lambda) \mathcal{A}_{i,j}^{1,1}\\
\frac{\partial^2 \Loss}{\partial u_{-1,i} \partial v_{s,-1}}&= \lambda C_{i,1}^{1,s} + D_{i}^{1,s} + \lambda \mathcal{C}_{i,1}^{1,s}\\
\frac{\partial^2 \Loss}{\partial u_{1,i} \partial v_{s,1}}&= (1-\lambda) C_{i,1}^{1,s} + D_{i}^{1,s} + (1-\lambda) \mathcal{C}_{i,1}^{1,s}\\
\frac{\partial^2 \Loss}{\partial u_{-1,i} \partial v_{s,1}}&= \lambda C_{i,1}^{1,s}  + \lambda \mathcal{C}_{i,1}^{1,s}\\
\frac{\partial^2 \Loss}{\partial u_{1,i} \partial v_{s,-1}}&= (1-\lambda) C_{i,1}^{1,s}  + (1-\lambda) \mathcal{C}_{i,1}^{1,s}\\
\frac{\partial^2 \Loss}{\partial v_{s,p} \partial v_{t,q}}&=  E_{1,1}^{s,t} + \mathcal{E}_{1,1}^{s,t} \\
\end{split}
\end{equation*}
and for $q\geq 2$ and $p\in\{-1,1\}$ and all $i,j,s,t$ 
\begin{equation*}
\begin{split}\frac{\partial^2 \Loss}{\partial u_{-1,i} \partial u_{q,j}}  & = \lambda A_{i,j}^{1,q} +\lambda  \mathcal{A}_{i,j}^{1,1} \\
\frac{\partial^2 \Loss}{\partial u_{1,i} \partial u_{q,j}}  & = (1-\lambda) A_{i,j}^{1,q} + (1-\lambda) \mathcal{A}_{i,j}^{1,q}\\
\frac{\partial^2 \Loss}{\partial u_{-1,i} \partial v_{s,q}}  & = \lambda C_{i,q}^{1,s} + \lambda \mathcal{C}_{i,q}^{1,s} \\
\frac{\partial^2 \Loss}{\partial u_{1,i} \partial v_{s,q}}  & = (1-\lambda) C_{i,q}^{1,s} +(1-\lambda) \mathcal{C}_{i,q}^{1,s} \\
\frac{\partial^2 \Loss }{\partial u_{q,i} \partial v_{s,p}}  & = C_{i,p}^{q,s} + \mathcal{C}_{i,p}^{q,s} \\
\frac{\partial^2 \Loss}{\partial v_{s,p} \partial v_{t,q}} & = E_{1 ,q}^{s,t} + \mathcal{E}_{1 ,q}^{s,t} \\
\end{split}
\end{equation*}
and for $p,q\geq 2$ and all $i,j,s,t$ 
\begin{equation*}
\begin{split}\frac{\partial^2 \Loss }{\partial u_{p,i} \partial u_{q,j}} &= A_{i,j}^{p,q} + \delta(q,p) B_{i,j}^{p}(x)+ \mathcal{A}_{i,j}^{p,q}=\frac{\partial^2 \loss}{\partial u^\varphi_{p,i} \partial u^\varphi_{q,j}}\\
\frac{\partial^2 \Loss }{\partial u_{p,i} \partial v_{s,q}} &= C_{i,q}^{p,s} + \delta(q,p) D_{i}^{p,s}+\mathcal{C}_{i,q}^{p,s}=\frac{\partial^2 \loss}{\partial u^\varphi_{p,i} \partial v^\varphi_{s,q}}\\
\frac{\partial^2 \Loss}{\partial v_{s,p} \partial v_{t,q}} & = E_{p ,q}^{s,t}+\mathcal{E}_{p ,q}^{s,t}=\frac{\partial^2 \loss}{\partial v^\varphi_{s,p} \partial v^\varphi_{t,q}}
\end{split}
\end{equation*}

\subsubsection{Change of basis}

Choose any real numbers $\alpha\neq - \beta$ such that $\lambda=\frac{\beta}{\alpha+\beta}$ (equivalently $\alpha\lambda-\beta (1-\lambda)=0$) and set 

\begin{center}
\begin{tabular}{ccc}
$\mu_{-1,i}=u_{-1,i}+u_{1,i}$& & $\mu_{1,i}=\alpha \cdot u_{-1,i} - \beta\cdot u_{1,i}$ \\
 $\nu_{s,-1}=v_{s,-1}+v_{s,1}$ & & $\nu_{s,1}=v_{s,-1}-v_{s,1}$. 
\end{tabular}
\end{center}




Then at $\gamma_\lambda^1([u_{1,i}]_i,[v_{s,1}]_s,\wrest)$,
\begin{equation*}
\begin{split}
\frac{\partial^2 \Loss}{\partial \mu_{-1,i} \partial \mu_{-1,j}} &=\left ( \frac{\partial  }{\partial u_{-1,i}} + \frac{\partial}{\partial u_{1,i}}\right ) \left ( \frac{\partial \Loss(x)}{\partial u_{-1,j}} + \frac{\partial \Loss(x) }{\partial u_{1,j}}\right ) 
\\ & = \frac{\partial^2 L(x)}{\partial u_{-1,i} \partial u_{-1,j}} + \frac{\partial^2 \Loss(x) }{\partial u_{-1,i}\partial u_{1,j}}
+ \frac{\partial^2 \Loss(x)}{\partial u_{1,i}\partial u_{-1,j}} + \frac{\partial^2 \Loss (x) }{\partial u_{1,i}\partial u_{1,j}}
\\ & = \left ( \lambda^2 A_{i,j}^{1,1} + \lambda B_{i.j}^1 + \lambda^2 \mathcal{A}_{i,j}^{1,1} \right ) + \left ( \lambda(1-\lambda) A_{i,j}^{1,1} + \lambda(1-\lambda) \mathcal{A}_{i,j}^{1,1} \right ) \\
& + \left ( \lambda(1-\lambda) A_{i,j}^{1,1} +\lambda(1-\lambda) \mathcal{A}_{i,j}^{1,1} \right )  + \left ( (1-\lambda)^2 A_{i,j}^{1,1} + (1-\lambda) B_{i.j}^1 + (1-\lambda)^2 \mathcal{A}_{i,j}^{1,1} \right )
\\ & =  A_{i,j}^{1,1} + B_{i.j}^1 + \mathcal{A}_{i,j}^{1,1}
\end{split}
\end{equation*}

\begin{equation*}
\begin{split}
\frac{\partial^2 \Loss}{\partial \mu_{1,i} \partial \mu_{1,j}} &=\left (\alpha \frac{\partial  }{ \partial u_{-1,i}} -\beta \frac{\partial}{\partial u_{1,i}}\right )\left (\alpha \frac{\partial \Loss(x)}{\partial u_{-1,j}} -\beta \frac{\partial \Loss(x) }{\partial u_{1,j}}\right )
\\ & = \alpha^2 \frac{\partial^2 \Loss(x)}{\partial u_{-1,i} \partial u_{-1,j}} -\alpha\beta \frac{\partial^2 \Loss(x) }{\partial u_{-1,i}\partial u_{1,j}}
- \alpha\beta\frac{\partial^2 \Loss(x)}{\partial u_{1,i}\partial u_{-1,j}} + \beta^2 \frac{\partial^2 \Loss(x) }{\partial u_{1,i}\partial u_{1,j}}
\\ & =\alpha^2 \left ( \lambda^2 A_{i,j}^{1,1} + \lambda B_{i.j}^1 + \lambda^2 \mathcal{A}_{i,j}^{1,1} \right ) -\alpha\beta \left ( \lambda(1-\lambda) A_{i,j}^{1,1} + \lambda(1-\lambda) \mathcal{A}_{i,j}^{1,1} \right ) \\ 
& - \alpha\beta \left ( \lambda(1-\lambda) A_{i,j}^{1,1} + \lambda(1-\lambda) \mathcal{A}_{i,j}^{1,1} \right ) \\ &  +\beta^2 \left ( (1-\lambda)^2 A_{i,j}^{1,1} + (1-\lambda) B_{i.j}^1 + (1-\lambda)^2 \mathcal{A}_{i,j}^{1,1} \right )
\\ & =   \alpha\beta B_{i.j}^1 
\end{split}
\end{equation*}

\begin{equation*}
\begin{split}
\frac{\partial^2 \Loss}{\partial \mu_{-1,i} \partial \mu_{1,j}} &=\left ( \frac{\partial  }{\partial u_{-1,i}} + \frac{\partial}{\partial u_{1,i}}\right )\left (\alpha \frac{\partial \Loss(x)}{\partial u_{-1,j}} -\beta \frac{\partial \Loss(x) }{\partial u_{1,j}}\right )
\\ & =\alpha \frac{\partial^2 \Loss(x)}{\partial u_{-1,i} \partial u_{-1,j}} -\beta \frac{\partial^2 \Loss(x) }{\partial u_{-1,i}\partial u_{1,j}}
+ \alpha \frac{\partial^2 \Loss(x)}{\partial u_{1,i}\partial u_{-1,j}} -\beta \frac{\partial^2 \Loss(x) }{\partial u_{1,i}\partial u_{1,j}}
\\ & = \alpha \left ( \lambda^2 A_{i,j}^{1,1} + \lambda B_{i.j}^2 + \lambda^2 \mathcal{A}_{i,j}^{1,1} \right ) - \beta \left ( \lambda(1-\lambda) A_{i,j}^{1,1} + \lambda(1-\lambda) \mathcal{A}_{i,j}^{1,1} \right )\\
& + \alpha \left ( \lambda(1-\lambda) A_{i,j}^{1,1} + \lambda(1-\lambda) \mathcal{A}_{i,j}^{1,1} \right ) -\beta \left ((1-\lambda)^2  A_{i,j}^{1,1} + (1-\lambda) B_{i.j}^2 + (1-\lambda)^2  \mathcal{A}_{i,j}^{1,1} \right )
\\ & = 0
\end{split}
\end{equation*}

\begin{equation*}
\begin{split}
\frac{\partial^2 \Loss}{\partial \nu_{s,-1} \partial \nu_{t,-1}} &=\left ( \frac{\partial  }{\partial v_{s,-1}} + \frac{\partial}{\partial v_{s,1}}\right )\left ( \frac{\partial \Loss(x)}{\partial v_{t,-1}} + \frac{\partial \Loss(x) }{\partial v_{t,1}}\right )
\\ & =\frac{\partial^2 \Loss(x)}{\partial v_{s,-1} \partial v_{t,-1}} + \frac{\partial^2 \Loss(x) }{\partial v_{s,-1}\partial v_{t,1}}
+ \frac{\partial^2 \Loss(x)}{\partial v_{s,1}\partial v_{t,-1}} + \frac{\partial^2 \Loss(x) }{\partial v_{s,1}\partial v_{t,1}}
\\ & = \left ( E_{1,1}^{s,t} + \mathcal{E}_{1,1}^{s,t}\right ) + \left ( E_{1,1}^{s,t} + \mathcal{E}_{1,1}^{s,t}\right ) + \left ( E_{1,1}^{s,t} + \mathcal{E}_{1,1}^{s,t} \right ) + \left ( E_{1,1}^{s,t} + \mathcal{E}_{1,1}^{s,t} \right )
\\ & = 4 E_{1,1}^{s,t} + 4 \mathcal{E}_{1,1}^{s,t}
\end{split}
\end{equation*}

\begin{equation*}
\begin{split}
\frac{\partial^2 \Loss}{\partial \nu_{s,1} \partial \nu_{t,1}} &=\left ( \frac{\partial  }{\partial v_{s,-1}} - \frac{\partial}{\partial v_{s,1}}\right )\left ( \frac{\partial \Loss(x)}{\partial v_{t,-1}} - \frac{\partial \Loss(x) }{\partial v_{t,1}}\right )
\\ & =\frac{\partial^2 \Loss(x)}{\partial v_{s,-1} \partial v_{t,-1}} - \frac{\partial^2 \Loss(x) }{\partial v_{s,-1}\partial v_{t,1}}
- \frac{\partial^2 \Loss(x)}{\partial v_{s,1}\partial v_{t,-1}} + \frac{\partial^2 \Loss(x) }{\partial v_{s,1}\partial v_{t,1}}
\\ & = \left ( E_{1,1}^{s,t} + \mathcal{E}_{1,1}^{s,t} \right ) -\left (  E_{1,1}^{s,t} + \mathcal{E}_{1,1}^{s,t} \right )
- \left ( E_{1,1}^{s,t} + \mathcal{E}_{1,1}^{s,t}  \right ) + \left ( E_{1,1}^{s,t} + \mathcal{E}_{1,1}^{s,t} \right )
\\ & = 0
\end{split}
\end{equation*}

\begin{equation*}
\begin{split}
\frac{\partial^2 \Loss}{\partial \nu_{s,-1} \partial \nu_{t,1}} &=\left ( \frac{\partial  }{\partial v_{s,-1}} + \frac{\partial}{\partial v_{s,1}}\right )\left ( \frac{\partial \Loss(x)}{\partial v_{t,-1}} - \frac{\partial \Loss(x) }{\partial v_{t,1}}\right )
\\ & =\frac{\partial^2 \Loss(x)}{\partial v_{s,-1} \partial v_{t,-1}} - \frac{\partial^2 \Loss(x) }{\partial v_{s,-1}\partial v_{t,1}}
+ \frac{\partial^2 \Loss(x)}{\partial v_{s,1}\partial v_{t,-1}} - \frac{\partial^2 \Loss(x) }{\partial v_{s,1}\partial v_{t,1}}
\\ & =\left ( E_{1,1}^{s,t} + \mathcal{E}_{1,1}^{s,t} \right ) - \left ( E_{1,1}^{s,t} + \mathcal{E}_{1,1}^{s,t} \right )
+ \left ( E_{1,1}^{s,t} + \mathcal{E}_{1,1}^{s,t} \right )  - \left ( E_{1,1}^{s,t} + \mathcal{E}_{1,1}^{s,t} \right )
\\ & = 0
\end{split}
\end{equation*}

\begin{equation*}
\begin{split}
\frac{\partial^2 \Loss}{\partial \mu_{-1,i} \partial \nu_{s,-1}} &=\left ( \frac{\partial  }{\partial u_{-1,i}} + \frac{\partial}{\partial u_{1,i}}\right )\left ( \frac{\partial \Loss(x)}{\partial v_{s,-1}} + \frac{\partial \Loss(x) }{\partial v_{s,1}}\right )
\\ & =\frac{\partial^2 \Loss(x)}{\partial u_{-1,i} \partial v_{s,-1}} + \frac{\partial^2 \Loss(x) }{\partial u_{-1,i}\partial v_{s,1}}
+ \frac{\partial^2 \Loss(x)}{\partial u_{1,i}\partial v_{s,-1}} + \frac{\partial^2 \Loss(x) }{\partial u_{1,i}\partial v_{s,1}}
\\ & = \left ( \lambda C_{i,1}^{1,s} + D_{i}^{1,s} + \lambda \mathcal{C}_{i,1}^{1,s} \right ) + \left ( \lambda C_{i,1}^{1,s}  + \lambda \mathcal{C}_{i,1}^{1,s} \right )\\ & + \left ( (1-\lambda) C_{i,1}^{1,s}  + (1-\lambda) \mathcal{C}_{i,1}^{1,s} \right ) + \left ( (1-\lambda) C_{i,1}^{1,s} + D_{i}^{1,s} + (1-\lambda) \mathcal{C}_{i,1}^{1,s} \right )
\\ & = 2 C_{i,1}^{1,s} + 2 D_{i}^{1,s} + 2 \mathcal{C}_{i,1}^{1,s}
\end{split}
\end{equation*}

\begin{equation*}
\begin{split}
\frac{\partial^2 \Loss}{\partial \mu_{-1,i} \partial \nu_{s,1}} &=\left ( \frac{\partial  }{\partial u_{-1,i}} + \frac{\partial}{\partial u_{1,i}}\right )\left ( \frac{\partial \Loss(x)}{\partial v_{s,-1}} - \frac{\partial \Loss(x) }{\partial v_{s,1}}\right )
\\ & =\frac{\partial^2 \Loss(x)}{\partial u_{-1,i} \partial v_{s,-1}} - \frac{\partial^2 \Loss(x) }{\partial u_{-1,i}\partial v_{s,1}}
+ \frac{\partial^2 \Loss(x)}{\partial u_{1,i}\partial v_{s,-1}} - \frac{\partial^2 \Loss(x) }{\partial u_{1,i}\partial v_{s,1}}
\\ & = \left ( \lambda C_{i,1}^{1,s} + D_{i}^{1,s} + \lambda \mathcal{C}_{i,1}^{1,s} \right )  -  \left ( \lambda C_{i,1}^{1,s}  + \lambda \mathcal{C}_{i,1}^{1,s} \right ) \\ &+ \left ( (1-\lambda) C_{i,1}^{1,s}  + (1-\lambda) \mathcal{C}_{i,1}^{1,s} \right )  -  \left ( (1-\lambda) C_{i,1}^{1,s} + D_{i}^{1,s} + (1-\lambda) \mathcal{C}_{i,1}^{1,s} \right )
\\ & = 0
\end{split}
\end{equation*}

\begin{equation*}
\begin{split}
\frac{\partial^2 \Loss}{\partial \mu_{1,i} \partial \nu_{s,-1}} &=\left ( \alpha \frac{\partial  }{\partial u_{-1,i}} -\beta \frac{\partial}{\partial u_{1,i}}\right )\left ( \frac{\partial \Loss(x)}{\partial v_{s,-1}} + \frac{\partial \Loss(x) }{\partial v_{s,1}}\right )
\\ & =\alpha \frac{\partial^2 \Loss(x)}{\partial u_{-1,i} \partial v_{s,-1}} +\alpha \frac{\partial^2 \Loss(x) }{\partial u_{-1,i}\partial v_{s,1}}
- \beta \frac{\partial^2 \Loss(x)}{\partial u_{1,i}\partial v_{s,-1}} - \beta \frac{\partial^2 \Loss(x) }{\partial u_{1,i}\partial v_{s,1}}
\\ & = \alpha \left ( \lambda C_{i,1}^{1,s} + D_{i}^{1,s} + \lambda \mathcal{C}_{i,1}^{1,s} \right ) + \alpha \left ( \lambda C_{i,1}^{1,s}  + \lambda \mathcal{C}_{i,1}^{1,s} \right ) \\ & -\beta \left ( (1-\lambda) C_{i,1}^{1,s}  + (1-\lambda) \mathcal{C}_{i,1}^{1,s} \right ) - \beta \left ( (1-\lambda) C_{i,1}^{1,s} + D_{i}^{1,s} + (1-\lambda) \mathcal{C}_{i,1}^{1,s} \right )
\\ & = (\alpha - \beta) D_{i}^{1,s}
\end{split}
\end{equation*}

\begin{equation*}
\begin{split}
\frac{\partial^2 \Loss}{\partial \mu_{1,i} \partial \nu_{s,1}} &=\left (\alpha  \frac{\partial  }{\partial u_{-1,i}} -\beta  \frac{\partial}{\partial u_{1,i}}\right )\left ( \frac{\partial \Loss(x)}{\partial v_{s,-1}} - \frac{\partial \Loss(x) }{\partial v_{s,1}}\right )
\\ & =\alpha \frac{\partial^2 \Loss(x)}{\partial u_{-1,i} \partial v_{s,-1}} -\alpha \frac{\partial^2 \Loss(x) }{\partial u_{-1,i}\partial v_{s,1}}
- \beta \frac{\partial^2 \Loss(x)}{\partial u_{1,i}\partial v_{s,-1}} +\beta \frac{\partial^2 \Loss(x) }{\partial u_{1,i}\partial v_{s,1}}
\\ & = \alpha \left ( \lambda C_{i,1}^{1,s} + D_{i}^{1,s} + \lambda \mathcal{C}_{i,1}^{1,s} \right ) - \alpha \left ( \lambda C_{i,1}^{1,s}  + \lambda \mathcal{C}_{i,1}^{1,s} \right )  \\- &\beta \left ( (1-\lambda) C_{i,1}^{1,s}  +  (1-\lambda) \mathcal{C}_{i,1}^{1,s} \right ) + \beta  \left (  (1-\lambda) C_{i,1}^{1,s} + D_{i}^{1,s} +  (1-\lambda) \mathcal{C}_{i,1}^{1,s} \right )
\\ & = (\alpha+\beta) D_{i}^{1,s}
\end{split}
\end{equation*}
For $q\geq 2 $ and $p\in\{-1,1\}$
\begin{equation*}
\begin{split}
\frac{\partial^2 \Loss}{\partial \mu_{-1,i} \partial u_{q,j}} &=\left ( \frac{\partial  }{\partial u_{-1,i}} +  \frac{\partial}{\partial u_{1,i}}\right )\left ( \frac{\partial \Loss(x)}{\partial u_{q,j}}\right )
\\
& = \lambda A_{i,j}^{1,q} +\lambda  \mathcal{A}_{i,j}^{1,1} +  (1-\lambda) A_{i,j}^{1,q} + (1-\lambda) \mathcal{A}_{i,j}^{1,q}\\
&=A_{i,j}^{1,q} + \mathcal{A}_{i,j}^{1,q}\\
\end{split}
\end{equation*}

\begin{equation*}
\begin{split}
\frac{\partial^2 \Loss}{\partial \mu_{1,i} \partial u_{q,j}} &=\left (\alpha  \frac{\partial  }{\partial u_{-1,i}} -\beta  \frac{\partial}{\partial u_{1,i}}\right )\left ( \frac{\partial \Loss(x)}{\partial u_{q,j}}\right )
\\
& = \alpha (\lambda A_{i,j}^{1,q} +\lambda  \mathcal{A}_{i,j}^{1,1}) + \beta ((1-\lambda) A_{i,j}^{1,q} + (1-\lambda) \mathcal{A}_{i,j}^{1,q})\\
&=0\\
\end{split}
\end{equation*}

\begin{equation*}
\begin{split}
\frac{\partial^2 \Loss}{\partial \mu_{-1,i} \partial v_{s,q}} &=\left ( \frac{\partial  }{\partial u_{-1,i}} +  \frac{\partial}{\partial u_{1,i}}\right )\left ( \frac{\partial \Loss(x)}{\partial v_{s,q}}\right )
\\
&= \lambda C_{i,q}^{1,s} + \lambda \mathcal{C}_{i,q}^{1,s} + (1-\lambda) C_{i,q}^{1,s} +(1-\lambda) \mathcal{C}_{i,q}^{1,s}\\
&= C_{i,q}^{1,s} + \mathcal{C}_{i,q}^{1,s}
\end{split}
\end{equation*}

\begin{equation*}
\begin{split}
\frac{\partial^2 \Loss}{\partial \mu_{1,i} \partial v_{s,q}} &=\left (\alpha  \frac{\partial  }{\partial u_{-1,i}} -\beta  \frac{\partial}{\partial u_{1,i}}\right )\left ( \frac{\partial \Loss(x)}{\partial v_{s,q}}\right )
\\
&= \alpha (\lambda C_{i,q}^{1,s} + \lambda \mathcal{C}_{i,q}^{1,s}) -\beta ((1-\lambda) C_{i,q}^{1,s} +(1-\lambda) \mathcal{C}_{i,q}^{1,s})\\
& = 0
\end{split}
\end{equation*}

\begin{equation*}
\begin{split}
\frac{\partial^2 \Loss}{\partial \nu_{s,-1} \partial u_{q,i}} &=\left ( \frac{\partial  }{\partial v_{s,-1}} +  \frac{\partial}{\partial v_{s,1}}\right )\left ( \frac{\partial \Loss(x)}{\partial u_{q,i}}\right )
\\
&=  C_{i,p}^{q,s} + \mathcal{C}_{i,p}^{q,s}+ C_{i,p}^{q,s} + \mathcal{C}_{i,p}^{q,s}\\
&=2  C_{i,p}^{q,s} + 2\mathcal{C}_{i,p}^{q,s}\\ 
\end{split}
\end{equation*}

\begin{equation*}
\begin{split}
\frac{\partial^2 \Loss}{\partial \nu_{s,1} \partial u_{q,i}} &=\left ( \frac{\partial  }{\partial v_{s,-1}} - \frac{\partial}{\partial v_{s,1}}\right )\left ( \frac{\partial \Loss(x)}{\partial u_{q,i}}\right )
\\
&=  C_{i,p}^{q,s} + \mathcal{C}_{i,p}^{q,s}- C_{i,p}^{q,s} - \mathcal{C}_{i,p}^{q,s}\\
&=0\\ 
\end{split}
\end{equation*}

\begin{equation*}
\begin{split}
\frac{\partial^2 \Loss}{\partial \nu_{s,-1} \partial v_{t,q}} &=\left ( \frac{\partial  }{\partial v_{s,-1}} +  \frac{\partial}{\partial v_{s,1}}\right )\left ( \frac{\partial \Loss(x)}{\partial v_{t,q}}\right )
\\
&= E_{1 ,q}^{s,t} + \mathcal{E}_{1 ,q}^{s,t} + E_{1 ,q}^{s,t} + \mathcal{E}_{1 ,q}^{s,t}\\
&= 2E_{1 ,q}^{s,t} + 2\mathcal{E}_{1 ,q}^{s,t}
\end{split}
\end{equation*}

\begin{equation*}
\begin{split}
\frac{\partial^2 \Loss}{\partial \nu_{s,1} \partial v_{t,q}} &=\left ( \frac{\partial  }{\partial v_{s,-1}} - \frac{\partial}{\partial v_{s,1}}\right )\left ( \frac{\partial \Loss(x)}{\partial v_{t,q}}\right )
\\
&= E_{1 ,q}^{s,t} + \mathcal{E}_{1 ,q}^{s,t} -E_{1 ,q}^{s,t} - \mathcal{E}_{1 ,q}^{s,t}\\
&=0\\
\end{split}
\end{equation*}

We also need to consider the second derivative with respect to the other variables of $\wrest$. If $w$ is closer to the output than $[u_{p,i}]_{p,i},[v_{s,q}]_{s,q}$ belonging to layer $\gamma$ where $\gamma>l+1$, then we get 
\begin{equation*}
\begin{split}
\frac{\partial^2 \Loss}{\partial w \partial \mu_{-1,i} } & =  \frac{\partial  }{\partial w} \left ( \frac{\partial \Loss}{\partial u_{-1,i}} + \frac{\partial \Loss }{\partial u_{1,i}} \right )\\
& =  \sum_\alpha \frac{\partial f(x_\alpha)}{\partial w} \left ( 
\sum_k \frac{\partial \h{l+1}{\n{l+1}{x_\alpha}}}{\partial \n{l+1,k}{x_\alpha}} \cdot v_{k,-1}\cdot \sigma'(\n{l,-1}{x_\alpha} ) \cdot \act{l-1,i}{x_\alpha}
 \right )\\
 & + \sum_\alpha \frac{\partial f(x_\alpha)}{\partial w}\left (    
\sum_k \frac{\partial \h{l+1}{\n{l+1}{x_\alpha}}}{\partial \n{l+1,k}{x_\alpha}} \cdot v_{k,1}\cdot \sigma'(\n{l,1}{x_\alpha} ) \cdot \act{l-1,i}{x_\alpha}
 \right )\\
   & = \sum_\alpha (f(x_\alpha)-y_\alpha) \cdot 
\sum_k \frac{\partial^2 \h{l+1}{\n{l+1}{x_\alpha}}}{\partial w \partial \n{l+1,k}{x_\alpha}} \cdot v_{k,-1}\cdot \sigma'(\n{l,-1}{x_\alpha} ) \cdot \act{l-1,i}{x_\alpha}
\\
 & + \sum_\alpha (f(x_\alpha)-y_\alpha)  \cdot 
\sum_k \frac{\partial^2 \h{l+1}{\n{l+1}{x_\alpha}}}{\partial w  \partial \n{l+1,k}{x_\alpha}} \cdot v_{k,1}\cdot \sigma'(\n{l,1}{x_\alpha} ) \cdot \act{l-1,i}{x_\alpha}
\\
& = \frac{\partial^2 \loss}{\partial w^\varphi \partial u^\varphi_{1,i} }
\end{split}
\end{equation*}
and
\begin{equation*}
\begin{split}
\frac{\partial^2 \Loss}{\partial w \partial \mu_{-1,i} } & =  \frac{\partial  }{\partial w} \left ( \alpha \frac{\partial \Loss}{\partial u_{-1,i}} -\beta  \frac{\partial \Loss }{\partial u_{1,i}} \right )\\
& = \sum_\alpha \frac{\partial f(x_\alpha)}{\partial w}\cdot  \left (
\sum_k \frac{\partial \h{l+1}{\n{l+1}{x_\alpha}}}{\partial \n{l+1,k}{x_\alpha}} \cdot \alpha v_{k,-1}\cdot \sigma'(\n{l,-1}{x_\alpha} ) \cdot \act{l-1,i}{x_\alpha}
 \right )\\
 & - \sum_\alpha \frac{\partial f(x_\alpha)}{\partial w}  \cdot \left (
\sum_k \frac{\partial \h{l+1}{\n{l+1}{x_\alpha}}}{\partial \n{l+1,k}{x_\alpha}} \cdot \beta v_{k,1}\cdot \sigma'(\n{l,1}{x_\alpha} ) \cdot \act{l-1,i}{x_\alpha}
 \right )\\ & + \sum_\alpha (f(x_\alpha)-y_\alpha) \cdot 
\sum_k \frac{\partial^2 \h{l+1}{\n{l+1}{x_\alpha}}}{\partial w \partial \n{l+1,k}{x_\alpha}} \cdot \alpha v_{k,-1}\cdot \sigma'(\n{l,-1}{x_\alpha} ) \cdot \act{l-1,i}{x_\alpha}
\\
 & - \sum_\alpha (f(x_\alpha)-y_\alpha)  \cdot 
\sum_k \frac{\partial^2 \h{l+1}{\n{l+1}{x_\alpha}}}{\partial w  \partial \n{l+1,k}{x_\alpha}} \cdot \beta v_{k,1}\cdot \sigma'(\n{l,1}{x_\alpha} ) \cdot \act{l-1,i}{x_\alpha}
\\
&=  0
\end{split}
\end{equation*}
and
\begin{equation*}
\begin{split}
\frac{\partial^2 \Loss}{\partial w \partial \nu_{s,-1} } & =  \frac{\partial  }{\partial w} \left ( \frac{\partial \Loss}{\partial v_{s,-1}} + \frac{\partial \Loss }{\partial v_{s,1}} \right )\\
& = \sum_\alpha \frac{\partial f(x_\alpha)}{\partial w} \cdot \left ( \frac{\partial \h{l+1}{\n{l+1}{x_\alpha}}}{\partial \n{l+1,s}{x_\alpha}} \cdot \act{l,-1}{x_\alpha} ) \right )\\
 & + \sum_\alpha \frac{\partial f(x_\alpha)}{\partial w} \cdot \left ( \frac{\partial \h{l+1}{\n{l+1}{x_\alpha}}}{\partial \n{l+1,s}{x_\alpha}} \cdot \act{l,1}{x_\alpha} 
 \right )\\ & + \sum_\alpha (f(x_\alpha)-y_\alpha) \cdot 
\frac{\partial^2 \h{l+1}{\n{l+1}{x_\alpha}}}{\partial w \partial \n{l+1,s}{x_\alpha}} \cdot \act{l,-1}{x_\alpha}  \\
 & + \sum_\alpha (f(x_\alpha)-y_\alpha)  \cdot 
\frac{\partial^2 \h{l+1}{\n{l+1}{x_\alpha}}}{\partial w  \partial \n{l+1,s}{x_\alpha}} \cdot \act{l,1}{x_\alpha} 
\\
& = 2\cdot \frac{\partial^2 \loss}{\partial w^\varphi \partial v^\varphi_{s,1} }
\end{split}
\end{equation*}
and
\begin{equation*}
\begin{split}
\frac{\partial^2 \Loss}{\partial w \partial \nu_{s,1} } & =  \frac{\partial  }{\partial w} \left ( \frac{\partial \Loss}{\partial v_{s,-1}} - \frac{\partial \Loss }{\partial v_{s,1}} \right )\\
& = \sum_\alpha \frac{\partial f(x_\alpha)}{\partial w} \cdot \left ( \frac{\partial \h{l+1}{\n{l+1}{x_\alpha}}}{\partial \n{l+1,s}{x_\alpha}} \cdot \act{l,-1}{x_\alpha} ) \right )\\
 & - \sum_\alpha \frac{\partial f(x_\alpha)}{\partial w}  \cdot \left ( \frac{\partial \h{l+1}{\n{l+1}{x_\alpha}}}{\partial \n{l+1,s}{x_\alpha}} \cdot \act{l,1}{x_\alpha} 
 \right )\\ & + \sum_\alpha (f(x_\alpha)-y_\alpha) \cdot 
\frac{\partial^2 \h{l+1}{\n{l+1}{x_\alpha}}}{\partial w \partial \n{l+1,s}{x_\alpha}} \cdot \act{l,-1}{x_\alpha}  \\
 & - \sum_\alpha (f(x_\alpha)-y_\alpha)  \cdot 
\frac{\partial^2 \h{l+1}{\n{l+1}{x_\alpha}}}{\partial w  \partial \n{l+1,s}{x_\alpha}} \cdot \act{l,1}{x_\alpha} 
\\
& = 0
\end{split}
\end{equation*}

If $w$ is closer to the input than $[u_{p,i}]_{p,i},[v_{s,q}]_{s,q}$ connecting neuron $j$ of layer $\gamma-1$ with neuron $r$ of layer $\gamma$ where $\gamma<l$, then we get
\begin{equation*}
\begin{split}
\frac{\partial^2 \Loss}{\partial \mu_{-1,i} \partial w  } & =  \left ( \frac{\partial  }{\partial u_{-1,i}} +  \frac{\partial  }{\partial u_{1,i}} \right )  \left ( \frac{\partial \Loss}{\partial w}  \right )\\
& = \sum_\alpha (f(x_\alpha)-y_\alpha) \cdot  \left ( \frac{\partial  }{\partial u_{-1,i}} +  \frac{\partial  }{\partial u_{1,i}} \right )\left ( \frac{\partial \h{\gamma}{\n{\gamma}{x_\alpha}}}{\partial \n{\gamma,r}{x_\alpha}} \cdot \act{\gamma-1,j}{x_\alpha}
 \right )\\
 & + \sum_\alpha  \left ( \frac{\partial f(x_\alpha) }{\partial u_{-1,i}} +  \frac{\partial f(x_\alpha)  }{\partial u_{1,i}} \right )  \cdot \left ( \frac{\partial \h{\gamma}{\n{\gamma}{x_\alpha}}}{\partial \n{\gamma,r}{x_\alpha}} \cdot \act{\gamma-1,j}{x_\alpha}
 \right )\\
& = \sum_\alpha (f(x_\alpha)-y_\alpha) \cdot  \sum_k \frac{\partial^2 \h{\gamma}{\n{\gamma}{x_\alpha}}}{\partial \n{l+1,k}{x_\alpha} \partial \n{\gamma,r}{x_\alpha}} \cdot v_{k,-1}\cdot \sigma'(\n{l,-1}{x_\alpha} ) \\ &\cdot \act{l-1,i}{x_\alpha}   \cdot \act{\gamma-1,j}{x_\alpha}
\\
 & + \sum_\alpha (f(x_\alpha)-y_\alpha) \cdot  \sum_k \frac{\partial^2 \h{\gamma}{\n{\gamma}{x_\alpha}}}{\partial \n{l+1,k}{x_\alpha} \partial \n{\gamma,r}{x_\alpha}} \cdot v_{k,1}\cdot \sigma'(\n{l,1}{x_\alpha} )  \\ & \cdot \act{l-1,i}{x_\alpha} \cdot \act{\gamma-1,j}{x_\alpha}
\\
 & + \sum_\alpha  \left ( \frac{\partial \varphi(x_\alpha) }{\partial u^\varphi _{1,i}} \right )  \cdot \left ( \frac{\partial \h{\gamma}{\n{\gamma}{x_\alpha}}}{\partial \n{\gamma,r}{x_\alpha}} \cdot \act{\gamma-1,j}{x_\alpha}
 \right )\\
& = \frac{\partial^2 \loss}{\partial w^\varphi \partial u^\varphi_{1,i} }
\end{split}
\end{equation*}
and
\begin{equation*}
\begin{split}
\frac{\partial^2 \Loss}{\partial \mu_{-1,i} \partial w  } & =  \left ( \alpha \frac{\partial  }{ \partial u_{-1,i}} - \beta  \frac{\partial  }{\partial u_{1,i}} \right )  \left ( \frac{\partial \Loss}{\partial w}  \right )\\
& = \sum_\alpha (f(x_\alpha)-y_\alpha) \cdot  \left (\alpha \frac{\partial  }{\partial u_{-1,i}} - \beta  \frac{\partial  }{\partial u_{-1,i}} \right )\left ( \frac{\partial \h{\gamma}{\n{\gamma}{x_\alpha}}}{\partial \n{\gamma,r}{x_\alpha}} \cdot \act{\gamma-1,j}{x_\alpha}
 \right )\\
& + \sum_\alpha   \left (\alpha \frac{\partial f(x_\alpha)}{\partial u_{-1,i}} - \beta  \frac{\partial f(x_\alpha)}{\partial u_{-1,i}} \right )\cdot \left ( \frac{\partial \h{\gamma}{\n{\gamma}{x_\alpha}}}{\partial \n{\gamma,r}{x_\alpha}} \cdot \act{\gamma-1,j}{x_\alpha}
 \right )\\
 & = \sum_\alpha (f(x_\alpha)-y_\alpha) \cdot  \sum_k \frac{\partial^2 \h{\gamma}{\n{\gamma}{x_\alpha}}}{\partial \n{l+1,k}{x_\alpha} \partial \n{\gamma,r}{x_\alpha}} \cdot v_{k,-1}\cdot \sigma'(\n{l,-1}{x_\alpha} ) \\ &\cdot \act{l-1,i}{x_\alpha}   \cdot \act{\gamma-1,j}{x_\alpha}
\\
 & - \sum_\alpha (f(x_\alpha)-y_\alpha) \cdot  \sum_k \frac{\partial^2 \h{\gamma}{\n{\gamma}{x_\alpha}}}{\partial \n{l+1,k}{x_\alpha} \partial \n{\gamma,r}{x_\alpha}} \cdot v_{k,1}\cdot \sigma'(\n{l,1}{x_\alpha} )  \\ & \cdot \act{l-1,i}{x_\alpha} \cdot \act{\gamma-1,j}{x_\alpha}
\\
& + \sum_\alpha   \left ((\alpha \lambda - \beta (1-\lambda) )\frac{\partial \varphi(x_\alpha)}{\partial u^\varphi_{1,i}}  \right )\cdot \left ( \frac{\partial \h{\gamma}{\n{\gamma}{x_\alpha}}}{\partial \n{\gamma,r}{x_\alpha}} \cdot \act{\gamma-1,j}{x_\alpha}
 \right )\\
& =0
\end{split}
\end{equation*}
and
\begin{equation*}
\begin{split}
\frac{\partial^2 \Loss}{\partial \nu_{s,-1} \partial w  } & =  \left ( \frac{\partial  }{\partial v_{s,-1}} +  \frac{\partial  }{\partial v_{s,1}} \right )  \left ( \frac{\partial \Loss}{\partial w}  \right )\\
& = \sum_\alpha (f(x_\alpha)-y_\alpha) \cdot  \left ( \frac{\partial  }{\partial v_{s,-1}} +  \frac{\partial  }{\partial v_{s,1}} \right )\left ( \frac{\partial \h{\gamma}{\n{\gamma}{x_\alpha}}}{\partial \n{\gamma,r}{x_\alpha}} \cdot \act{\gamma-1,j}{x_\alpha}
 \right )\\
& + \sum_\alpha \left ( \frac{\partial f(x_\alpha)}{\partial v_{s,-1}} +  \frac{\partial f(x_\alpha)}{\partial v_{s,1}} \right ) \cdot  \left ( \frac{\partial \h{\gamma}{\n{\gamma}{x_\alpha}}}{\partial \n{\gamma,r}{x_\alpha}} \cdot \act{\gamma-1,j}{x_\alpha}
 \right )\\
& = \sum_\alpha (f(x_\alpha)-y_\alpha) \cdot \frac{\partial^2 \h{\gamma}{\n{\gamma}{x_\alpha}}}{\partial \n{l+1,s}{x_\alpha} \partial \n{\gamma,r}{x_\alpha}} \cdot \act{l,-1}{x_\alpha} ) \cdot \act{\gamma-1,j}{x_\alpha}
\\
 & + \sum_\alpha (f(x_\alpha)-y_\alpha) \cdot  \frac{\partial^2 \h{\gamma}{\n{\gamma}{x_\alpha}}}{\partial \n{l+1,s}{x_\alpha} \partial \n{\gamma,r}{x_\alpha}} \cdot \act{l,1}{x_\alpha} )  \cdot \act{\gamma-1,j}{x_\alpha}
\\
& + \sum_\alpha \left ( \frac{\partial \varphi(x_\alpha)}{\partial v^\varphi_{s,1}} +  \frac{\partial \varphi(x_\alpha)}{\partial v^\varphi_{s,1}} \right ) \cdot  \left ( \frac{\partial \h{\gamma}{\n{\gamma}{x_\alpha}}}{\partial \n{\gamma,r}{x_\alpha}} \cdot \act{\gamma-1,j}{x_\alpha}
 \right )\\
& =2\cdot  \frac{\partial^2 \loss}{\partial w^\varphi \partial v^\varphi_{s,1} }
\end{split}
\end{equation*}
and
\begin{equation*}
\begin{split}
\frac{\partial^2 \Loss}{\partial \nu_{s,1} \partial w  } & =  \left ( \frac{\partial  }{\partial v_{s,-1}} -  \frac{\partial  }{\partial v_{s,1}} \right )  \left ( \frac{\partial \Loss}{\partial w}  \right )\\
& = \sum_\alpha (f(x_\alpha)-y_\alpha) \cdot  \left ( \frac{\partial  }{\partial v_{s,-1}} -  \frac{\partial  }{\partial v_{s,1}} \right )\left ( \frac{\partial \h{\gamma}{\n{\gamma}{x_\alpha}}}{\partial \n{\gamma,r}{x_\alpha}} \cdot \act{\gamma-1,j}{x_\alpha}
 \right )\\
& + \sum_\alpha\left ( \frac{\partial f(x_\alpha)}{\partial v_{s,-1}} -  \frac{\partial f(x_\alpha) }{\partial v_{s,1}} \right ) \cdot  \left ( \frac{\partial \h{\gamma}{\n{\gamma}{x_\alpha}}}{\partial \n{\gamma,r}{x_\alpha}} \cdot \act{\gamma-1,j}{x_\alpha}
 \right )\\
& = \sum_\alpha (f(x_\alpha)-y_\alpha) \cdot \frac{\partial^2 \h{\gamma}{\n{\gamma}{x_\alpha}}}{\partial \n{l+1,s}{x_\alpha} \partial \n{\gamma,r}{x_\alpha}} \cdot \act{l,-1}{x_\alpha} ) \cdot \act{\gamma-1,j}{x_\alpha}
\\
 & - \sum_\alpha (f(x_\alpha)-y_\alpha) \cdot  \frac{\partial^2 \h{\gamma}{\n{\gamma}{x_\alpha}}}{\partial \n{l+1,s}{x_\alpha} \partial \n{\gamma,r}{x_\alpha}} \cdot \act{l,1}{x_\alpha} )  \cdot \act{\gamma-1,j}{x_\alpha}
\\
& =0
\end{split}
\end{equation*}

Finally, if $w$ and $w'$ are parameters different from $[u_{p,i}]_{p,i},[v_{s,q}]_{s,q},\mu$ and $\nu$, then 
\begin{equation*}
\frac{\partial^2 \Loss}{\partial w \partial w' }  =  \frac{\partial^2 \loss}{\partial w^\varphi \partial w'^\varphi } \\
\end{equation*}

\subsubsection{The Hessian}
Putting things together, the matrix for the second derivative of $\Loss$ with respect to $\mu_-1, \nu_-1, \wrest, \mu_1, \nu_1$, where $\wrest$ stands for the collection of all other parameters, at $\gamma_\lambda^1([u_{1,i}^*]_i,[v_{s,1}^*]_s,\wrest*)$ is given by:
$$
\begin{pmatrix}
[\frac{\partial^2 \loss}{\partial u_{1,i} \partial u_{1,j} }]_{i,j}  & 2 [\frac{\partial^2 \loss}{\partial u_{1,i} \partial v_{s,1} }]_{i,s}& [\frac{\partial^2 \loss}{\partial \wrest\ \partial u_{1,i} }]_{i,\wrest}  &0  &0  \\ 
2[\frac{\partial^2 \loss}{\partial u_{1,i} \partial v_{s,1} }]_{s,i} & 4 [\frac{\partial^2 \loss}{\partial v_{s,1} \partial v_{t,1} }]_{s,t} &  2[\frac{\partial^2 \loss}{\partial \wrest\ \partial v_{s,1} }]_{s,\wrest}  & (\alpha-\beta)[D_{i}^{1,s}]_{s,i} &  0  \\
[\frac{\partial^2 \loss}{\partial \wrest\ \partial u_{1,i} }]_{\wrest,i} &  2[ \frac{\partial^2 \loss}{\partial \wrest\ \partial v_{s,1} }]_{\wrest,s} & [\frac{\partial^2 \loss}{\partial \wrest\ \partial \wrest' }]_{\wrest,\wrest'} &  0  & 0   \\
0 & (\alpha-\beta)[D_{i}^{1,s}]_{i,s}& 0 & \alpha \beta [B_{i,j}^1]_{i,j}  &   (\alpha + \beta) [ D_{i}^{1,s}]_{i,s}  \\
0 &  0 & 0 & (\alpha + \beta) [ D_{i}^{1,s}]_{s,i} & 0  \\
\end{pmatrix}
$$

\end{document}

\end{document}